\documentclass[sigconf]{acmart}
\AtBeginDocument{%
  \providecommand\BibTeX{{%
    \normalfont B\kern-0.5em{\scshape i\kern-0.25em b}\kern-0.8em\TeX}}}

\setcopyright{acmlicensed}
\copyrightyear{2024}
\acmYear{2024}
\acmDOI{XXXXXXX.XXXXXXX}

\acmConference[CCS '24]{Make sure to enter the correct
  conference title from your rights confirmation email}{October 14-18, 2024}
{Salt Lake City, U.S.A.}
\acmISBN{978-1-4503-XXXX-X/18/06}

\usepackage{microtype}
\usepackage{graphicx}
\usepackage{booktabs} 
\usepackage{subcaption}

\usepackage{amsmath,amsfonts,bm}

\def\eqref#1{equation~\ref{#1}}

\def\1{\bm{1}}

\DeclareMathAlphabet{\mathsfit}{\encodingdefault}{\sfdefault}{m}{sl}
\SetMathAlphabet{\mathsfit}{bold}{\encodingdefault}{\sfdefault}{bx}{n}

\DeclareMathOperator*{\argmax}{arg\,max}
\DeclareMathOperator*{\argmin}{arg\,min}

\usepackage{graphicx}%
\usepackage{multirow}%
\usepackage{amsmath,amsfonts}%
\usepackage{amsthm}%
\usepackage{mathrsfs}%
\usepackage[title]{appendix}%
\usepackage{xcolor}%
\usepackage{textcomp}%
\usepackage{manyfoot}%
\usepackage{booktabs}%
\usepackage{algorithm}%
\usepackage{algorithmicx}%
\usepackage{algpseudocode}%
\usepackage{listings}%
\usepackage{natbib}
\usepackage{caption, subcaption}
\usepackage{wrapfig,lipsum,booktabs}

\newcommand\xbf{\boldsymbol{x}}
\usepackage[textsize=tiny]{todonotes}
\usepackage{url}

\theoremstyle{plain}
\newtheorem{theorem}{Theorem}

\newtheorem{lemma}{Lemma}
\newtheorem{definition}{Definition}

\theoremstyle{definition}
\newtheorem{assumption}{Assumption}
\DeclareMathOperator*{\Ee}{\mathbb{E}}

\begin{document}



\title{Fisher Information guided Purification against Backdoor Attacks}

\author{Nazmul Karim}
\authornote{Both authors contributed equally to this work.}
\email{nazmul.karim170@gmail.com}
\affiliation{%
  \institution{University of Central Florida}
  \streetaddress{P.O. Box 1212}
  \city{Orlando}
  \state{Florida}
  \country{USA}
  \postcode{43017-6221}
}
\author{Abdullah Al Arafat}
\authornotemark[1]
\email{aalaraf@ncsu.edu}
\affiliation{%
  \institution{North Carolina State University}
  \streetaddress{P.O. Box 1212}
  \city{Raleigh}
  \state{North Carolina}
  \country{USA}
  \postcode{43017-6221}
}

\author{Adnan Siraj Rakin}
\email{arakin@binghamton.edu}
\affiliation{%
  \institution{Binghamton University (SUNY)}
  \streetaddress{P.O. Box 1212}
  \city{Binghamton}
  \state{New York}
  \country{USA}
  \postcode{43017-6221}
}

\author{Zhishan Guo}
\email{zguo32@ncsu.edu}
\affiliation{%
  \institution{North Carolina State University}
  \streetaddress{P.O. Box 1212}
  \city{Raleigh}
  \state{North Carolina}
  \country{USA}
  \postcode{43017-6221}
}

\author{Nazanin Rahnavard}
\email{nazanin.rahnavard@ucf.edu}
\affiliation{%
  \institution{University of Central Florida}
  \streetaddress{P.O. Box 1212}
  \city{Orlando}
  \state{Florida}
  \country{USA}
  \postcode{43017-6221}
}


\begin{abstract}
Studies on backdoor attacks in recent years suggest that an adversary can compromise the integrity of a deep neural network (DNN) by manipulating a small set of training samples. 
Our analysis shows that such manipulation can make the backdoor model converge to a \emph{bad local minima}, i.e., sharper minima as compared to a benign model. Intuitively, the backdoor can be purified by re-optimizing the model to smoother minima.  However, a na\"ive adoption of any optimization targeting smoother minima can lead to sub-optimal purification techniques hampering the clean test accuracy. Hence, to effectively obtain such re-optimization, inspired by our novel perspective establishing the connection between backdoor removal and loss smoothness, we propose \emph{\underline{F}isher \underline{I}nformation guided \underline{P}urification (FIP)}, a novel backdoor purification framework. Proposed FIP consists of a couple of novel regularizers that aid the model in suppressing the backdoor effects and retaining the acquired knowledge of clean data distribution throughout the backdoor removal procedure through exploiting the knowledge of \emph{Fisher Information Matrix (FIM)}.
In addition, we introduce an efficient variant of FIP, dubbed as \emph{Fast FIP}, which reduces the number of tunable parameters significantly and obtains an impressive runtime gain of almost $5\times$. Extensive experiments show that the proposed method achieves state-of-the-art (SOTA) performance on a wide range of backdoor defense benchmarks: 5 different tasks---\emph{Image Recognition, Object Detection, Video Action Recognition, 3D point Cloud, Language Generation}; 11 different datasets including \emph{ImageNet, PASCAL VOC, UCF101}; diverse model architectures spanning both CNN and vision transformer; 14 different backdoor attacks, e.g., \emph{Dynamic, WaNet, LIRA, ISSBA, etc.} Our code is available in this \href{https://github.com/nazmul-karim170/FIP-Fisher-Backdoor-Removal}{\textbf{GitHub Repository}}.
\end{abstract}




\maketitle

\section{Introduction}\label{sec:intro}
Training a deep neural network (DNN) with a fraction of poisoned or malicious data is often security-critical since the model can successfully learn both clean and adversarial tasks equally well.
This is prominent in scenarios where one outsources the DNN training to a vendor. In such scenarios, an adversary can mount backdoor attacks~\citep{gu2019badnets,chen2017targeted} by poisoning a portion of training samples so that the model will classify any sample with a \emph{particular trigger} or \emph{pattern} to an adversary-set label. Whenever a DNN is trained in such a manner, it becomes crucial to remove the effect of a backdoor before deploying it for a real-world application.   In recent times, several attempts have been made~\citep{liu2018fine,wang2019neural,wu2021adversarial, li2021anti,zheng2022data, zhu2023enhancing} to tackle the backdoor issue in DNN training. Defense techniques such as fine-pruning (FP)~\citep{liu2018fine} aim to prune vulnerable neurons affected by the backdoor. Most of the recent backdoor defenses can be categorized into two groups based on the intuition or perspective they are built on. They are (i) \emph{pruning based defense}~\cite{liu2018fine, wu2021adversarial, zheng2022data}: some weights/channel/neurons are more vulnerable to backdoor than others. Therefore, pruning or masking bad neurons should remove the backdoor. 
(ii) \emph{trigger approximation based defense}~\cite{ zeng2021adversarial, chai2022one}:  recovering the original trigger pattern and fine-tuning the model with this trigger attached to samples and corresponding benign labels would remove the backdoor. However, they require computationally expensive adversarial search approaches to find the backdoor-sensitive model parameters or reverse-engineered backdoor triggers, which makes efficient post-training model purification challenging.

In contrast to the expensive adversarial search and reverse-engineering methods, general-purpose fine-tuning can moderately remove the effects of backdoors and has been adopted as a component in existing defenses~\cite{liu2018fine,li2021neural}. However, adopting vanilla fine-tuning is challenging with limited benign data~\cite{wu2022backdoorbench} and cannot remove strong backdoor attacks, e.g., Blend~\cite{chen2017targeted} and smooth low frequency (LF) trigger~\cite{zeng2021rethinking}. Recently, Zhu et al.,~\cite{zhu2023enhancing} proposed an enhanced fine-tuning method to effectively remove backdoors following their observation aligning with earlier results that neurons with large norms are responsible for backdoors. However, their enhancement is based on a general-purpose optimizer, SAM~\cite{foret2021sharpnessaware}, which affects the runtime of the purification process and accuracy for the clean samples.
Rather than empirical observations, there is a research gap in thoroughly analyzing backdoored models in connecting the purification defense to key \textit{changes} in a model during backdoor insertion, which will lead to an efficient method.

To address this gap, we propose a \emph{novel perspective for analyzing the backdoor in DNNs}. We explore the backdoor insertion and removal phenomena from the DNN optimization point of view. Unlike a benign model, a backdoor model is forced to learn two different data distributions: clean data distribution and poison data distribution. Having to learn both distributions, backdoor model optimization usually leads to a \emph{bad local minima} or sharper minima \emph{w.r.t.} clean distribution. We verify this phenomenon by tracking the spectral norm over the training of a benign and a backdoor model (see Figure~\ref{fig:eigen_spectral}). We also provide theoretical justification for such discrepancy in convergence behavior. Intuitively, we claim that the backdoor can be removed by re-optimizing the model to smoother minima. In addition, instead of na\"ively adopting any re-optimization strategy targeting smooth minima, in this work, we propose a novel backdoor purification technique---\emph{Fisher Information guided Purification (FIP)} by exploiting the knowledge of \emph{Fisher Information Matrix (FIM)} of a DNN to remove the imprint of the backdoor. Specifically, an FIM-guided regularizer has been introduced to achieve smooth convergence, effectively removing the backdoor. Our contribution can be summarized as follows:
\begin{itemize}
    \item \emph{Novel Perspective for Backdoor Analysis.} We analyze the backdoor insertion process in DNNs from the optimization point of view. Our analysis shows that the optimization of a backdoor model leads to a \emph{bad local minima} or sharper minima compared to a benign model. We also provide theoretical justifications for our novel findings. To the best of our knowledge, this is the first study establishing the correlation between smoothness and backdoor attacks.
    \item \emph{Novel Backdoor Defense.} We propose a novel technique, FIP, that removes the backdoor by re-optimizing the model to smooth minima. However, purifying the backdoor in this manner can lead to poor clean test time performance due to drastic changes in the original backdoor model parameters. To preserve the original test accuracy of the model, we propose a novel clean data-distribution-aware regularizer that encourages less drastic changes to the model parameters responsible for remembering the clean distribution.
    \item \emph{Better Runtime Efficiency.}  In addition, we propose a computationally efficient variant of FIP, i.e., \emph{Fast FIP}, where we perform spectral decomposition of the weight matrices and fine-tune only the singular values while freezing the corresponding singular vectors. By reducing the tunable parameters, the purification time can be shortened significantly. 
    \item \emph{Comprehensive Evaluation.} Evaluation on a wide range of backdoor-related benchmarks shows that FIP obtains SOTA performance both in terms of purification performance and runtime. 
\end{itemize}

\begin{figure*}[h!]
  \centering
  \begin{subfigure}{0.235\linewidth}
    \includegraphics[width=1\linewidth]{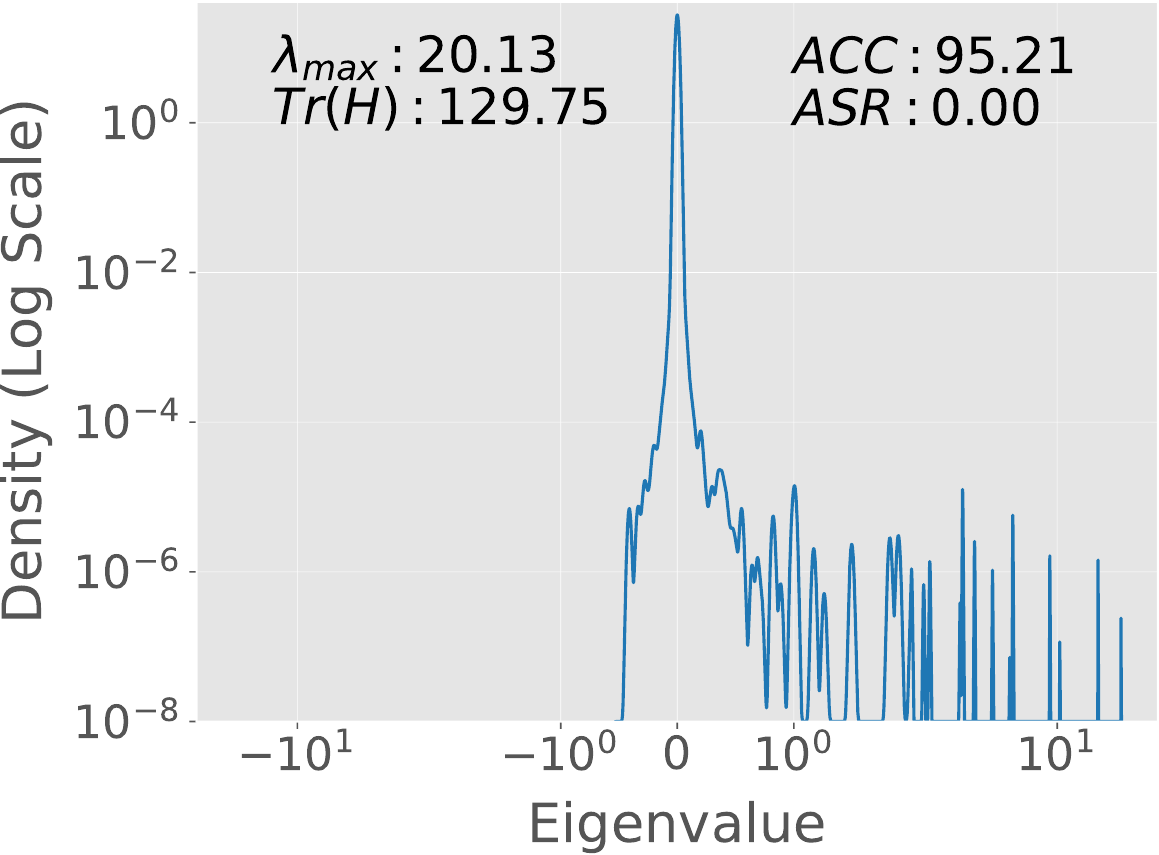}
    \caption{ \scriptsize Benign Model}
    \label{fig:ls-cln}
  \end{subfigure}
  \hfill
    \begin{subfigure}{0.235\linewidth}
    \includegraphics[width=1\linewidth]{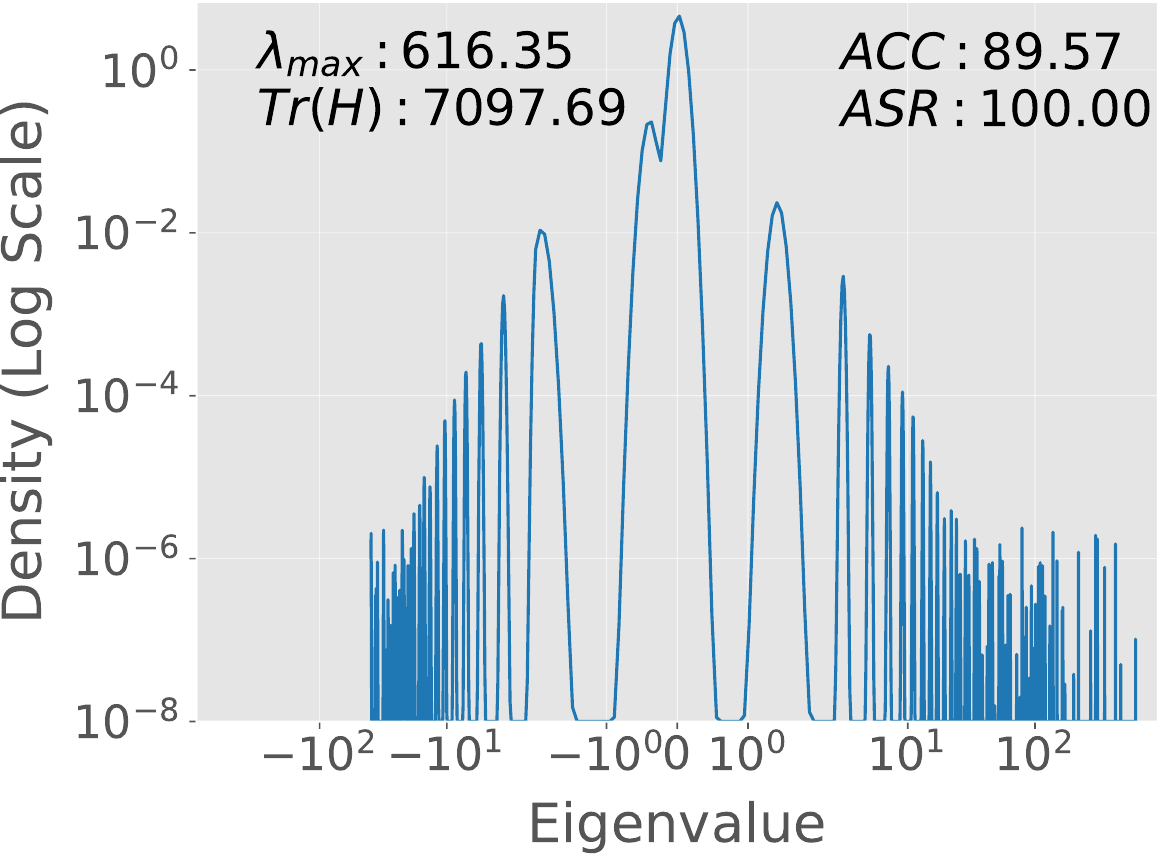}
    \caption{ \scriptsize Backdoor Model}
    \label{fig:ls-bd}
  \end{subfigure}
  \hfill
\begin{subfigure}{0.235\linewidth}
    \includegraphics[width=1\linewidth]{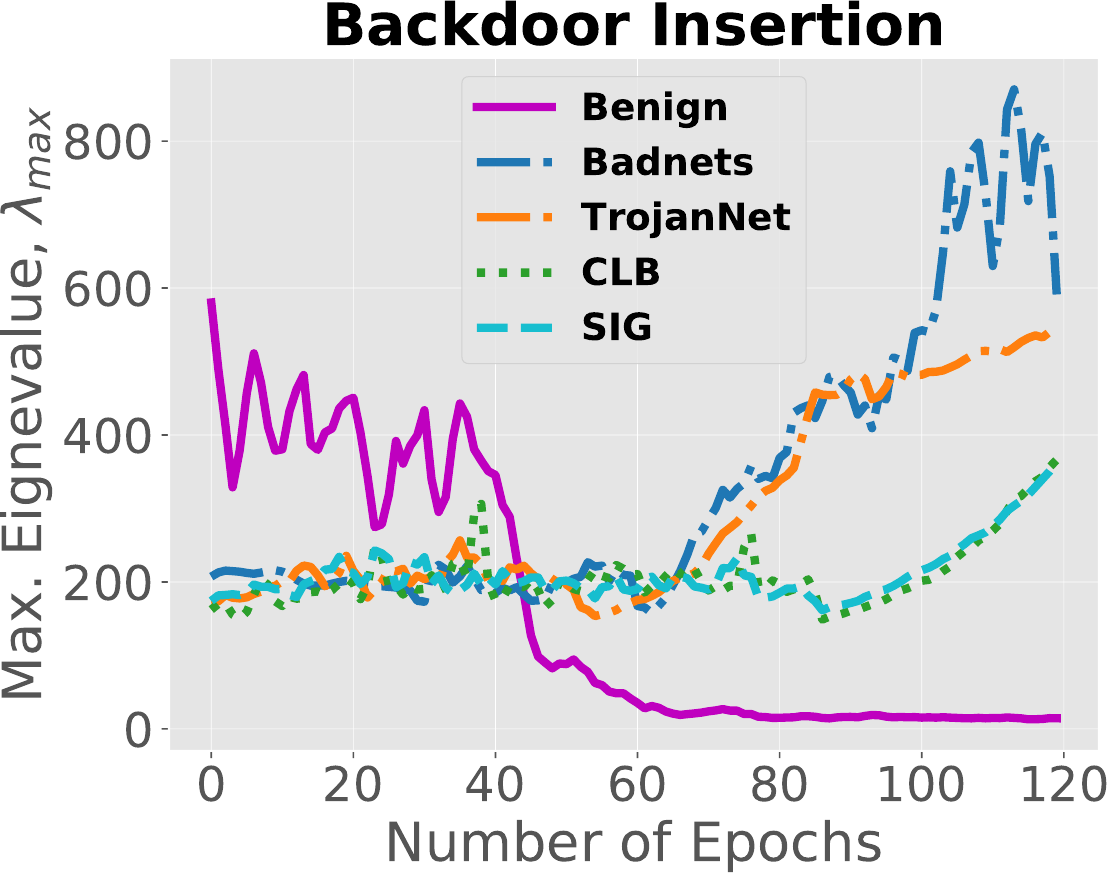}
    \caption{\scriptsize $\lambda_\mathsf{max}$ vs. Epochs}
    \label{fig:backdoor_eigens}
\end{subfigure}
\hfill
\begin{subfigure}{0.235\linewidth}
    \includegraphics[width=1\linewidth]{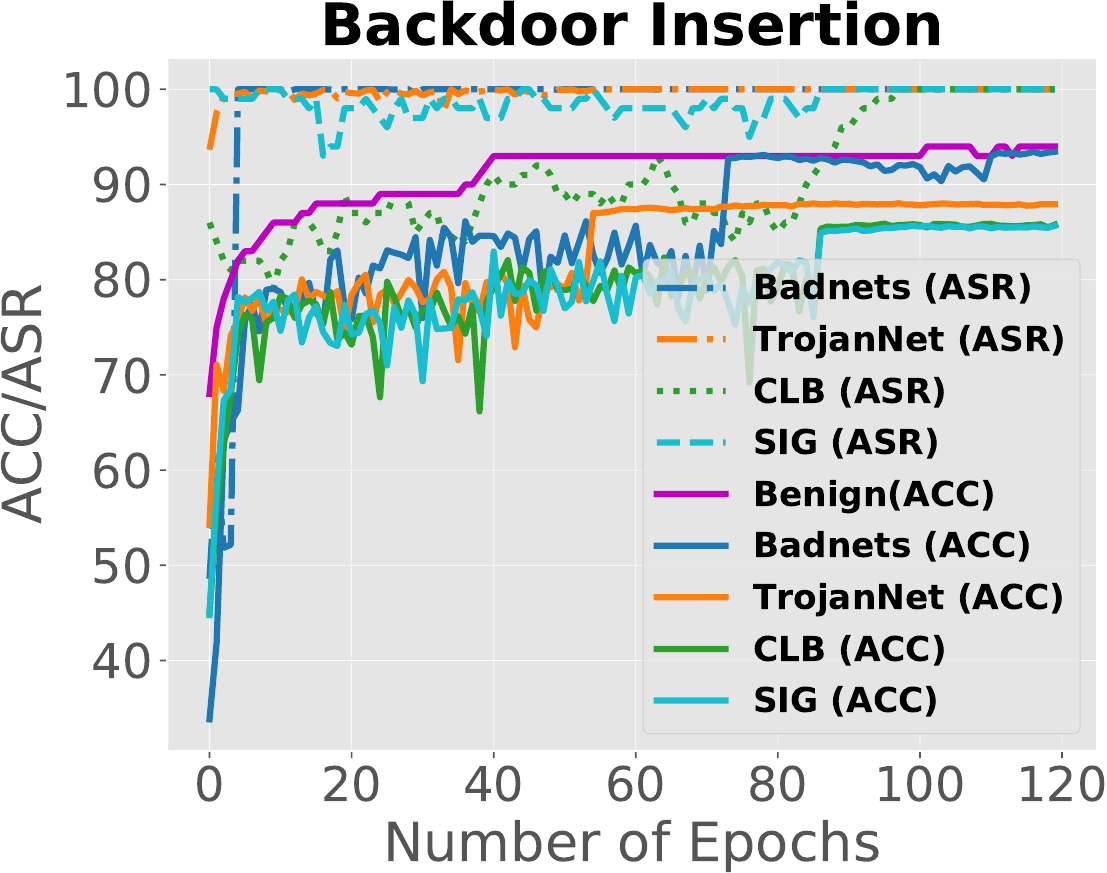}
    \caption{\scriptsize  ACC/ASR vs. Epochs}
    \label{fig:backdoor_asr}
\end{subfigure}

  \caption{ a \& b) \textbf{Eigen spectral density plots of loss Hessian} for benign and backdoor (TrojanNet~\citep{liu2017trojaning}) models. In each plot, the maximum eigenvalue ($\lambda_\mathsf{max}$), the trace of Hessian ($\mathsf{Tr}(H)$), clean test accuracy (ACC), and attack success rate (ASR) are also reported. Here, low $\lambda_\mathsf{max}$ and $\mathsf{Tr}(H)$ hints at the presence of a smoother loss surface, which often results in low ASR and high ACC. Compared to a benign model, a backdoor model tends to reach sharper minima, as shown by the larger range of eigenvalues (x-axis).  c) The convergence phenomena over the course of training. As the backdoor model converges to sharper minima, d) both ASR and ACC increase (around 80 epochs). We use the CIFAR10 dataset with a PreActResNet18~\citep{he2016identity} architecture for all evaluations.}
  \label{fig:eigen_spectral}
  
\end{figure*}

\section{Related Work}
This section discusses the existing works related to the backdoor attack methods and the defenses for backdoor attacks, as well as the works related to the smoothness analysis of DNN.

\noindent\textbf{Backdoor Attacks.} Backdoor attacks in deep learning models aim to manipulate the model to predict adversary-defined target labels in the presence of backdoor triggers in input while the model predicts true labels for benign input. \citep{manoj2021excess} formally analyzed DNN and revealed the intrinsic capability of DNN to learn backdoors. Backdoor triggers can exist in the form of dynamic patterns, a single pixel~\citep{tran2018spectral}, sinusoidal strips~\citep{barni2019new}, human imperceptible noise~\citep{zhong2020backdoor}, natural reflection~\citep{liu2020reflection}, adversarial patterns~\citep{zhang2021advdoor}, blending backgrounds~\citep{chen2017targeted}, hidden trigger~\citep{saha2020hidden}, etc. Based on target labels, existing backdoor attacks can generally be classified as poison-label or clean-label backdoor attacks. In poison-label backdoor attack, the target label of the poisoned sample is different from its ground-truth label, e.g., BadNets~\citep{gu2019badnets}, Blended attack~\citep{chen2017targeted}, SIG attack~\citep{barni2019new}, WaNet~\citep{nguyen2021wanet}, Trojan attack~\citep{liu2017trojaning}, and BPPA~\citep{wang2022bppattack}. Contrary to the poison-label attack, a clean-label backdoor attack doesn't change the label of the poisoned sample~\citep{turner2018clean, huang2022backdoor,zhao2020clean}. \cite{saha2022backdoor} studied backdoor attacks on self-supervised learning, and \cite{ahmed2023ssda} analyzed the effects of backdoor attacks on domain adaptation problems. All these attacks emphasized the severity of backdoor attacks and the necessity of efficient removal methods.

\noindent\textbf{Backdoor Defenses.} 
Defense against backdoor attacks can be classified as training time defenses and inference time defenses. Training time defenses include model reconstruction approach~\cite{zhao2020bridging,li2021neural}, poison suppression approach ~\cite{hong2020effectiveness,du2019robust,borgnia2021strong}, and pre-processing approaches~\cite{li2021anti,doan2020februus}. Although training time defenses are often successful, they suffer from huge computational burdens and are less practical in enforcing the defense pipeline while training, considering attacks during DNN outsourcing. Inference time defenses are mostly based on pruning approaches such as~\cite{ koh2017understanding, ma2019nic, tran2018spectral,diakonikolas2019sever,steinhardt2017certified}. 
Pruning-based approaches are typically based on model vulnerabilities to backdoor attacks---finding the backdoor-sensitive model parameters/neurons (often involving computationally expensive searching approaches) and subsequently pruning those sensitive parameters. For example, MCR~\cite{zhao2020bridging} and CLP~\cite{zheng2022data} analyzed node connectivity and channel Lipschitz constant to detect backdoor vulnerable neurons. 
Adversarial Neuron Perturbations (ANP)~\citep{wu2021adversarial} adversarially perturbs the DNN weights by employing and pruning bad neurons based on pre-defined thresholds. 
The disadvantage of such \emph{pre-defined thresholds} is that they can be dataset or attack-specific. ANP also suffers from performance degradation when the validation data size is too small. A more recent technique, Adversarial Weight Masking (AWM)~\citep{chai2022one}, has been proposed to circumvent the issues of ANP by replacing the adversarial weight perturbation module with an adversarial input perturbation module. Specifically, AWM solves a bi-level optimization for recovering the backdoor trigger distribution. Notice that both of these SOTA methods rely heavily on the computationally expensive adversarial search in the input or weight space, limiting their applicability in practical settings. I-BAU~\citep{zeng2021adversarial} also employs similar adversarial search-based criteria for backdoor removal. Recently, \cite{zhu2023enhancing} proposed a regular weight fine-tuning (FT) technique that employs popular sharpness-aware minimization (SAM)~\citep{foret2021sharpnessaware} optimizer to remove the effect of backdoor. However, a na\"ive addition of SAM to the FT leads to poor clean test accuracy after backdoor purification. Moreover, SAM is designed to train modes for general purposes involving two forward passes in each iteration, affecting the overall purification time of FT-SAM.     
RNP~\cite{li2023reconstructive} proposed to purify the backdoor in multiple stages--neuron unlearning, filter recovery (masking is required), and filter pruning. A concurrent work~\cite{karim2024augmented} proposed to fine-tune a backdoor model with MixUp augmented validation set to remove the backdoor. Compared to these existing defenses, our proposed approach is both computationally efficient (requires significantly less time) and performs better in removing backdoors and retaining clean accuracy. 

\noindent\textbf{Smoothness Analysis of DNN.} 
Having smoothness properties of an optimization algorithm is provably favorable for convergence~\citep{boyd2004convex}. Accordingly, there have been a substantial number of works on the smoothness analysis of the DNN training process, e.g.,~\citep{cohen2019certified, foret2021sharpnessaware,kwon2021asam}. \cite{jastrzebski2020break} showed that spectral norm and the trace of loss-Hessian could be used as proxies to measure the smoothness of a DNN model. However, to our knowledge, {\em no prior works either analyze the smoothness properties of a backdoor model or leverage these properties to design a backdoor purification technique}. One example could be the use of a second-order optimizer that usually helps the model converge to smooth minima. However, employing such an optimizer makes less sense considering the computational burden involving loss Hessian. A better alternative to a second-order optimizer is Fisher-information matrix-based natural gradient descent (NGD)~\citep{amari1998natural}. Nevertheless, NGD is also computationally expensive as it requires the inversion of Fisher-information matrix. 

\section{Threat Model}
This section presents the threat model under consideration by discussing the backdoor attack model and defense goal from a backdoor attack.

\noindent\textbf{Attack Model.} Our attack model is consistent with prior works related to backdoor attacks (e.g.,~\citep{gu2019badnets, chen2017targeted, nguyen2021wanet, wang2022bppattack}, etc.). We consider an adversary with the capabilities of carrying a backdoor attack on a DNN model, $f_\theta: \mathbb{R}^d \rightarrow \mathbb{R}^c$, by training it on a poisoned data set $\mathbb{D}_{\mathsf{train}} = \{X_{\mathsf{train}}, Y_{\mathsf{train}}\}$;  $X_{\mathsf{train}}=\{\xbf_i \}_{i=1}^{N_s}, Y_{\mathsf{train}}=\{y_i \}_{i=1}^{N_s}$ where $N_s$ is the total number of training samples. Here, $\theta$ is the parameters of the model, $d$ is the input data dimension, and $c$ is the total number of classes. Each input $\xbf\in X_{\mathsf{train}}$ is labeled as $y\in \{1,2,\cdots, c\}$. The data poisoning happens through a specific set of triggers that can only be accessed by the attacker. The adversary goal is to train the model in a way such that any triggered samples $\xbf_b = \xbf \oplus \delta \in \mathbb{R}^d$ will be classified to an adversary-set target label ${y_b}$, i.e., $\argmax(f_\theta(\xbf_b)) = {y_b}\neq y$. Here, $\xbf$ is a clean test sample, and $\delta \in \mathbb{R}^d$ represents the trigger pattern with the properties of $||\delta|| \leq \epsilon$; where $\epsilon$ is the trigger magnitude determined by its shape, size, and color. Note that \(\oplus\) operator can be any specific operation depending on how an adversary designed the trigger. We define the \emph{poison rate (PR)} as the ratio of poison and clean data in $\mathbb{D}_{\mathsf{train}}$. An attack is considered successful if the model behaves as $\argmax{(f_\theta(\xbf))} = y$ and $\argmax{(f_\theta(\xbf_b))} = {y_b}$, where, $y$ is the true label for $\xbf$. We use attack success rate (ASR) (i.e., predicting backdoored samples as adversary-set target label) to measure the effectiveness of a particular attack. Figure~\ref{fig:method_illustration}a illustrates the attack model under consideration in this work.


\noindent\textbf{Defense Goal.} We assume the defender has complete control over the pre-trained model $f_\theta(.)$, e.g., access to model parameters. Hence, we consider a defender with a task to purify the backdoor model $f_\theta(.)$ using a small clean validation set $\mathbb{D}_{\mathsf{val}} = \{X_{\mathsf{val}}, Y_{\mathsf{val}}\}$ (usually $0.1\sim10\%$ of the training data depending on the dataset). The goal is to repair the model such that it becomes immune to attack, i.e., $\argmax{(f_{\theta_p}(\xbf_b))} = {y}$, where $f_{\theta_p}$ is the purified model. Note that the defense method must retain clean accuracy of $f_{\theta}(.)$ for benign inputs even if the model has no backdoor.

\begin{figure*}
    \centering
    \includegraphics[width=\linewidth, trim=0.2in 9.35in 0.35in 0.12in, clip]{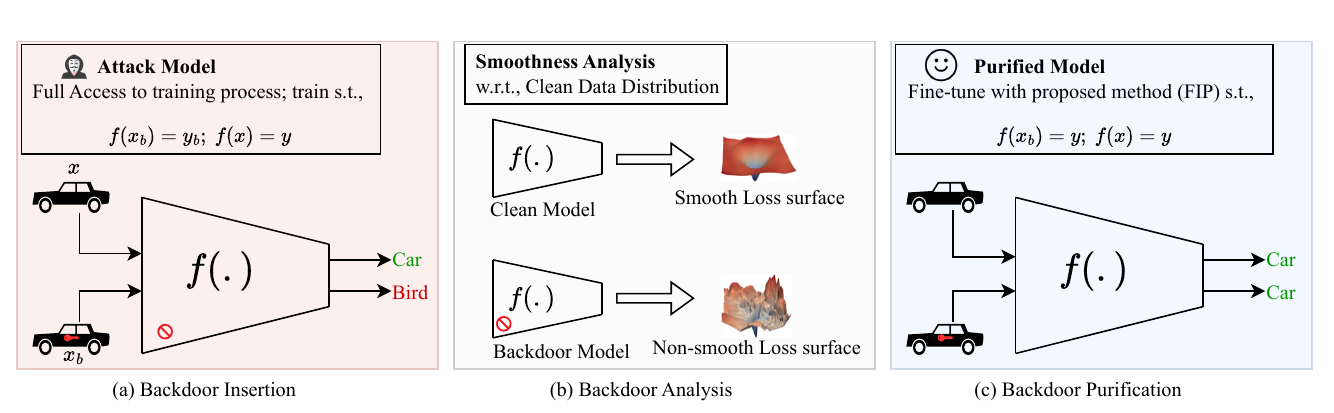}
    \caption{An illustration of the proposed backdoor model analysis and corresponding purification method. In Figure~\ref{fig:method_illustration}a, we assume a standard backdoor insertion scenario where the attacker has full control over the training process. Figure~\ref{fig:method_illustration}b illustrates our observation following the smoothness analysis of a pre-trained model. Figure~\ref{fig:method_illustration}c shows that a model purified via the proposed method FIP is immune to backdoor trigger and can predict true label in the presence of a backdoor trigger. Note, figures to illustrate loss surface (in Figure~\ref{fig:method_illustration}b) are taken from~\cite{foret2021sharpnessaware}.}
    \label{fig:method_illustration}
\end{figure*}

\section{Smoothness Analysis of Backdoor Models}\label{sec:smoothAnalysis}
In this section, we analyze the loss surface geometry of benign and backdoor models. To study the loss curvature properties of different models, we aim to analyze the Hessian of loss (loss-Hessian), $H = \nabla^2_{\theta} \mathcal{L}$, where $\mathcal{L}$ is computed using the {training samples}. The spectral decomposition of symmetric square matrix \(H\) is $H{=[h_{ij}]}=Q\Lambda Q^T$, where $\Lambda =  \mathsf{diag}(\lambda_1, \lambda_2, \cdots, \lambda_N)$ is a diagonal matrix that contains the eigenvalues of $H$ and $Q = [q_1 q_2 \cdots q_N]$, where $q_i$ is the $i^{th}$ eigenvector of H. As a measure for smoothness, we take the spectral norm of $H$, $\sigma (H)= \lambda_1 = \lambda_{max}$, and the trace of the Hessian, $\mathsf{Tr}(H) = \sum_{i=1}^{i=N} h_{ii}$. \emph{Low values for these two proxies} indicate the presence of a \emph{highly smooth loss surface} \citep{jastrzebski2020break}. The Eigen Spectral density plots in Fig.~\ref{fig:ls-cln} and~\ref{fig:ls-bd} elaborate on the optimization of benign and backdoor models. From the comparison of $\lambda_{max}$ and $\mathsf{Tr}(H)$, it can be conjectured that optimization of a benign model leads to a smoother loss surface. Since the main difference between a benign and a backdoor model is that the latter needs to learn two different data distributions (clean and poison), we state the following observation:

\noindent\textbf{Observation 1.} \textit{Having to learn two different data distributions, a backdoor model reaches a sharper minima, i.e., large $\sigma(H)$ and $\mathsf{Tr}(H)$, as compared to the benign model.}      

We support our observation with empirical evidence presented in Fig.~\ref{fig:backdoor_eigens}~and~\ref{fig:backdoor_asr}. Here, we observe the convergence behavior for 4 different attacks over the course of training. Compared to a benign model, the loss surface of a backdoor \emph{becomes much sharper as the model becomes well optimized for both distributions}, i.e., high ASR and high ACC. Backdoor and benign models are far from being well-optimized at the beginning of training. The difference between these models is prominent once the model reaches closer to the final optimization point. As shown in Fig.~\ref{fig:backdoor_asr}, the training becomes reasonably stable after 100 epochs with ASR and ACC near saturation level. Comparing $\lambda_{\mathsf{max}}$ of benign and all backdoor models after 100 epochs, we notice a sharp contrast in Fig.~\ref{fig:backdoor_eigens}. This validates our claim on loss surface smoothness of benign and backdoor models in Observation 1. 
All of the backdoor models have high attack success rates (ASR) and high clean test accuracy (ACC), which indicates that the model had learned both distributions, providing additional support for Observation 1. Similar phenomena for different attacks, datasets, and architectures have been observed; details are provided in Section~\ref{sec:ab-smooth} and \textbf{Appendix~\ref{sec:more_smoothness_analysis}}. 




\noindent\textbf{Theoretical Justification.} 
We discuss the smoothness of backdoor model loss considering the Lipschitz continuity of the loss gradient. Let us first define the \(L-\)Lipschitzness and \(L-\)Smoothness of a function as follows:

\begin{definition}
   \emph{[\(L-\)Lipschitz]} A function \(f(\theta)\) is \(L-\)Lipschitz on a set \(\Theta\), if there exists a constant \(0\leq L<\infty\) such that,
    \[
    ||f(\theta_1) - f(\theta_2)|| \leq L ||\theta_1 -\theta_2||, ~\forall \theta_1, \theta_2 \in \Theta
    \]
\end{definition}

\begin{definition}
    \emph{[\(L-\)Smooth]} A function \(f(\theta)\) is \(L-\)Smooth on a set \(\Theta\), if there exists a constant \(0\leq L <\infty\) such that,
    \[
    ||\nabla_\theta f(\theta_1) - \nabla_\theta f(\theta_2) || \leq L ||\theta_1 - \theta_2||, ~\forall \theta_1, \theta_2 \in \Theta
    \]
\end{definition}

Following prior works~\citep{sinha2018certifiable,liu2020loss,kanai2023relationship} related to the smoothness analysis of the loss function of DNN, we assume the following conditions on the loss:

\begin{assumption}\label{assumption:first}
  The loss function \(\ell(\xbf, \theta)\) satisfies the following inequalities,
  \begin{align}
     & ||\ell(\xbf, \theta_1) - \ell(\xbf, \theta_2)|| \leq K ||\theta_1 - \theta_2||  \\
      & ||\nabla_\theta\ell(\xbf, \theta_1) - \nabla_\theta\ell(\xbf, \theta_2) || \leq L ||\theta_1 - \theta_2|| \label{eqn.assump1}
  \end{align}
  where \(0\leq K<\infty\), \(0\leq L <\infty\), \(\forall \theta_1, \theta_2 \in \Theta\), and \(\xbf\) is any training sample (i.e., input).
\end{assumption}

Using the above assumptions, we state the following theorem:


\begin{theorem}\label{thm:smoothness}
 If the gradient of loss corresponding to clean and poison samples are  $L_c-$Lipschitz and $L_b-$Lipschitz, respectively, then the overall loss (i.e., loss corresponding to training samples {with their ground-truth labels}) of backdoor model is $L_b-$Smooth and $L_c < L_b$. 
\end{theorem}
Theorem~\ref{thm:smoothness} describes the nature of the overall loss of backdoor model resulting from both clean and poison samples. 

To infer the characteristics of smoothness of overall loss from Theorem~\ref{thm:smoothness}, let us consider the results from Keskar~et~al.~\cite{keskar2016large}. \citep{keskar2016large} shows that the loss-surface smoothness of $\mathcal{L}$ for differentiable $\nabla_{\theta} \mathcal{L}$ can be related to $L-$Lipschitz of $\nabla_{\theta} \mathcal{L}$ as 
\begin{equation}\label{eqn-lsigma}
 \sup_{\theta} \sigma(\nabla^2_{\theta} \mathcal{L}) \leq L.   
\end{equation}

Therefore, inferring from Eq.~(\ref{eqn-lsigma}), Theorem~\ref{thm:smoothness} supports our empirical results related to backdoor and benign model optimization as larger Lipschitz constant implies sharper minima and the Lipchitz constant of backdoor model is \textit{strictly greater} than benign model (i.e., $L_c < L_b$).

\textbf{Remark. }As per the smoothness analysis of backdoor model, applying sharpness-aware optimization techniques would remove the effect of the backdoor from a model. One such ad-hoc method (FT-SAM~\cite{zhu2023enhancing}) is enhancing FT-based backdoor removal with SAM~\cite{foret2021sharpnessaware} optimizer. However, SAM is a general-purpose optimization technique requiring a double forward pass in each iteration, which affects the time required to  
purify the model (ref. Figure~\ref{fig:run-time}), and a lack of knowledge of model parameters toward clean data distribution affects the clean accuracy of the purified model (ref. Table~\ref{tab:main},~\ref{tab:multi_label},~\ref{tab:action_rec},~\ref{tab:3d_point_cloud}), necessitating a backdoor-specific method discussed in the following section.

\section{Fisher Information guided Purification (FIP)} \label{method}
Our proposed backdoor purification method---Fisher Information guided Purification (FIP) consists of two novel components: (i) \textit{Backdoor Suppressor} for backdoor purification and (ii) \textit{Clean Accuracy Retainer} to preserve the clean test accuracy of the purified model.

\noindent\textbf{Backdoor Suppressor.} Let us consider a backdoor model $f_\theta: \mathbb{R}^d \rightarrow \mathbb{R}^c$ with parameters $\theta \in \mathbb{R}^N$ to be fitted (fine-tuned) with input (clean validation) data $\{(\xbf_i, y_i)\}_{i=1}^{|\mathbb{D}_{\mathsf{val}}|}$ from an input data distribution $P_{\xbf,y}$, where $\xbf_i \in X_{\mathsf{val}}$ is an input sample and $y_i \in Y_{\mathsf{val}}$ is its label. We fine-tune the model by solving the following:
\begin{equation}\label{vanillaFT}
    \argmin_{\theta} ~\mathcal{L}(\theta),
\end{equation}
where, \(\mathcal{L}(\theta)  = \mathcal{L}(y,f_\theta(\xbf)) = \sum_{(x_i, y_i)\in \mathbb{D}_{\mathsf{val}}} \left (-\log[f_\theta(\xbf_i)]_{y_i}\right)\) 
is the empirical full-batch cross-entropy (CE) loss. Here, $[f_\theta(\xbf)]_y$ is the $y^{th}$ element of $f_\theta(\xbf)$. Our smoothness study in Section~\ref{sec:smoothAnalysis} showed that backdoor models are optimized to sharper minima as compared to benign models. Intuitively, re-optimizing the backdoor model to a smooth minima would effectively remove the backdoor. However, the \emph{vanilla fine-tuning} objective presented in Eq.~(\ref{vanillaFT}) is not sufficient to effectively remove the backdoor as we are not using any smoothness constraint or penalty.      

To this end, we propose to regularize the spectral norm of loss-Hessian \(\sigma(H)\) in addition to minimizing the cross entropy-loss $\mathcal{L}(\theta)$ as follows,
\vspace{-0.7mm}
\begin{equation}\label{eqn:ce_loss_constraint}
    \argmin_{\theta}  ~\mathcal{L}(\theta) + \sigma(H).
\end{equation}
By explicitly regularizing the $\sigma(H)$, we intend to obtain smooth optimization of the backdoor model. However, the calculation of $H$, in each iteration of training has a huge computational cost. Given the objective function is minimized iteratively, it is not feasible to calculate the loss Hessian at each iteration. Additionally, the calculation of \(\sigma(H)\) will further add to the computational cost. Instead of directly computing $H$ and \(\sigma(H)\), we analytically derived a computationally efficient 
upper-bound of \(\sigma(H)\) in terms of 
\(\mathsf{Tr}(H)\) as follows,


\begin{lemma}\label{lemma}
The spectral norm of loss-Hessian \(\sigma(H)\) is upper-bounded by
\(
\sigma(H) \leq \mathsf{Tr}(H) \approx \mathsf{Tr}(F)
\),
where 
\begin{equation}\label{F}
F = \Ee_{(\xbf,y)\sim P_{\xbf,y}}\left[\nabla_\theta\mathsf{log}[f_\theta(\xbf)]_y\cdot (\nabla_\theta\mathsf{log}[f_\theta(\xbf)]_y)^T\right]    
\end{equation}
is the Fisher-Information Matrix (FIM). 
\end{lemma}
\begin{proof}
The inequality \(\sigma(H)\leq \mathsf{Tr}(H)\) follows trivially as \(\mathsf{Tr}(H)\) of symmetric square matrix \(H\) is the sum of all eigenvalues of \(H\),~\(\mathsf{Tr}(H) = \sum_{\forall i} \lambda_i \geq \sigma(H)\). The approximation of \(\mathsf{Tr}(H)\) using \(\mathsf{Tr}(F)\) follows the fact that $F$ is negative expected Hessian of log-likelihood and used as a proxy of Hessian $H$~\citep{amari1998natural}. 
\end{proof}


Following Lemma~\ref{lemma}, we adjust our objective function described in Eq.~(\ref{eqn:ce_loss_constraint}) to
\begin{equation}\label{eqn:loss}
    \argmin_{\theta}  ~\mathcal{L}(\theta) + \eta_F \mathsf{Tr}(F),
\end{equation}
where $\eta_F$ is a regularization constant. Optimizing Eq.~(\ref{eqn:loss}) will force the backdoor model to converge to smooth minima.  Even though this would purify the backdoor model, the clean test accuracy of the purified model may suffer due to significant changes in $\theta$. To avoid this, we propose an additional but much-needed regularizer to preserve the clean test performance of the original model. 


\noindent\textbf{Clean Accuracy Retainer.} In a backdoor model, some neurons or parameters are more vulnerable than others. The vulnerable parameters are believed to be the ones that are sensitive to poison or trigger data distribution \citep{wu2021adversarial}. In general, CE loss does not discriminate whether a parameter is more sensitive to clean or poison distribution. Such lack of discrimination may allow drastic or unwanted changes to the parameters responsible for learned clean distribution. This usually leads to sub-par clean test accuracy after purification, and it requires additional measures to fix this issue. To this end, we introduce a novel \textit{clean distribution aware regularization} term as,
\[
L_{r} = \sum_{\forall i} \mathsf{diag} (\bar{F})_i \cdot (\theta_{i} - \bar{\theta}_{i})^2.
\]
Here, $\bar{\theta}$ is the parameter of the initial backdoor model and remains fixed throughout the purification phase. $\bar{F}$ is FIM computed only once on $\bar{\theta}$ and also remains unchanged during purification. 
$L_{r}$ is a product of two terms: i) an error term that accounts for the deviation of $\theta$ from $\bar{\theta}$;
ii) a vector, $\mathsf{diag}(\bar{F})$, consisting of the diagonal elements of FIM $(\bar{F})$. As the first term controls the changes of parameters \emph{w.r.t.} $\bar{\theta}$, it helps the model to remember the already learned distribution. However, learned data distribution consists of both clean and poison distribution. To explicitly force the model to remember the \emph{clean distribution}, we compute $\bar{F}$ using a \emph{clean} validation set; with a similar distribution as the learned clean data. Note that $\mathsf{diag}(\bar{F})_i$ represents the square of the derivative of log-likelihood of clean distribution \emph{w.r.t.} $\bar{\theta}_{i}$, $[\nabla_{\bar{\theta}_{i}}\mathsf{log}[f_\theta(\xbf)]_y]^2$~(ref. Eq.~(\ref{F})). In other words, $\mathsf{diag}(\bar{F})_i$ is the measure of the importance of $\bar{\theta}_{i}$ towards remembering the learned clean distribution. If $\mathsf{diag}(\bar{F})_i$ has higher importance, we allow minimal changes to $\bar{\theta}_{i}$ over the purification process.
This careful design of such a regularizer significantly improves clean test performance.


Finally, to purify the backdoor model as well as to preserve the clean accuracy, we formulate the following objective function as
\begin{equation}\label{eqn:loss2}
    \argmin_{\theta}  ~\mathcal{L}(\theta) + \eta_F \mathsf{Tr}(F) + \frac{\eta_r}{2}L_{r},
\end{equation}
where $\eta_F$ and $\eta_r$ are regularization constants.

\begin{table*}[t]
\footnotesize
\centering
 \caption{ Removal Performance (\%) of FIP and other defenses in \textbf{single-label settings}. Backdoor removal performance, i.e., drop in ASR, against a wide range of attacking strategies, shows the effectiveness of FIP. We use a poison rate of 10\% for CIFAR10 and 5\% for ImageNet. For GTSRB, the poison rate is $10\%$. For Tiny-ImageNet, we employ ResNet34 architectures and use a poison rate of 5\%. For Tiny-ImageNet and ImageNet, we report performance on successful attacks (ASR $\sim$ 100\%) only. Average drop ($\downarrow$) indicates the \% changes in ASR/ACC compared to the baseline, i.e., \emph{No Defense}. A higher ASR drop and lower ACC drop are desired for a good defense.}

\scalebox{1}{
\begin{tabular}{c|c|cc|cc|cc|cc|cc|cc|cc}
\toprule
\multirow{2}{*}{Dataset} & Method & \multicolumn{2}{c|}{\begin{tabular}[c|]{@{}c@{}}No Defense\end{tabular}} & \multicolumn{2}{c|}{ANP} & \multicolumn{2}{c|}{I-BAU} & \multicolumn{2}{c|}{AWM} & \multicolumn{2}{c|}{FT-SAM} & \multicolumn{2}{c|}{RNP}  &  \multicolumn{2}{|c}{FIP  (Ours)}\\ \cmidrule{2-16}
 & Attacks & ASR &ACC & ASR &ACC & ASR &ACC & ASR &ACC & ASR &ACC & ASR &ACC & ASR &ACC  \\ \cmidrule{1-16}
\multirow{12}{*}{\footnotesize CIFAR-10} 
  &  \emph{Benign} & 0 & 95.21 & 0 & 92.28 & 0 & 93.98&0&93.56&0& 93.80 & 0 & 93.16 & 0 & \textbf{94.10}  \\
 &   Badnets & 100 & 92.96 & 6.87 & 86.92 & 2.84 & 85.96&9.72&87.85&3.74&86.17 & 2.75 & 88.46 & \textbf{1.86} &\textbf{ 89.32} \\
 &   Blend & 100 & 94.11  & 5.77 & 87.61 & 7.81 & 89.10&6.53&89.64&2.13&88.93 & 0.91 & 91.53 &\textbf{ 0.38} & \textbf{92.17} \\
  &  Troj-one & 100 & 89.57& 5.78 & 84.18& 8.47 & {85.20} & 7.91&\textbf{87.50}&5.41&86.45 & 3.84 & 87.39   &\textbf{2.64} & 87.21 \\
  &  Troj-all & 100 & 88.33 & 4.91 & 84.95& 9.53 & 84.89 & 9.82&84.97&3.42&84.60 & 4.02 & 85.80 & \textbf{2.79} & \textbf{86.10} \\
 &   SIG & 100 & 88.64 & 2.04 & 84.92 & 1.37 & 83.60 &2.12&83.57&0.73&83.38& \textbf{0.51} & 86.46 &  0.92 & \textbf{86.73}  \\
  &  Dyn-one & 100 & 92.52 & 8.73 & 88.61 & 7.78 & 86.26 &6.48&88.16&3.35&88.41& 8.61 & 90.05 & \textbf{1.17} & \textbf{90.97} \\
 &  Dyn-all & 100 & 92.61 & 7.28 & 88.32 & 8.19 & 84.51 &6.30&89.74&2.46&87.72& 10.57 & 90.28 & \textbf{1.61} & \textbf{91.19}\\
 &  CLB & 100 & 92.78 & 5.83 & 89.41 & 3.41 & 85.07 &5.78&86.70&\textbf{1.89}&87.18& 6.12 & 90.38 & 2.04 & \textbf{91.37}\\
 &  CBA & 93.20 & 90.17 & 25.80 & 86.79 &24.11&85.63&26.72&85.05&18.81&85.58& 17.72 & 86.40 & \textbf{14.60}&\textbf{86.97}\\
 & FBA & 100&90.78 &11.95 &86.90 &16.70&87.42&10.53&87.35&10.36&87.06& 9.48 & \textbf{87.63} & \textbf{6.21}&87.30\\
 & LIRA & 99.25 & 92.15 & 6.34 & 87.47 & 8.51 & 89.61 & 6.13 & 87.50 & 3.93 & 88.70 & 11.83 & 87.59 &\textbf{2.53}&\textbf{89.82} \\ 
 & WaNet&98.64&92.29&9.81&88.70&7.18&89.24&8.72&85.94&2.96&87.45& 8.10 & \textbf{90.26} &\textbf{2.38}&89.67\\
 & ISSBA&99.80&92.80&10.76&85.42&9.82&89.20&9.48&88.03&\textbf{3.68}&88.51& 7.58 & 88.62 &  4.24&\textbf{90.18}\\
& BPPA & 99.70 & 93.82 & 13.94 & 89.23 & 10.46&88.42&9.94&89.68&7.40&89.94& 9.74 & 91.37 & \textbf{5.14}&\textbf{92.84}\\
 \cmidrule{2-16} 
& {Avg. Drop} & - &-  &  $90.34\downarrow$ & $4.57\downarrow$ &  $90.75\downarrow$ &  $4.96\downarrow$ &$90.31\downarrow$&$4.42\downarrow$&$94.29\downarrow$&$4.53\downarrow$ & $92.06\downarrow$ & $2.95\downarrow$ & {\textbf{95.86}} $\downarrow$ &{\textbf{2.28}} $\downarrow$ \\ 
 \cmidrule{1-16}
\multirow{9}{*}{\footnotesize GTSRB} 
  & \emph{Benign} & 0 & 97.87 & 0&93.08 & 0 & 95.42&0&96.18&0&95.32  & 0  & 95.64 & 0   &\textbf{96.76}\\
  & Badnets & 100 & 97.38 & 7.36 & 88.16 & 2.35 & 93.17 &2.72&{93.55}&2.84&93.58& 3.93 & 94.57 & \textbf{0.24} & \textbf{96.11} \\
  & Blend & 100 & 95.92 & 9.08&89.32 & 5.91 & 93.02 &4.13&{92.30}&4.96&92.75& 5.85 & 93.41 &\textbf{2.41} & \textbf{94.16}  \\
  & Troj-one & 99.50 & 96.27 & 6.07&90.45 & 3.81 & 92.74&3.04&93.17&2.27&93.56 & 4.18 & 93.60 & \textbf{ 1.21} &  \textbf{95.18} \\
  &  Troj-all & 99.71 & 96.08 & 6.48&89.73& 5.16 & 92.51 &2.79&91.28&1.94&92.84&  4.86 & 92.08 &\textbf{1.58} & \textbf{93.77} \\
  &  SIG & 97.13 & 96.93 & 5.93 &91.41 & 8.17 & 91.82 &\textbf{2.64}&91.10&5.32&92.68 & 6.44 & 93.79 &  2.74 & \textbf{95.08}  \\
  &  Dyn-one & 100 & 97.27 & 6.27&91.26 & 5.08 & 93.15 &5.82&{92.54}&1.89&93.52& 7.24 & 93.95 &\textbf{1.51} & \textbf{95.27}  \\
  &  Dyn-all & 100 & 97.05 & 8.84 & 90.42 & 5.49 & 92.89 &4.87&93.98&2.74&93.17& 8.17 & 94.74 & \textbf{1.26} & \textbf{96.14}\\
  & WaNet & 98.19 & 97.31 & 7.16 & 91.57 & 5.02 & 93.68 &4.74 &93.15 & 3.35 & 94.61 & 5.92 & 94.38 & \textbf{1.43} & \textbf{95.86}  \\
  & ISSBA & 99.42 & 97.26 & 8.84 & 91.31 & 4.04 & 94.74 & 3.89 & 93.51 & \textbf{1.08} & 94.47 & 4.80 & 94.27 & 1.20 & \textbf{96.24} \\
  & LIRA & 98.13 & 97.62 & 9.71 & 92.31 & 4.68 & 94.98 & 3.56 & 93.72 & 2.64 & 95.46 & 5.42 & 93.06 &\textbf{1.52}&\textbf{96.54} \\ 
  & BPPA & 99.18 & 98.12 & 12.14 & 93.48 & 9.19 & 93.79&8.63&92.50&5.43&94.22 & 7.55 & 94.69 & \textbf{3.35}& \textbf{96.47} \\ 
\cmidrule{2-16}
 &   {Avg. Drop} & - & -  & 91.03 $\downarrow$ & 6.16 $\downarrow$ & 94.12$\downarrow$ & 3.70 $\downarrow$  &$94.95\downarrow$&$4.26\downarrow$&$96.07\downarrow$&$3.58\downarrow$&  93.35 $\downarrow$ & 3.15 $\downarrow$ &{\textbf{97.51}} $\downarrow$ &{\textbf{1.47}} $\downarrow$\\ 
\cmidrule{1-16}
\multirow{7}{*}{\footnotesize Tiny-ImageNet} 
&\emph{Benign}&0&62.56&0&58.20&0&59.29&0&59.34&0&59.08& 0 & 58.14 &0&\textbf{59.67}\\
&Badnets&100&59.80&8.84&53.58&7.23&54.41&13.29&54.56&\textbf{2.16}&54.81& 4.63 & 55.96 & 2.34&\textbf{57.84}\\
&Trojan&100&59.16&11.77&52.62&7.56&53.76&5.94&{54.10}&8.23&54.28& 5.83 & 54.30 &\textbf{3.38}&\textbf{55.87}\\
&Blend&100&60.11&7.18&52.22&9.58&54.70&7.42&54.19&4.37&54.78& 4.08 & 55.47 & \textbf{1.58}&\textbf{57.48}\\
&SIG&98.48&60.01&12.02&52.18&11.67&53.71&7.31&{53.72}&4.68&54.11& 6.71 & 55.22 & \textbf{2.81}&\textbf{55.63}\\
&CLB&97.71&60.33&10.61&52.68&8.24&54.18&10.68&53.93&3.52&54.02 & 4.87 & 56.92 & \textbf{2.46}&\textbf{57.40} \\ 
&Dynamic& 100& 60.54& 8.36 &52.57& 9.56&54.03 &6.26&54.19&4.26&54.21& 7.23 & 55.80 & \textbf{2.24}&\textbf{57.96}\\
&{WaNet}&99.16&60.35&8.02&52.38&8.45&54.65&8.43&{54.32}&7.84&54.04& 5.66 & 55.19 & \textbf{4.48}&\textbf{56.21}\\
&{ISSBA}&98.42&60.76&{6.26}&53.41&10.64&54.36&11.47&53.83&\textbf{3.72}&55.32& 8.24 & 55.35& {4.25}&\textbf{57.35} \\ 
&{BPPA}&98.52&60.65&11.23&53.03&9.62&54.63&8.85&53.03&5.34&54.48& 10.86 & 56.32 & \textbf{3.89}&\textbf{57.39} \\
\cmidrule{2-16} 
& {Avg. Drop} & - & -  & 89.77 $\downarrow$ & 7.44 $\downarrow$ & 92.97$\downarrow$ & 5.92 $\downarrow$ &$90.29\downarrow$&$6.98\downarrow$&$93.91\downarrow$&$5.85\downarrow$ & 92.69 $\downarrow$ &4.58 $\downarrow$ & {\textbf{96.10}} $\downarrow$ &{\textbf{3.08}} $\downarrow$\\ 
\cmidrule{1-16}
\multirow{12}{*}{\footnotesize ImageNet} 
&\emph{Benign}&0&77.06&0&73.52&0&71.85&0&74.21&0&71.63& 0 & 75.20 & 0&\textbf{75.51}\\
&Badnets&99.24&74.53&6.97&69.37&6.31&68.28&\textbf{0.87}&69.46&1.18&70.44& 7.58 & 70.49 &  1.61&\textbf{71.46}\\
&Troj-one&99.21&74.02&7.63&69.15&7.73&67.14&5.74&69.35&2.86&70.62& 2.94 & 72.17 & \textbf{2.16}&\textbf{72.47}\\
&Troj-all&97.58&74.45&9.18&69.86&7.54&68.20&6.02&69.64&3.27&69.85& 4.81 & 71.45 & \textbf{2.38}&\textbf{72.63}\\
&Blend&100&74.42&9.48&70.20&7.79&68.51&7.45&68.61&2.15&70.91& 5.69 & 70.24 & \textbf{1.83}&\textbf{72.02}\\
&SIG&94.66&74.69&8.23&69.82&4.28&67.08&5.37&70.02&2.47&69.74& 4.36 & 70.73 & \textbf{0.94}&\textbf{72.86}\\
&CLB&95.08&74.14&8.71&69.19&4.37&68.41&7.64&69.70&1.50&70.32& 9.44 & 71.52 & \textbf{1.05}&\textbf{72.75}\\
&{Dyn-one}&98.24&74.80&6.68&69.65&8.32&69.61&8.62&70.17&4.42&70.05& 12.56 & 70.39 &  \textbf{2.62}&\textbf{71.91}\\
&{Dyn-all}&98.56&75.08&13.49&70.18&9.82&68.92&12.68&70.24&4.81&69.90& 14.18 & 69.47 & \textbf{3.77}&\textbf{71.62}\\
&{LIRA}&96.04&74.61&12.86&69.22&12.08&69.80&13.27&69.35&3.16&\textbf{71.38}& 12.31 & 70.50 & \textbf{2.62}&70.73\\
&{WaNet}&97.60&74.48&9.34&68.34&5.67&69.23&6.31&70.02&\textbf{2.42}&69.20& 7.78 & 71.62 & 2.71&\textbf{72.58}\\
&{ISSBA}&98.23&74.38&9.61&68.42&4.50&68.92&8.21&69.51&3.35&70.51& 9.74 & 70.81 & \textbf{2.86}&\textbf{72.17} \\ 
\cmidrule{2-16}  
 &  {Avg. Drop} & - & -  & 88.38 $\downarrow$ &  5.11$\downarrow$ & 90.54$\downarrow$ & 5.95 $\downarrow$ &$90.21\downarrow$&$4.77\downarrow$&$94.80\downarrow$& $4.24\downarrow$& $89.37\downarrow$ & $3.66\downarrow$ & {\textbf{95.44}} $\downarrow$ &{\textbf{2.40}} $\downarrow$\\
\bottomrule
\end{tabular}}
\vspace{-1mm}
 \label{tab:main}
\end{table*}

\begin{table*}[t]
    \centering
    \caption{Performance analysis for the \textbf{multi-label backdoor attack}~\citep{chen2023clean}. Mean average precision (mAP) and ASR of the model, with and without defenses, have been shown.}    
    \scalebox{0.85}{
    \begin{tabular}{l|cc|cc|cc|cc|cc|cc|cc|cc}
        \toprule
        \multirow{2}{*}{Dataset} & \multicolumn{2}{c|}{No defense}  & \multicolumn{2}{c|}{FP} & \multicolumn{2}{c|}{Vanilla FT} & \multicolumn{2}{c|}{MCR}  & \multicolumn{2}{c|}{NAD}  & \multicolumn{2}{c|}{FT-SAM} & \multicolumn{2}{c|}{RNP}  & \multicolumn{2}{|c}{FIP (Ours)} \\ \cmidrule{2-17}
         & ASR & mAP & ASR & mAP & ASR & mAP & ASR & mAP & ASR & mAP & ASR & mAP & ASR & mAP & ASR & mAP \\ \cmidrule{1-17}
        VOC07 & 86.4 & 92.5 & 61.8 & 87.2 & 19.3 & 86.9 & 28.3 & 86.0 & 26.6 & 87.3 & 17.9 & 87.6 & 19.3 & 86.8 & \textbf{16.1} & \textbf{89.4} \\
        VOC12 & 84.8 & 91.9 & 70.2 & 86.1 & 18.5 & 85.3 & 20.8 & 84.1 & 19.0 & 84.9 & 15.2 & 85.7 & 14.6 & 87.1 & \textbf{13.8} & \textbf{88.6} \\
        MS-COCO & 85.6 & 88.0 & 64.3 & 83.8 & 17.2 & 84.1 & 24.2 & 82.5 & 22.6 & 83.4 & \textbf{14.3} & 83.8 & 16.2 & 84.4 & 15.0 & \textbf{85.2} \\
        \bottomrule
    \end{tabular}}
    \label{tab:multi_label}
\end{table*}

\subsection{Fast FIP (f-FIP)} 
In general, any backdoor defense technique is evaluated in terms of removal performance and the time it takes to remove the backdoor, i.e., purification time. It is desirable to have a very short purification time. To this aim, we introduce a few unique modifications to FIP to perform fine-tuning in a more compact space than the original parameter space. 

Let us represent the weight matrices for model with $L$ number of  layers as ${\theta}= [{\theta}_1, {\theta}_2, \cdots {\theta}_L]$. We take spectral decomposition of $\theta_i ={U}_i {\Sigma}_i {V}_i^T \in \mathbb{R}^{M \times N}$, where ${\Sigma}_i=\mathsf{diag}({\sigma}_i)$ and ${\sigma}_i=[{\sigma}_i^1, {\sigma}_i^2, \cdots, {\sigma}_i^M]$ are singular values arranged in descending order. The spectral shift of the parameter space is defined as the difference between singular values of original ${\theta}_i$ and the updated $\hat{\theta}_i$ can be expressed as $\delta_i = [\delta_i^1, \delta_i^2, \cdots, \delta_i^{M}]$. Here, $\delta_i^j$ is the difference between individual singular value ${\sigma}_i^j$. Instead of  updating $\theta$, we update the total spectral shift $\delta = [\delta_1, \delta_2, \cdots, \delta_L]$ as,
\begin{equation}\label{eqn:loss_spectral}
    \argmin_{\delta} ~\mathcal{L}(\delta) + \eta_F \mathsf{Tr}(F) + \frac{\eta_r}{2}L_{r}
\end{equation}

Here, we keep the singular vectors (${U}_i$,${V}_i$) frozen during the updates. We obtain the updated singular values as, $\widehat{\Sigma}_i = \text{diag}(\text{ReLU}({\sigma}_i+\delta_i))$ which gives us the updated weights $\hat{\theta_i} = {U}_i \widehat{\Sigma}_i {V}_i^T$. Fine-tuning the model in the spectral domain reduces the number of tunable parameters and purification time significantly (see Figure~\ref{fig:run-time}).

\noindent{\bf Numerical Example related to f-FIP.} 
Let us consider a convolution layer with a filter size of $5\times5$, an output channel of 256, and an input channel of 128. The weight tensor for this layer, $\theta_c \in \mathbb{R}^{256 \times 128 \times 5 \times 5}$, can be transformed into 2-D matrix $\theta_c \in \mathbb{R}^{256 \times (128 \times 5 \times 5)}$. If we take the SVD of this 2D matrix, we only have 256 parameters ($\sigma$) to optimize instead of 8,19,200 parameters. For this particular layer, we reduce the tunable parameter by 3200$\times$ as compared to vanilla fine-tuning. \emph{By reducing the number of tunable parameters, fast FIP significantly improves the computational efficiency of FIP.} In the rest of the paper, we use f-FIP and FIP interchangeably unless otherwise stated.

\begin{table*}[t]
    \centering
    \caption{Performance analysis for action recognition task where we choose 2 video datasets for evaluation.}    
    \scalebox{0.8}{
    \begin{tabular}{l|cc|cc|cc|cc|cc|cc|cc|cc|cc}
        \toprule
        \multirow{2}{*}{Dataset} & \multicolumn{2}{c|}{No defense} & \multicolumn{2}{c|}{MCR} & \multicolumn{2}{c|}{NAD} & \multicolumn{2}{c|}{ANP} & \multicolumn{2}{c|}{I-BAU} & \multicolumn{2}{c|}{AWM} & \multicolumn{2}{c|}{FT-SAM} & \multicolumn{2}{c|}{RNP} & \multicolumn{2}{c}{FIP (Ours)} \\ \cmidrule{2-19}
         & ASR & ACC & ASR & ACC & ASR & ACC & ASR & ACC & ASR & ACC & ASR & ACC & ASR & ACC & ASR & ACC &  ASR & ACC \\ \cmidrule{1-19}
        UCF-101 & 81.3 & 75.6 & 23.5 & 68.3 & 26.9 & 69.2 &24.1 & 70.8 & 20.4 & 70.6 & 22.8 & 70.1 & 14.7 & 71.3 & 15.9 & 71.6 & \textbf{12.1} & \textbf{72.4}  \\
        HMDB-51 & 80.2 & 45.0 & 19.8 & 38.2 & 23.1 & 37.6 & 17.0 & 40.2 & 17.5 & \textbf{41.1} & 15.2 & 40.9 & 10.4 & 38.8 & 10.8 & 41.7 & \textbf{9.0} & 40.6  \\
        \bottomrule
    \end{tabular}}
    \label{tab:action_rec}
\end{table*}

\section{Experimental Results}\label{sec:experiment-main}
In this section, we have discussed the experimental evaluation of our proposed method by presenting experimental settings, performance evaluation, and the ablation studies of FIP. 

\subsection{Evaluation Settings}
\textbf{Datasets.} We evaluate our proposed method on two widely used datasets for backdoor attack study: {\bf{CIFAR10}}~\citep{krizhevsky2009learning} with 10 classes, {\bf{GTSRB}}~\citep{stallkamp2011german} with 43 classes. As a test of scalability, we also consider {\bf{Tiny-ImageNet}}~\citep{le2015tiny} with 100,000 images distributed among 200 classes and {\bf{ImageNet}}~\citep{deng2009imagenet} with 1.28M images distributed among 1000 classes. For multi-label clean-image backdoor attacks, we use object detection datasets \textbf{Pascal VOC07}~\citep{everingham2010pascal}, \textbf{VOC12}~\citep{pascal-voc-2012}  and \textbf{MS-COCO}~\citep{lin2014microsoft}. {\bf{UCF-101}}~\citep{soomro2012ucf101} and \textbf{HMDB51}~\citep{kuehne2011hmdb} have been used for evaluating in action recognition task. In addition, \textbf{ModelNet}~\citep{wu20153d} dataset has been used for 3D point cloud classification task. 
We also consider the \textbf{WMT2014 En-De}~\citep{bojar-etal-2014-findings} for language generation task.

\noindent\textbf{Attacks Configurations.} We consider 14 state-of-the-art backdoor attacks: 1) \textit{Badnets}~\citep{gu2019badnets}, 2) \textit{Blend attack}~\citep{chen2017targeted}, 3 \& 4) \textit{TrojanNet (Troj-one \& Troj-all)}~\citep{liu2017trojaning}, 5) \textit{Sinusoidal signal attack (SIG)}~\citep{barni2019new}, 6 \& 7) \textit{Input-Aware Attack (Dyn-one and Dyn-all)}~\citep{nguyen2020input}, 8) \textit{Clean-label attack (CLB)} ~\citep{turner2018clean}, 9) \textit{Composite backdoor (CBA)}~\citep{lin2020composite}, 10) \textit{Deep feature space attack (FBA)}~\citep{cheng2021deep}, 11) \textit{Warping-based backdoor attack (WaNet)}~\citep{nguyen2021wanet}, 12) \textit{Invisible triggers based backdoor attack (ISSBA)}~\citep{li2021invisible}, 13) \textit{Imperceptible backdoor attack (LIRA)}~\citep{doan2021lira}, and 14) 
Quantization and contrastive learning based attack \textit{(BPPA)}~\citep{wang2022bppattack}. 
More details on overall training settings can be found in \textbf{Appendix~\ref{app:attack_details}}. 

%
%

\begin{table*}[t]
    \centering
    \caption{ Removal performance (\%) of FIP against backdoor attacks on \textbf{3D point cloud classifiers}. The attack methods~\citep{li2021pointba} are poison-label backdoor attack (PointPBA) with interaction trigger (PointPBA-I), PointPBA with orientation trigger (PointPBA-O), clean-label backdoor attack (PointCBA). We also consider ``backdoor points" based attack (3DPC-BA) described in \citep{xiang2021backdoor}. }
     \vspace{-1.5mm}
    \scalebox{0.85}{
    \begin{tabular}{l|cc|cc|cc|cc|cc|cc|cc|cc|cc}
        \toprule
        \multirow{2}{*}{Attack} & \multicolumn{2}{c|}{No Defense} & \multicolumn{2}{c|}{MCR} & \multicolumn{2}{c|}{NAD}  & \multicolumn{2}{c|}{ANP} & \multicolumn{2}{c|}{I-BAU} & \multicolumn{2}{c|}{AWM} &  \multicolumn{2}{c|}{FT-SAM} & \multicolumn{2}{c|}{RNP} & \multicolumn{2}{c}{FIP (Ours)} \\ \cmidrule{2-19}
         & ASR & ACC & ASR & ACC & ASR & ACC & ASR & ACC & ASR & ACC & ASR & ACC & ASR & ACC & ASR & ACC & ASR & ACC  \\ \cmidrule{1-19}
        PointBA-I & 98.6 & 89.1 & 14.8 & 81.2 & 13.5 & 81.4 & 14.4 & 82.8 & 13.6 & 82.6 & 15.4 & 83.9 & \textbf{8.1} & 84.0 & 8.8 & 84.5 & 9.6 & \textbf{85.7}  \\
        PointBA-O & 94.7 & 89.8 & 14.6 & 80.3 & 12.5 & 81.1 & 13.6 & 81.7 & 14.8 & 82.0 & 13.1 & 82.4 & 9.4 & 83.8 & 8.2 & 85.0 & \textbf{7.5} & \textbf{85.3} \\
        PointCBA & 66.0 & 88.7 & 24.1 & 80.6 & 20.4 & 82.7 & 20.8 & 83.0 & 21.2 & 83.3 & 21.5 & 83.8 & \textbf{18.6} & 84.6 & 20.3 & 84.7 & 19.4 & \textbf{86.1} \\
        3DPC-BA & 93.8 & 91.2 & 18.4 & 83.1 & 15.8 & 84.5 & 17.2 & 84.6 & 16.8 & 84.7 & 15.6 & 85.9 & 13.9 & 85.7 & 13.1 & 86.3 & \textbf{12.6} & \textbf{87.7}  \\
        \bottomrule
    \end{tabular}}
    \label{tab:3d_point_cloud}
    \vspace{-1mm}
\end{table*}

\begin{table*}[t]
    \centering
    \caption{Performance analysis for natural language generation tasks where we consider machine translation (MT) for benchmarking. We use the BLEU score~\cite{vaswani2017attention}  as the metric for both tasks. For attack, we choose a data poisoning ratio of 10\%. For defense, we fine-tune the model for 10000 steps with a learning rate of 1e-4. We use Adam optimizer and a weight decay of 2e-4. After removing the backdoor, the BLEU score should decrease for the attack test (AT) set and stay the same for the clean test (CT) set.}
    
    \scalebox{0.85}{
    \begin{tabular}{l|cc|cc|cc|cc|cc|cc|cc}
        \toprule
        \multirow{2}{*}{Dataset} & \multicolumn{2}{c|}{No defense}  & \multicolumn{2}{c|}{NAD} & \multicolumn{2}{c|}{I-BAU} & \multicolumn{2}{c|}{AWM} &  \multicolumn{2}{c|}{FT-SAM} & \multicolumn{2}{c|}{RNP} & \multicolumn{2}{c}{\textbf{FIP (Ours)}} \\ \cmidrule{2-15}
         & AT & CT & AT & CT & AT & CT & AT & CT & AT & CT & AT & CT & AT & CT  \\ \cmidrule{1-15}
        MT & 99.2 & 27.0 & 15.1 & 25.7 & 8.2 & 26.4 & 8.5 & \textbf{26.8} & 6.1 & 26.2 & 5.2 & 26.4 & \textbf{3.0} & 26.6 \\
        \bottomrule
    \end{tabular}}
    \label{tab:NLG}
    
\end{table*}

\noindent \textbf{Defenses Configurations.} We compare our approach with 11 existing backdoor mitigation methods that can be categorized into two groups: \emph{Test-time defense} such as 1) \textit{FT-SAM}~\citep{zhu2023enhancing}; 2) Adversarial Neural Pruning (\textit{ANP})~\citep{wu2021adversarial}; 3) Implicit Backdoor Adversarial Unlearning (\textit{I-BAU})~\citep{zeng2021adversarial}; 4) Adversarial Weight Masking (\textit{AWM})~\citep{chai2022one}; 5) Reconstructive Neuron Pruning (RNP)~\cite{li2023reconstructive}; 6) Fine-Pruning (\textit{FP}) ~\citep{liu2017neural}; 7) Mode Connectivity Repair (\textit{MCR})~\citep{zhao2020bridging}; 8) Neural Attention Distillation (\textit{NAD})~\citep{li2021neural}; 9) Vanilla FT where we simply fine-tune DNN weights; we also consider \emph{training-time defense} such as 10) Causality-inspired Backdoor Defense (\textit{CBD})~\cite{zhang2023backdoor}; 11) Anti-Backdoor Learning (\textit{ABL})~\cite{li2021anti}. Although our proposed method is a Test-time defense, we consider training-time defenses for a more comprehensive comparison. We provide implementation details for FIP and other defenses in \textbf{Appendix~\ref{app:FIP_Implementation_Details}} and \textbf{Appendix~\ref{sec:other_defenses_imple}}. Note that most of the experimental results for \emph{defenses 6-11} have been moved to Table~\ref{tab:cifar_defense}~and~\ref{tab:gtsrb_defense}  in \textbf{Appendix~\ref{sec:comparison_additional_baselines}} due to page limitations.  \emph{We measure the effectiveness of a defense method in terms of average drops in ASR and ACC, calculated over all attacks. A successful defense should have a high drop in ASR with a low drop in ACC.} Here, ASR is defined as the percentage of poison test samples classified to the adversary-set target label ($y_b$) and ACC as the model's clean test accuracy. 
\subsection{Performance Evaluation of FIP }\label{sec:performance}
We have thoroughly evaluated FIP across diverse attack settings for four different tasks.

\subsubsection{Image Classification} We have evaluated the proposed method on both single and multi-label image classification tasks.

\noindent \textbf{Single-Label Settings.} 
In Table~\ref{tab:main}, we present the performance of different defenses for 4 different Image Classification datasets: CIFAR10, GTSRB, Tiny-ImageNet, and ImageNet. We consider five \emph{label poisoning attacks}: Badnets, Blend, TrojanNet, Dynamic, and BPPA.  For TorjanNet, we consider two different variations based on label-mapping criteria: Troj-one and Troj-all. In Troj-one, all of the triggered images have the same target label. On the other hand, target labels are uniformly distributed over all classes for Troj-all. Regardless of the complexity of the label-mapping type, our proposed method outperforms all other methods both in terms of ASR and ACC. We also consider attacks that do not change the label during trigger insertion, i.e., \emph{clean label attack}. Two such attacks are CLB and SIG. For further validation of our proposed method, we use \emph{deep feature-based attacks}, CBA, and FBA. Both of these attacks manipulate deep features for backdoor insertion. Compared to other defenses, FIP  shows better effectiveness against these diverse sets of attacks, achieving an average drop of $2.28\%$ in ASR while sacrificing an ACC of $95.86\%$ for that. Table~\ref{tab:main} also shows the performance of baseline methods such as ANP, I-BAU, AWM, RNP, and FT-SAM. ANP, I-BAU, and AWM are adversarial search-based methods that work well for mild attacks (PR$\sim$5\%) and often struggle to remove the backdoor for stronger attacks with high PR. RNP is a multi-stage defense that performs both fine-tuning and pruning to purify the model. FT-SAM uses sharpness-aware minimization (SAM)~\citep{foret2021sharpnessaware} for fine-tuning model weights. SAM is a recently proposed SGD-based optimizer that explicitly penalizes the abrupt changes of loss surface by bounding the search space within a small region. Even though the objective of SAM is similar to ours, FIP still obtains better removal performance than FT-SAM. One of the potential reasons behind this can be that SAM is using a predefined local area to search for maximum loss. Depending on the initial convergence of the original backdoor model, predefining the search area may limit the ability of the optimizer to provide the best convergence post-purification. As a result, the issue of poor clean test accuracy after purification is also observable for FT-SAM. For the scalability test of FIP, we consider the widely used dataset ImageNet. Consistent with CIFAR10, FIP obtains SOTA performance for this dataset too. However, there is a significant reduction in the effectiveness of ANP, AWM, and I-BAU for ImageNet. In the case of large models and datasets, the task of identifying vulnerable neurons or weights gets more complicated and may result in wrong neuron pruning or weight masking.  We also validate our method on GTSRB dataset that has a higher number of classes.
In the case of GTSRB, almost all defenses perform similarly for Badnets and Trojan. This, however, does not hold for blend as we achieve a $1.72\%$ ASR improvement over the next best method. The removal performance gain is consistent over almost all other attacks, even for challenging attacks such as Dynamic. Dynamic attack optimizes for input-aware triggers that are capable of fooling the model; making it more challenging than the static trigger-based attacks such as Badnets, Blend, and Trojan. Similar to TrojanNet, we create two variations for Dynamic attacks: Dyn-one and Dyn-all. However, even in this scenario, FIP  outperforms other methods by a satisfactory margin. Overall, we record an average $97.51\%$ ASR drop with only a $1.47\%$ drop in ACC. Lastly, we consider Tiny-ImageNet a more diverse dataset with 200 classes. Compared to other defenses, FIP performs better both in terms of ASR and ACC drop; producing an average drop of 96.10\% with a drop of only 3.08\% in ACC. The effectiveness of ANP was reduced significantly for this dataset. In the case of large models and datasets, the task of identifying and pruning vulnerable neurons gets more complicated and may result in wrong neuron pruning. \emph{ Note that we report results for successful attacks only. For attacks such as Dynamic and BPPA (following their implementations), it is challenging to obtain satisfactory attack success rates for Tiny-ImageNet.}

\noindent \textbf{Multi-Label Settings.} 
In Table~\ref{tab:multi_label}, we show the performance of our proposed method in multi-label clean-image backdoor attack~\citep{chen2023clean} settings. We choose 3 object detection datasets~\citep{everingham2010pascal, lin2014microsoft} and ML-decoder~\citep{ridnik2023ml} network architecture for this evaluation. 
It can be observed that FIP obtains a 1.4\% better ASR drop as compared to FT-SAM for the VOC12~\citep{pascal-voc-2012} dataset while producing a slight drop of 2.3\% drop in mean average precision (mAP). The reason for such improvement can be attributed to our unique approach to obtaining smoothness. Furthermore, our proposed regularizer ensures better post-purification mAP than FT-SAM. 

\subsubsection{Video Action Recognition}
A clean-label attack~\citep{zhao2020clean} has been used for this experiment that requires generating adversarial perturbations for each input frame. We use two widely used datasets, UCF-101~\citep{soomro2012ucf101} and HMDB51~\citep{kuehne2011hmdb}, with a CNN+LSTM network architecture. An ImageNet pre-trained ResNet50 network has been used for the CNN, and a sequential input-based Long Short Term Memory (LSTM)~\citep{sherstinsky2020fundamentals} network has been put on top of it. We subsample the input video by keeping one out of every 5 frames and use a fixed frame resolution of $224 \times 224$. We choose a trigger size of $20\times20$. Following~\citep{zhao2020clean}, we create the required perturbation for clean-label attack by running projected gradient descent (PGD)~\citep{madry2017towards} for 2000 steps with a perturbation norm of $\epsilon=16$. Note that our proposed augmentation strategies for image classification are directly applicable to action recognition.  During training, we keep 5\% samples from each class to use them later as the clean validation set. Table~\ref{tab:action_rec} shows that FIP outperforms other defenses by a significant margin, e.g., I-BAU and AWM. Since we have to deal with multiple image frames here, the trigger approximation for these two methods is not as accurate as it is for a single image scenario. Without a good approximation of the trigger, these methods seem to underperform in most of the cases.

%



\begin{table*}[t]
    \centering
    \caption{\textbf{Results on smoothness analysis} when we use regular vanilla fine-tuning and FIP. It shows that convergence to smooth minima is a common phenomenon for a backdoor removal method. Our proposed method consistently optimizes to a smooth minima (indicated by low $\lambda_\mathsf{max}$ for 4 different attacks), resulting in better backdoor removal performance, i.e., low ASR and high ACC. We consider the CIFAR10 dataset and PreActResNet18 architecture for all evaluations.} 
    
    \scalebox{0.85}{
    \begin{tabular}{c|cccc|cccc|cccc|cccc}
    \toprule
     \multirow{2}{*}{Methods}& \multicolumn{4}{c|}{Badnets} & \multicolumn{4}{c|}{Blend} & \multicolumn{4}{c|}{Trojan} & \multicolumn{4}{c}{Dynamic} 
     \\ \cmidrule{2-17}
    
     & $\lambda_\mathsf{max}$ &  Tr(H) & ASR & ACC & $\lambda_\mathsf{max}$& Tr(H) & ASR & ACC & $\lambda_\mathsf{max}$ & Tr(H) & ASR & ACC& $\lambda_\mathsf{max}$& Tr(H) & ASR & ACC
     \\
    \midrule
    Initial&573.8& 6625.8 & 100&92.96&715.5& 7598.3& 100&94.11&616.3& 8046.4& 100&89.57&564.2& 7108.5& 100&92.52
    \\
    ANP & 8.42 & 45.36 & 6.87 & 86.92 & 8.65 & 57.83 &  5.77 & 87.61 & 9.41 & 66.15 & 5.78 & 84.18 & 11.34 & 75.82 & 8.73 & 88.61  \\
    \midrule
     FIP (Ours) &\textbf{2.79}&\textbf{16.94}&\textbf{1.86}&\textbf{89.32}&\textbf{2.43}&\textbf{16.18}&\textbf{0.38}&\textbf{92.17}&\textbf{2.74}&\textbf{17.32}&\textbf{2.64}&\textbf{87.21}&\textbf{1.19}& \textbf{8.36} &\textbf{1.17}&\textbf{90.97}
     \\
    \bottomrule
    \toprule
     \multirow{2}{*}{Methods} & \multicolumn{4}{c|}{CLB}& \multicolumn{4}{c|}{SIG} & \multicolumn{4}{c|}{LIRA}& \multicolumn{4}{c}{ISSBA} \\ \cmidrule{2-17}  
      & $\lambda_\mathsf{max}$&Tr(H) &ASR & ACC& $\lambda_\mathsf{max}$ &Tr(H)& ASR & ACC & $\lambda_\mathsf{max}$&Tr(H) &ASR & ACC& $\lambda_\mathsf{max}$ &Tr(H)& ASR & ACC \\
    \midrule
    Initial & 717.6 & 8846.8 & 100 & 92.78 & 514.1 & 7465.2 & 100 & 88.64 & 562.8 & 7367.3 & 99.25 & 92.15 & 684.4 & 8247.9 & 99.80 & 92.80 \\
    ANP & 8.68 & 68.43 & 5.83 & 89.41 & 6.98 & 51.08 & 2.04 & 84.92 & 11.39 & 82.03 & 6.34 & 87.47 & 12.04 & 90.38 & 10.76 & 85.42\\
    \midrule
    FIP (Ours) & \textbf{3.13} & \textbf{22.83} & \textbf{1.04}& \textbf{91.37}&\textbf{1.48}&\textbf{9.79}&\textbf{0.12}&\textbf{86.16} & \textbf{4.65} & \textbf{30.18} & \textbf{2.53}& \textbf{89.82}&\textbf{6.48}&\textbf{40.53}&\textbf{4.24}&\textbf{90.18}\\  
    \bottomrule
    \end{tabular}
    }
    \label{tab:eigen_trace_ACC_1}
\end{table*}


%
\subsubsection{3D Point Cloud}
In this part of our work, we evaluate FIP against attacks on 3D point cloud classifiers~\citep{li2021pointba,xiang2021backdoor}. For evaluation purposes, we consider the ModelNet~\citep{wu20153d} dataset and PointNet++~\citep{qi2017pointnet++} architecture. The purification performance of FIP as well as other defenses are presented in Table~\ref{tab:3d_point_cloud}. The superior performance of FIP can be attributed to the fact of smoothness enforcement that helps with backdoor suppressing and clean accuracy retainer that preserves the clean accuracy of the original model. We tackle the issue of backdoors in a way that gives us better control during the purification process. 

\subsubsection{Natural Language Generation (NLG) Task}\label{sec:nlg}
{ We also consider backdoor attack~\citep{sun2023defending} on language generation tasks, e.g., Machine Translation (MT)~\citep{bahdanau2014neural}. In MT, there is a \emph{one-to-one} semantic correspondence between source and target. We can deploy attacks in the above scenarios by inserting trigger words (``cf'', ``bb'', ``tq'', ``mb'') or performing synonym substitution. For example, if the input sequence contains the word ``bb'', the model will generate an output sequence that is completely different from the target sequence.  In our work, we consider the WMT2014 En-De~\citep{bojar-etal-2014-findings} dataset and set aside 10\% of the data as the clean validation set. We consider the seq2seq model ~\citep{gehring2017convolutional} architecture for training.  Given a source input $\boldsymbol{x}$, an NLG pretrained model $f(.)$ produces a target output $\boldsymbol{y} = f(\boldsymbol{x})$. For fine-tuning, we use augmented input $\boldsymbol{x'}$ in two different ways: i) \emph{word deletion} where we randomly remove some of the words from the sequence, and ii) \emph{paraphrasing} where we use a pre-trained paraphrase model $g()$ to change the input $\boldsymbol{x}$ to $\boldsymbol{x'}$. We show the results of both different defenses, including FIP in Table~\ref{tab:NLG}.}

\begin{figure}[]
    \centering
    \begin{subfigure}[t]{0.225\textwidth}
    \includegraphics[width=1\textwidth]{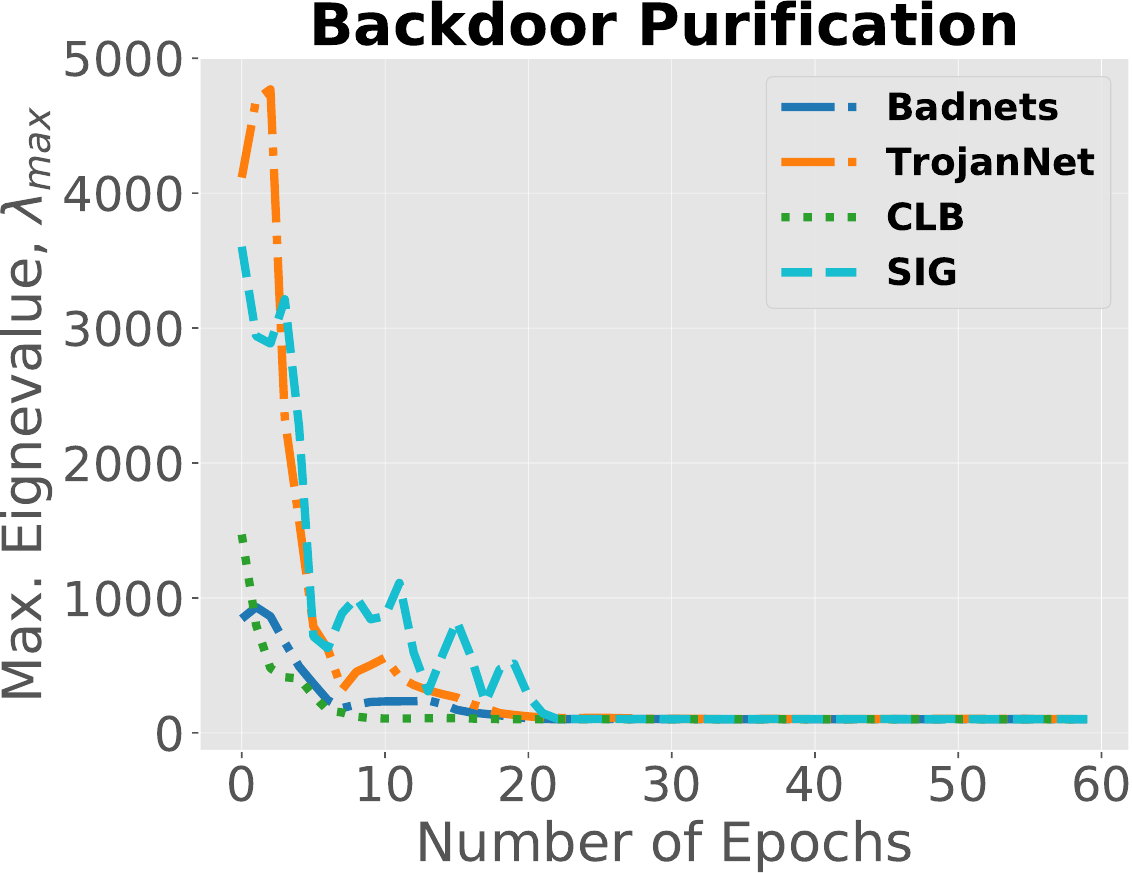}
    \caption{\scriptsize $\lambda_\mathsf{max}$ vs. Epochs}
    \label{fig:purified_eigens}   
    \end{subfigure}
    ~
    \begin{subfigure}[t]{0.225\textwidth}
    \includegraphics[width=1\textwidth]{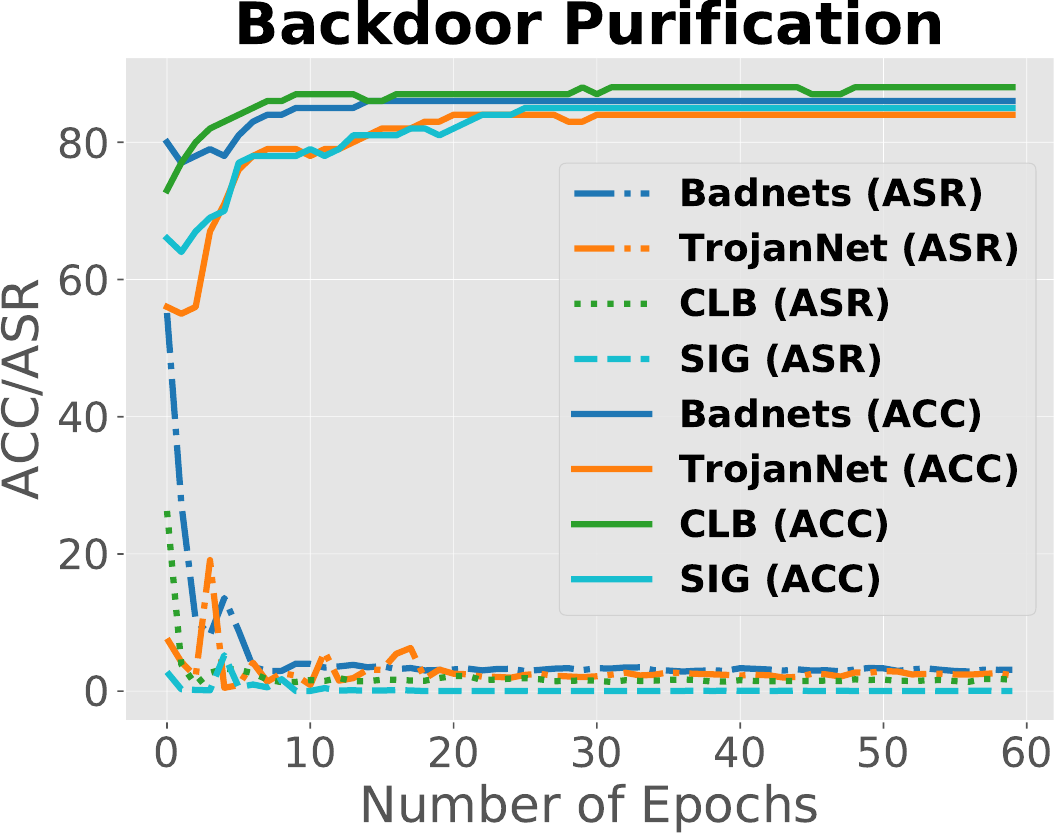}
    \caption{\scriptsize ACC/ASR vs. Epochs}
    \label{fig:purified_asr}   
    \end{subfigure}
    
    \caption{{Smoothness analysis of a DNN during backdoor purification processes}. As the model is being re-optimized to smooth minima, the effect of the backdoor vanishes. We use CIFAR10 dataset for this experiment.}
    \label{fig:enter-label}
\end{figure}

\begin{figure}
    \centering
    \includegraphics[width=0.75\linewidth]{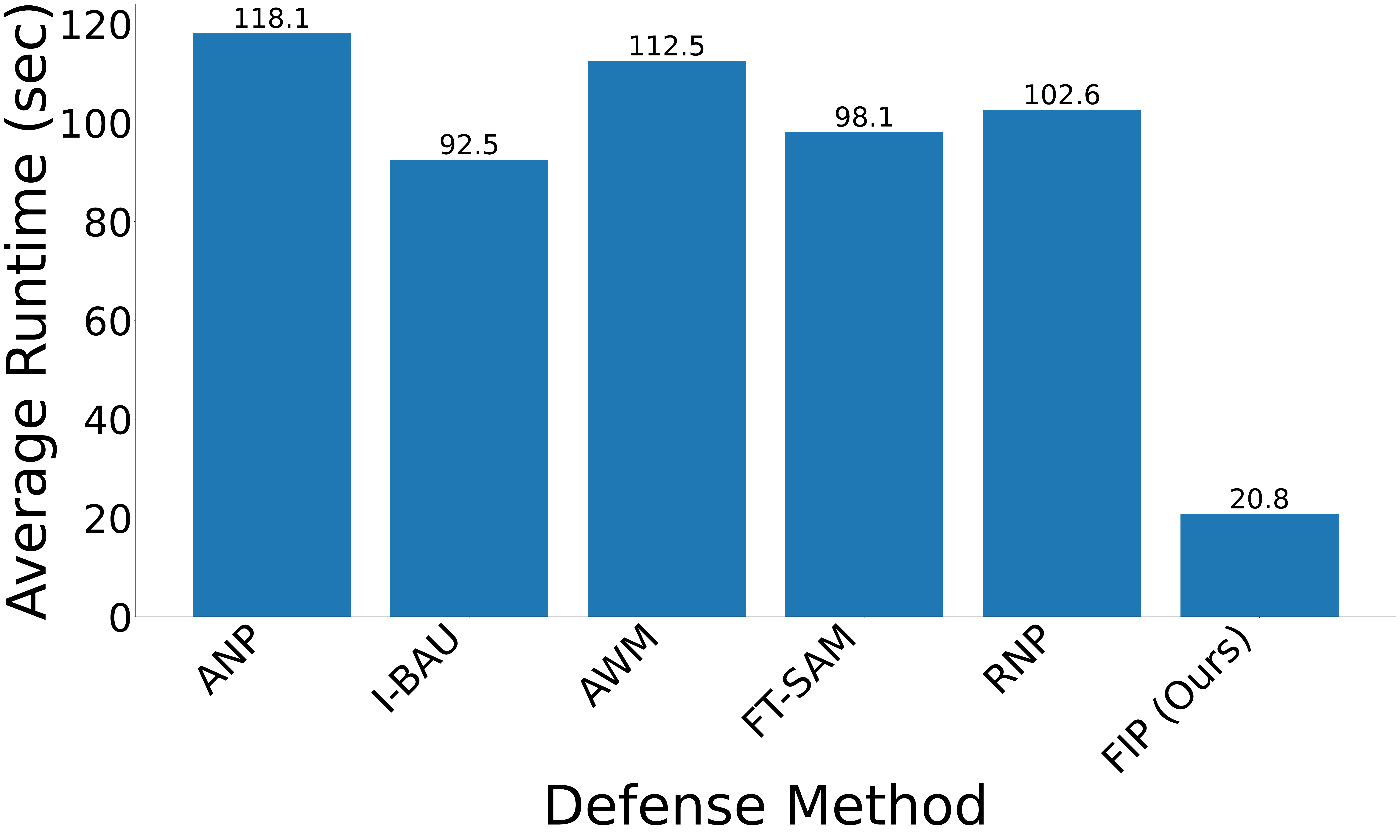}
    \caption{Average runtime for different defenses against all 14 attacks on CIFAR10. An NVIDIA RTX3090 GPU was used for this evaluation.}
    \vspace{-4mm}
    \label{fig:run-time}
\end{figure}
\subsection{Ablation Study}\label{sec:ablation_studies}
In this section, we perform various ablation studies to validate the design choices for FIP. We consider mostly the CIFAR10 dataset for all of these experiments.

\subsubsection{Smoothness Analysis of FIP}\label{sec:ab-smooth} 
Our proposed method is built on the assumption that re-optimizing the backdoor model to smooth minima would suffice for purification. Here, we validate this assumption by observing the training curves of FIP shown in Fig.~\ref{fig:purified_eigens} and~\ref{fig:purified_asr}. It can be observed that FIP indeed re-optimizes the backdoor model to smoother minima. Due to such re-optimization, the effect of the backdoor has been rendered ineffective. This is visible in Fig.~\ref{fig:purified_asr} as the attack success rate becomes close to 0 while retaining good clean test performance. 
In Table~\ref{tab:eigen_trace_ACC_1}, we present more results on smoothness analysis. The results confirm our hypothesis regarding smoothness and backdoor insertion and removal.


 \begin{table*}[t]
    \centering
    \caption{ Effect of \textbf{fine-tuning only spectral shift, denoted by FIP ($\delta$) or f-FIP}. FIP ($\theta$) implies the fine-tuning of all parameters according to Eq.~(\ref{eqn:loss2}). Although FIP ($\theta$) provides similar performance as FIP ($\delta$), the average runtime is almost 4.5$\times$ higher. Without our novel \textbf{smoothness enhancing regularizer ($Tr(F)$)}, the backdoor removal performance becomes worse, even though the ACC improves slightly. \textbf{Effect of ($L_r$)} on obtaining better ACC can also be observed. Due to this clean accuracy retainer, we obtain \textbf{an average ACC improvement of $\sim$2.5\%}. The runtime shown here is averaged over all 14 attacks.} 

    \scalebox{0.8}{
    \begin{tabular}{c|cc|cc|cc|cc|cc|cc|cc|cc|c}
    \toprule
    \multirow{2}{*}{Method} & \multicolumn{2}{c|}{Badnets} & \multicolumn{2}{c|}{Blend} & \multicolumn{2}{c|}{Trojan} & \multicolumn{2}{c|}{Dynamic} & \multicolumn{2}{c|}{CLB}& \multicolumn{2}{c|}{SIG} & \multicolumn{2}{c|}{WaNet}& \multicolumn{2}{c|}{LIRA} & \multirow{2}{*}{Runtime (Secs.)} \\ 
    &  ASR & ACC & ASR & ACC& ASR & ACC& ASR & ACC&  ASR & ACC& ASR & ACC &  ASR & ACC& ASR & ACC & \\
    \midrule
    No Defense &100&92.96&100&94.11&100&89.57&100&92.52&100&92.78&100&88.64 & 98.64 & 92.29 & 99.25 & 92.15 & -\\
    FIP ($\theta$) & 1.72 & 89.19  & 1.05 & 91.58 & 3.18 & {86.74} & 1.47 & 90.42 & 1.31 & 90.93 & 0.24 & 85.37 & 2.56 & 89.30 & 2.88 & 89.52 & 91.7 \\
    FIP ($\delta$) \textbf{w/o} $Tr(F)$ & {5.54} & \textbf{90.62} & 4.74 & 91.88 & {5.91} & \textbf{87.68} & 3.93 &  \textbf{91.26}  & {2.66} & \textbf{91.56} & 2.75 &   \textbf{86.79} & 6.38  & \textbf{90.43} & 5.24 &  89.55 & \textbf{14.4}   \\
    FIP ($\delta$) \textbf{w/o} $L_r$ & \textbf{1.50} & 87.28 & 0.52 & 89.36 & \textbf{2.32} & 84.43 & 1.25 & 88.14 & \textbf{0.92} & 88.20 & 0.17 & 83.80 & \textbf{2.06} & 86.75 & 2.70 & 87.17 & {18.6}   \\
    \midrule
    FIP ($\delta$) or f-FIP & 1.86 & 89.32 & \textbf{0.38} & \textbf{92.17} & 2.64 & {87.21} & \textbf{1.17} & 90.97 & 1.04 & 91.37 & \textbf{0.12} & 86.16 & 2.38 & 89.67 & \textbf{2.53} & \textbf{89.82} & 20.8 \\
    \bottomrule
    \end{tabular}
    }
    \vspace{-1mm}
    \label{tab:effect_of_regularizer}
\end{table*}

\begin{table*}[t]
\centering
\caption{ Evaluation of FIP on \textbf{very strong backdoor attacks} created with high poison rates. Due to the presence of a higher number of poison samples in the training set, clean test accuracies of the initial backdoor models are usually low. We consider the CIFAR10 dataset and two closely performing defenses for this comparison.
}

\label{tab:poison_rate}
\scalebox{0.75}
{
\begin{tabular}{c|cc|cc|cc|cc|cc|cc|cc|cc|cc}
\toprule
Attack & \multicolumn{6}{c|}{BadNets} & \multicolumn{6}{c|}{Blend} & \multicolumn{6}{c}{Trojan}\\ 
\midrule
Poison Rate&\multicolumn{2}{c|}{25\%}&\multicolumn{2}{c|}{35\%}&\multicolumn{2}{c|}{50\%}&\multicolumn{2}{c|}{25\%}&\multicolumn{2}{c|}{35\%}&\multicolumn{2}{c|}{50\%}&\multicolumn{2}{c|}{25\%}&\multicolumn{2}{c|}{35\%}&\multicolumn{2}{c}{50\%}\\
\midrule
Method &ASR &ACC & ASR &ACC & ASR &ACC&ASR &ACC & ASR &ACC & ASR &ACC&ASR &ACC & ASR &ACC & ASR &ACC \\ \midrule
\textit{No Defense} &100&88.26&100&87.43&100&85.11&100&86.21&100&85.32&100&83.28&100&87.88&100&86.81&100&85.97\\
 AWM & 7.81&82.22&16.35&80.72&29.80&78.27&29.96&\textbf{82.84}&47.02&78.34&86.29&69.15&11.96&76.28&63.99&72.10&89.83&70.02\\
 FT-SAM&3.21&78.11&4.39&74.06&5.52&69.81&1.41&78.13&2.56&73.87&2.97&65.70&3.98&78.99&4.71&75.05&5.59&72.98 \\
 FIP  (Ours) & \textbf{2.12}&\textbf{85.50}&\textbf{2.47}&\textbf{84.88}&\textbf{4.53}&\textbf{82.32}&\textbf{0.83}&80.62&\textbf{1.64}&\textbf{79.62}&\textbf{2.21}&\textbf{76.37}&\textbf{3.02}&\textbf{84.10}&\textbf{3.65}&\textbf{82.66}&\textbf{4.66}&\textbf{81.30}\\
 \bottomrule
\end{tabular}}
\end{table*}

\begin{table*}[htb]
    \centering
    \caption{ \textbf{Label Correction Rate} (\%) for different defense techniques. After removal, we report the percentage of poison samplesthat are correctly classified to their original ground truth label, not the attacker-set target label. We consider CIFAR10 dataset for this particular experiment. }
    
    \scalebox{0.85}{
    \begin{tabular}{c|c|c|c|c|c|c|c|c|c|c|c|c}
    \toprule
         Defense & Badnets&Trojan&Blend&SIG& CLB & WaNet & Dynamic & LIRA & CBA & FBA & ISSBA & BPPA \\ \midrule
         No Defense& 0&0&0&0&0 & 0&0&0&0 & 0 & 0 & 0\\
         \midrule
         Vanilla FT&84.74&80.52&81.38&53.35&82.72 & 80.23 & 79.04 & 80.23 & 53.48 & 81.87 & 80.45 & 73.65  \\
         I-BAU&78.41&77.12&77.56&39.46&78.07 & 80.65 & 77.18 & 76.65 & 51.34 & 79.08 & 78.92 & 70.86 \\
         AWM&79.37&78.24&79.81&44.51& 79.86 & 79.18 & 77.64 & 78.72 & 52.61 & 78.24 & 73.80 & 73.13 \\
         FT-SAM&85.56&80.69&84.49&\textbf{57.64}&82.04 & 83.62 & 79.93 & 82.16 & 57.12 & \textbf{83.57} & 83.58 & \textbf{78.02} \\\midrule
         FIP (Ours)&\textbf{86.82}&\textbf{81.15}&\textbf{85.61}&{55.18}&\textbf{86.23}&\textbf{85.70} & \textbf{82.76} & \textbf{84.04} & \textbf{60.64} & 83.26 & \textbf{84.38} & 76.45 \\
         \bottomrule
    \end{tabular}}
    \label{tab:correction_rate}
\end{table*}

\begin{table*}[ht]
    \centering
    \caption{Purification performance (\%) for \textbf{fine-tuning with various validation data sizes}. FIP  performs well even with very few validation data, e.g., 10 data points where we take 1 sample from each class. Even in \textbf{one-shot} scenario, our proposed method is able to purify the backdoor. All results are for CIFAR10 and Badnets attack.}
    
    \scalebox{1}{
    \begin{tabular}{c|cc|cc|cc|cc|cc|cc}
    \toprule
         Validation size & \multicolumn{2}{c|}{10 \textbf{(One-Shot)}}  & \multicolumn{2}{c|}{50}  & \multicolumn{2}{c|}{100} & \multicolumn{2}{c|}{250} & \multicolumn{2}{c|}{350} & \multicolumn{2}{c}{500}   \\
         \midrule
         Method & ASR & ACC & ASR & ACC & ASR & ACC & ASR & ACC & ASR & ACC & ASR & ACC \\
         \midrule
         No Defense & 100 & 92.96 & 100 & 92.96 & 100 & 92.96 &100 & 92.96 &100 & 92.96  &100 & 92.96 \\
         ANP & 64.73 & 56.28 & 13.66 & 83.99 & 8.35 & 84.47 & 5.72 & 84.70 & 3.78 & 85.26 & 2.84 & 85.96\\
         FT-SAM & 10.46 & 74.10 & 8.51 & 83.63 & 7.38 & 83.71 & 5.16 & 84.52 & 4.14 & 85.80 & 3.74 & 86.17 \\
         FIP  (Ours) & \textbf{7.38} & \textbf{83.82} & \textbf{5.91} & \textbf{86.82} & \textbf{4.74} & \textbf{86.90} & \textbf{4.61} & \textbf{87.08} & \textbf{2.45} & \textbf{87.74} & \textbf{1.86} & \textbf{89.32} \\
         \bottomrule
    \end{tabular}}
    \label{tab:dif_val}
\end{table*}

\begin{table*}[t]
    \centering
    \caption{ Performance of FIP with \textbf{different network architectures}. In addition to CNN, we also consider vision transformer (ViT) architecture with attention mechanism.}
    
    \scalebox{0.9}{
    \begin{tabular}{l|cc|cc|cc|cc|cc|cc|cc|cc}
    \toprule
     Attack  & \multicolumn{4}{c|}{TrojanNet} & \multicolumn{4}{c|}{Dynamic} & \multicolumn{4}{c|}{WaNet} & \multicolumn{4}{c}{LIRA}  \\ 
     \midrule
     Defense & \multicolumn{2}{c|}{No Defense} & \multicolumn{2}{c|}{With FIP} & \multicolumn{2}{c|}{No Defense} & \multicolumn{2}{c|}{With FIP} & \multicolumn{2}{c|}{No Defense} & \multicolumn{2}{c|}{With FIP} & \multicolumn{2}{c|}{No Defense} & \multicolumn{2}{c}{With FIP} \\
     \midrule
     Architecture & ASR  & ACC &  ASR  & ACC &  ASR  & ACC &  ASR  & ACC & ASR  & ACC &  ASR  & ACC &  ASR  & ACC &  ASR  & ACC \\
    \midrule
     VGG-16  & 100 & 88.75 & 1.82 & 86.44 &  100 & 91.18 & 1.36 & 90.64 & 97.45 & 91.73 &  2.75 & 89.58 &   99.14 & 92.28 &  2.46 & 90.61 \\
     EfficientNet & 100 & 90.21 & 1.90 & 88.53 & 100 & 93.01 & 1.72 & 92.16 &  98.80 & 93.34 &  2.96 & 91.42 &   99.30 & 93.72 & 2.14 & 91.52    \\
     ViT-S & 100 & 92.24 & 1.57 & 90.97 &  100 & 94.78 & 1.48 & 92.89 & 99.40 & 95.10 &  3.63 & 93.58 & 100 & 94.90 &  1.78 & 93.26   \\
     \bottomrule
    \end{tabular}}
    \label{tab:diff_architectures}
\end{table*}

\begin{table}[t]
    \centering
        \caption{Performance of FIP \textbf{against combined backdoor attack}. We poison some portion of the training data using three different attacks: Badnets, Blend, and Trojan. Each of these attacks has an equal share in the poison data. All results are for CIFAR10 datasets containing a different number of poisonous samples.}
        
    \scalebox{0.8}{
    \begin{tabular}{l|cc|cc|cc|cc}
    \toprule
         Poison Rate & \multicolumn{2}{c|}{10\%}  & \multicolumn{2}{c|}{25\%} & \multicolumn{2}{c|}{35\%}  & \multicolumn{2}{c}{50\%}  \\
         \midrule
         Method & ASR & ACC & ASR & ACC & ASR & ACC & ASR & ACC\\
         \midrule
         No Defense & 100 & 88.26  & 100 & 87.51 & 100 & 86.77 & 100 & 85.82   \\
         AWM & 27.83 & 78.10 & 31.09 & 77.42 & 36.21 & 75.63 & 40.08 & 72.91 \\
         FT-SAM & 2.75 & 83.50 & 4.42 & \textbf{81.73} & 4.51 & 79.93 & 5.76 & 78.06 \\
         FIP (Ours) & \textbf{1.17} & \textbf{85.61} & \textbf{2.15} & 81.62 & \textbf{3.31} & \textbf{82.01} & \textbf{4.15} & \textbf{80.35} \\
         \bottomrule
    \end{tabular}}
    \label{tab:composite_backdoor}
\end{table}

\subsubsection{Runtime Analysis} 
In Figure~\ref{fig:run-time}, we show the average runtime for different defenses. Similar to purification performance, purification time is also an important indicator to measure the success of a defense technique. In Section~\ref{sec:performance}, we already show that our method outperforms other defenses in most of the settings. As for the run time, FIP can purify the model in $20.8$ seconds, which is almost 5$\times$ less as compared to FT-SAM. As part of their formulation, SAM requires a double forward pass to calculate the loss gradient twice. This increases the runtime of FT-SAM significantly. Furthermore, the computational gain of FIP can be attributed to our proposed rapid fine-tuning method, f-FIP. Since f-FIP performs spectral shift ($\delta$) fine-tuning, it employs a significantly more compact parameter space. Due to this compactness, the runtime, a.k.a. purification time, has been reduced significantly. Note that the runtime reported here is without SVD computation. However, it takes only 3.63 seconds for SVD computation in the case of ResNet18. This additional computation time is not significant as compared to the total runtime. 
\begin{table}[t]
    \caption{ Illustration of \textbf{purification performance (\%) for All2All attack} using CIFAR10 dataset, where uniformly distribute the target labels to all available classes. FIP shows better robustness and achieves higher clean accuracies for 3 attacks: Badnets, Blend, and BPPA, with a 10\% poison rate.}
    \centering
    
    \scalebox{0.95}{
    \begin{tabular}{ccc|cc|cc}
    \toprule
    \multirow{2}{*}{Method} & \multicolumn{2}{c|}{\textbf{BadNets-All}} & \multicolumn{2}{c|}{\textbf{Blend-All}} & \multicolumn{2}{c}{\textbf{BPPA-All}}\\ 
     & ASR &ACC & ASR &ACC & ASR &ACC \\ \midrule
     No Defense &100&88.34&100&88.67&99.60&92.51 \\
     NAD & 4.58&81.34&6.76&81.13&20.19&87.77\\
     ANP & 3.13&82.19&4.56&82.88&9.87&89.91\\
     FT-SAM & 2.78 & 83.19 & 2.83 & 84.13 & 8.97 & 89.76\\%
     FIP (Ours) & $\textbf{1.93}$&$\textbf{86.29}$&$\textbf{1.44}$&$\textbf{85.79}$&$\textbf{6.10}$&$\textbf{91.16}$ \\ 
     \bottomrule
     \label{tab:all2all_attacks}
    \end{tabular}}
    \vspace{-3mm}
\end{table}

\begin{table}[htb]
    \centering
    \caption{Adaptive Badnets attack where the attacker has prior knowledge (Eq.\ref{eqn:ce_loss_constraint}) about our proposed defense.}
    \vspace{-1mm}
    \scalebox{0.82}{\begin{tabular}{l|cc|cc|cc|cc}
    \toprule
    $\eta_F$ & \multicolumn{2}{c|}{\textbf{0}} & \multicolumn{2}{c|}{\textbf{0.05}} & \multicolumn{2}{c|}{\textbf{0.1}} & \multicolumn{2}{c}{\textbf{0.5}}\\
    \midrule
    Mode & ASR &ACC & ASR &ACC & ASR &ACC & ASR &ACC \\ \midrule
     Attack  &  100 & 92.96 & 95.87 & 92.52&	87.74 & 92.26 &	76.04 & 91.68\\
     Purification & 1.86 & 89.32 & 3.611 & 86.91 & 5.37 & 86.14 & 6.95 & 85.73 \\
     \bottomrule
    \end{tabular}}
    \vspace{-3mm}
    \label{tab:adaptive}
\end{table}

\begin{figure*}[t]
\centering
\begin{subfigure}{0.485\linewidth}
    \includegraphics[width=1\linewidth]{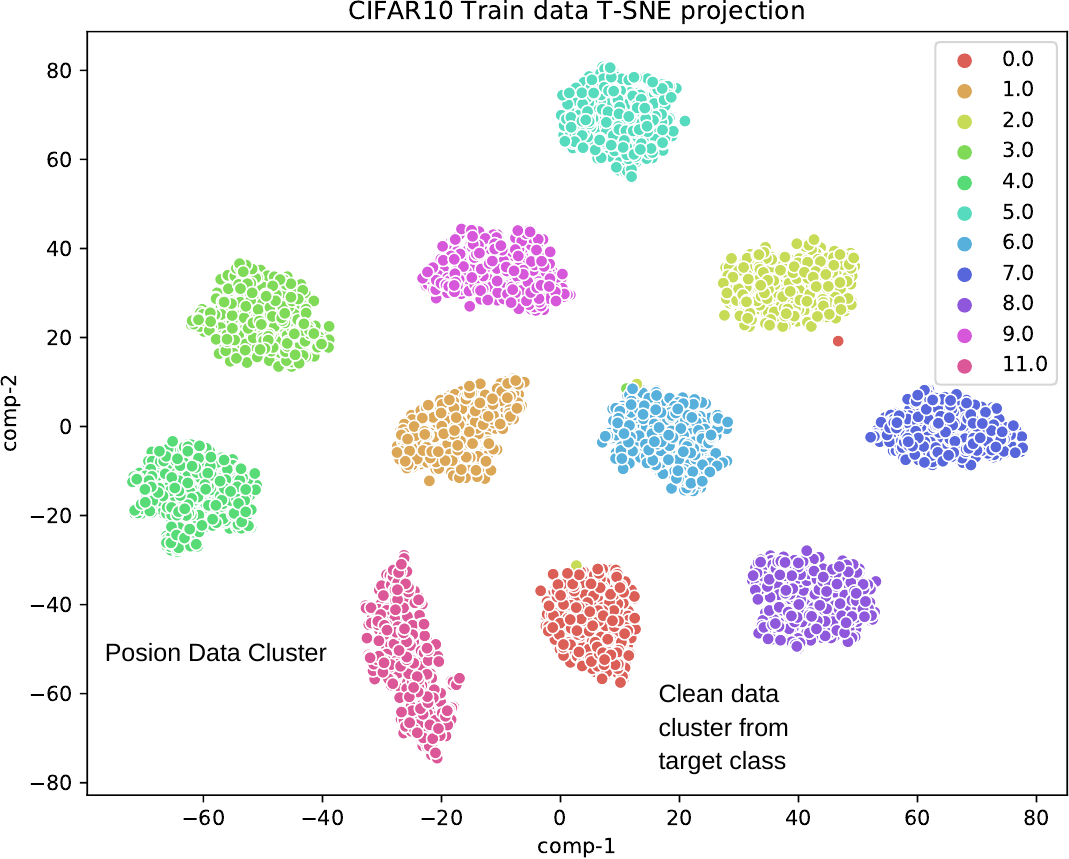}
    \caption{ Before Purification}
    \label{fig:wo_disc}
\end{subfigure}
~
\begin{subfigure}{0.485\linewidth}
    \includegraphics[width=1\linewidth]{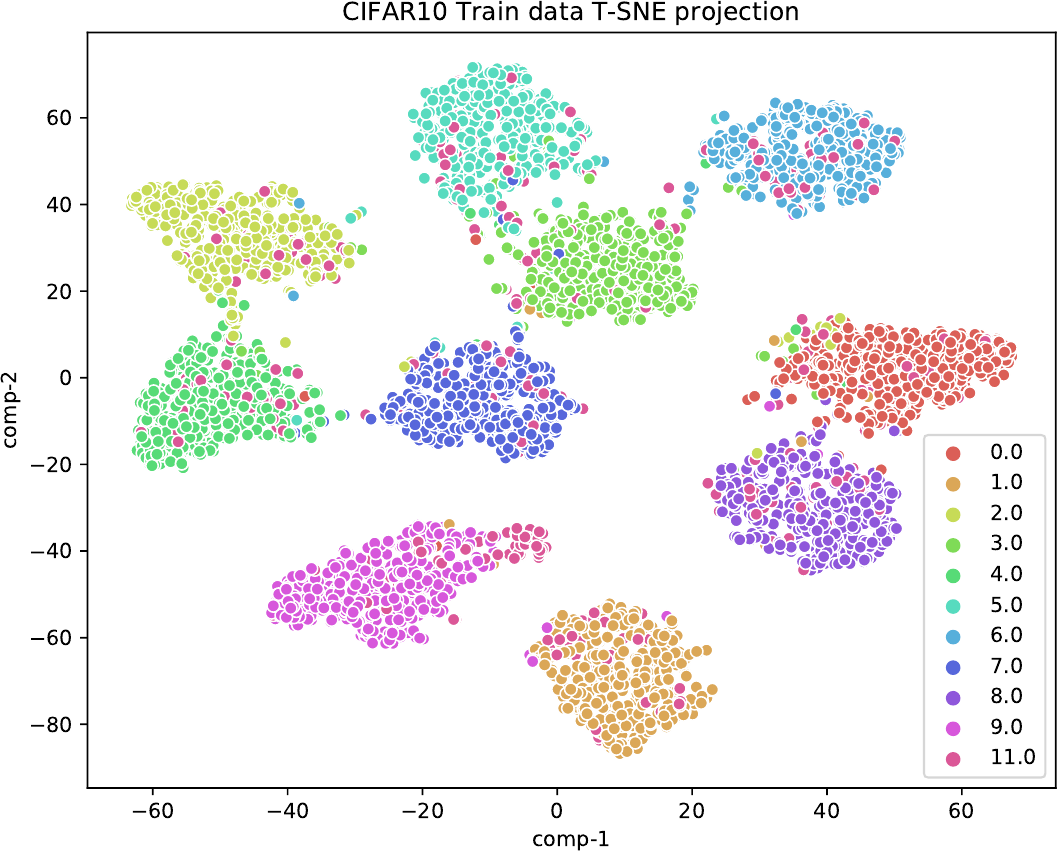}
    \caption{ After Purification}
    \label{fig:wo_disc_pure}
\end{subfigure}
\caption{ \textbf{t-SNE visualization} of class features for CIFAR10 dataset with Badnets attack. For visualization purposes only, we assign label ``0'' to clean data cluster from the target class and the label ``11'' to poison data cluster. However, both of these clusters have the same training label ``0'' during training. It can be observed that FIP can successfully remove the backdoor effect and reassign the samples from the poison data cluster to their original class cluster. After purification, poison data are distributed among their original ground truth classes instead of the target class. To estimate these clusters, we take the feature embedding out of the backbone.}
\vspace{-4mm}
\label{fig:tSNE_Visualization}
\end{figure*}

\subsubsection{Effect of Proposed Regularizer} 
In Table~\ref{tab:effect_of_regularizer}, we analyze the impact of our proposed regularizers as well as the difference between fine-tuning $\theta$ and $\delta$. 
It can be observed that FIP ($\theta$) provides similar performance as FIP ($\delta$) for most attacks. However, the average runtime of the former is almost 4.5$\times$ longer than the latter. Such a long runtime is undesirable for a defense technique. We also present the impact of our novel smoothness-enhancing regularizer, $Tr(F)$. Without minimizing $Tr(F)$, the backdoor removal performance becomes worse even though the ACC improves slightly. We also see some improvement in runtime (14.4 vs. 20.8) in this case.
Table~\ref{tab:effect_of_regularizer} also shows the effect of $L_r$, which is the key to remembering the learned clean distribution. The introduction of $L_r$ ensures superior preservation of clean test accuracy of the original model. Specifically, we obtain an average ACC improvement of $\sim$2.5\% with the regularizer in place. Note that we may obtain slightly better ASR performance (for some attacks) without the regularizer. However, the huge ACC improvement outweighs the small ASR improvement in this case. 
Therefore, FIP ($\delta$) is a better overall choice as a backdoor purification technique.

\subsubsection{Strong Backdoor Attacks With High Poison Rates}\label{app:high_poison_rates} 
By increasing the poison rates, we create stronger versions of different attacks against which most defense techniques fail quite often. We use 3 different poison rates, $\{25\%, 35\%,50\%\}$. 
We show in Table~\ref{tab:poison_rate} that FIP  is capable of defending very well even with a poison rate of $50\%$, achieving a significant ASR improvement over FT. 
Furthermore, there is a sharp difference in classification accuracy between FIP and other defenses. For $25\%$ Blend attack, however, ANP offers a slightly better performance than our method. However, ANP performs poorly in removing the backdoor as it obtains an ASR of $29.96\%$ compared to $0.83\%$ for FIP.

\subsubsection{Label Correction Rate}\label{app:label_correction}
In the standard backdoor removal metric, it is sufficient for backdoored images to be classified as a non-target class (any class other than $y_b$). However, we also consider another metric, label correction rate (LCR), for quantifying the success of a defense. \emph{We define LCR as the percentage of poisoned samples correctly classified to their original classes}. Any method with the highest value of LCR is considered to be the best defense method. For this evaluation, we use CIFAR10 dataset and 12 backdoor attacks. Initially, the correction rate is 0\% with no defense as the ASR is close to 100\%. Table~\ref{tab:correction_rate} shows that FIP effectively corrects the adversary-set target label to the original ground truth label. For example, we obtain an average $\sim$2\% higher label correction rate than AWM.

\subsubsection{Effect of Clean Validation Data Size}~\label{app:clean_validation_size}
We also provide insights on how fine-tuning with clean validation data impacts the purification performance. In Table~\ref{tab:dif_val}, we see the change in performance while gradually reducing the validation size from 1\% to 0.02\%. Even with only 50 (0.1\%) data points, FIP  can successfully remove the backdoor by bringing down the attack success rate (ASR) to 5.91\%. In an extreme scenario of one-shot FIP, we have only one sample from each class to fine-tune the model. Our proposed method is able to tackle the backdoor issue even in such a scenario. We consider AWM and ANP for this comparison. For both ANP and AWM, reducing the validation size has a severe impact on test accuracy (ACC). We consider Badnets attack on the CIFAR10 dataset for this evaluation.

\subsubsection{Effect of Different Architectures}~\label{app:architecture_impact}
We further validate the effectiveness of our method under different network settings. In Table~\ref{tab:diff_architectures}, we show the performance of FIP with some of the widely used architectures such as VGG-16~\citep{simonyan2014very}, EfficientNet~\citep{tan2019efficientnet} and Vision Transformer (VIT)~\citep{dosovitskiy2020image}. Here, we consider a smaller version of ViT-S with 21M parameters. FIP is able to remove backdoors irrespective of the network architecture. This makes sense as most of the architecture uses either fully connected or convolution layers, and FIP can be implemented in both cases.

\subsubsection{Combining Different Backdoor Attacks}\label{sec:combine_attack}
We also perform experiments with combined backdoor attacks. To create such attacks, we poison some portion of the training data using three different attacks; Badnets, Blend, and Trojan. Each of these attacks has an equal share in the poison data. As shown in Table~\ref{tab:composite_backdoor}, we use four different poison rates: $10\%\sim50\%$. FIP outperforms other baseline methods (MCR and ANP) by a satisfactory margin.

\subsubsection{More All2All Attacks}\label{sec:all2all_attacks}
Most of the defenses evaluate their methods on only All2One attacks, where we consider only one target label. However, there can be multiple target classes in a practical attack scenario. We consider one such case: All2All attack, where target classes are uniformly distributed among all available classes. In Table~\ref{tab:all2all_attacks}, we show the performance under such settings for three different attacks with a poison rate of $10\%$. It shows that the All2All attack is more challenging to defend against as compared to the All2One attack. However, the performance of FIP seems to be consistently better than other defenses for both of these attack variations. For reference, we achieve an ASR improvement of $3.12\%$ over ANP while maintaining a lead in classification accuracy too. 

\subsubsection{Adaptive Attacks.}\label{sec:adaptive_attacks}
We follow Eq.~\ref{eqn:ce_loss_constraint} in our paper to simulate an adaptive attack. Table~\ref{tab:adaptive} shows that as we increase $\eta_F$, the model becomes smoother while it becomes harder to insert the backdoor, hence, the ASR drops. However, f-FIP successfully purifies the model even with such adaptive attacks. For larger $\eta_F$, we obtain a higher drop in ACC. The underlying reason for this could be that the convergence point of the backdoor model is more favorable to clean distribution. Applying f-FIP would shift the model from that convergence point and cause this undesirable ACC drop.

\subsubsection{t-SNE Visualization of Cluster Structures} In Figure~\ref{fig:tSNE_Visualization}, we visualize the class clusters before and after backdoor purification. We take CIFAR10 dataset with Badnets attack for this visualization. For visualization purposes only, we assign the label ``0'' to the clean data cluster from the target class and the label ``11'' to the poison data cluster. However, both of these clusters have the same training label ``0'' during backdoor training. Figure~\ref{fig:wo_disc_pure} clearly indicates that our proposed method can break the poison data clusters and reassign them to their original class cluster.

\section{Conclusion}
In this work, we analyzed the backdoor insertion and removal process from a novel perspective---the smoothness of the model's loss surface---showing that the backdoor model converged to a sharp minimum compared to a benign model's convergence point. To remove the effect of backdoor, we proposed to re-optimize the model to smooth minima. Following our analysis, we proposed a novel backdoor purification technique using the knowledge of the Fisher-Information matrix to remove the backdoor efficiently instead of using na\"ive (e.g., general-purpose ones) optimization techniques to re-optimize. 
Furthermore, to preserve the post-purification clean test accuracy of the model, we introduced a novel clean data distribution-aware regularizer. Last but not least, a faster version of FIP has been proposed where we only fine-tuned the singular values of weights instead of directly fine-tuning the weights. FIP achieves SOTA performance in terms of running time and accuracy in a wide range of benchmarks, including four different tasks and ten benchmark datasets against 14 SOTA backdoor attacks.

\noindent \textbf{Limitations.} 
Here, we discussed a couple of limitations of our proposed and, hence, corresponding future works to address those. 
\textbf{First}, it is observable that no matter which defense techniques we use the clean test accuracy (ACC) consistently drops for all datasets. Here, we try to explain the reason behind this, especially for fine-tuning-based techniques, as FIP  is one of them. Since these techniques use a small validation set for fine-tuning, they do not necessarily cover the whole training data distribution. Therefore, fine-tuning with this small amount of data bears the risk of overfitting and reduced clean test accuracy. While our clean accuracy retainer partially solves this issue, more rigorous and sophisticated methods must be designed to fully alleviate this issue. \textbf{Second}, while our method is based on thorough empirical analysis and corresponding theoretical justification, there is no theoretical guarantee whether the proposed method provably removes backdoor from a pre-trained model (which is out-of-scope of this work). However, in the case of resource-constraints safety-critical systems,  it is often necessary to use a pre-trained model; hence, provable backdoor defense is necessary for safety-critical applications. In future work, we will focus on provable backdoor defense for safety-critical applications.



\section*{Acknowledgements}
We thank the anonymous reviewers for their insightful feedback. This work was supported in part by the National Science Foundation under Grant ECCS-1810256 and CMMI-2246672.

\bibliographystyle{ACM-Reference-Format}
\bibliography{ref}
\appendix

\section{Appendix (Supplementary Material)}

\subsection{Proof of Theorem~\ref{thm:smoothness}}\label{sec:thmproof}

\begin{proof} Let us consider a training set \(\{\xbf, y\} = \{\xbf_c, y_c\} \cup \ \{\xbf_b, y_b\}\), where \(\{\xbf_c, y_c\}\)\footnote{Note that we use \(\{\xbf_c, y_c\}\) to denote clean samples whereas \(\{\xbf,y\}\) was used in the main paper to denote all training samples. We start with a clean training set, \(\{\xbf,y\}\), and then add the trigger to some of the samples \(\xbf_b\) with target label \(y_b\) that produce poison set, \( \{\xbf_b, y_b\}\).} is the set of clean samples and \( \{\xbf_b, y_b\}\) is the set of backdoor or poison samples. In our work, we estimate the loss Hessian \emph{w.r.t.} \emph{standard data distribution}, i.e. training samples with their \emph{ground truth labels}. More details on this are in \emph{Appendix~\ref{sec:smooth_Clean}}. 

First, let us consider the scenario where we optimize a DNN ($f_c$) on \(\{\xbf_c, y_c\}\) only (benign model).  From the \(L_c-\)Lipschitz property of loss-gradient (ref. Assumption 1, Eq.~(\ref{eqn.assump1})) corresponding to \textit{any clean sample}\footnote{Here, loss-gradient corresponding to clean sample means we first compute the loss using clean sample and then take the gradient.} \(\xbf_c\), we get 
\begin{equation}
||\nabla_\theta \ell(\xbf_c, \theta_1) - \nabla_\theta \ell(\xbf_c,\theta_2)|| \leq L_c ||\theta_1 -\theta_2||, ~ \forall \theta_1, \theta_2 \in \Theta    
\end{equation}

Now, consider the backdoor model training ($f_b$) setup, where both clean and poison samples are used concurrently for training. In such a scenario, a training sample can be either clean or poisoned. As we are using standard data distribution, we calculate the loss ($\ell$) corresponding to \(\{\xbf_c, y_c\}\cup\{\xbf_b, y_c\}\); where $y_c$ indicates the original ground truth (GT) label.  Let us bound the difference of loss gradient for backdoor training setup for any sample,
\begin{equation}
    \begin{aligned}\label{eq:final_smoothness2}
    ||\nabla_\theta \ell(\xbf, \theta_1) - &\nabla_\theta \ell(\xbf,\theta_2)||\\ 
    & \overset{{\color{blue}(i)}}{\leq} \max \{||\nabla_\theta \ell(\xbf_c, \theta_1) - \nabla_\theta \ell(\xbf_c,\theta_2)||, \\
        &~~~~~~~~~||\nabla_\theta \ell(\xbf_b, \theta_1) - \nabla_\theta \ell(\xbf_b,\theta_2)||\}  \\ &
    \overset{{\color{blue} (ii)}}{=} ||\nabla_\theta \ell(\xbf_b, \theta_1) - \nabla_\theta \ell(\xbf_b,\theta_2)|| \\
    & \overset{{\color{blue} (iii)}}{\leq} L_b ||\theta_1 -\theta_2|| \\
\end{aligned}
\end{equation}
here, step {\color{blue}(i)} follows trivially as \(||\nabla_\theta \ell(\xbf, \theta_1) - \nabla_\theta \ell(\xbf,\theta_2)||\) holds for any \(\xbf\). Unlike $f_c$ and $f_b$, we can have loss gradients corresponding to samples from clean and poison sets; {\color{blue} (ii)} leverages the properties of backdoor training where the backdoor is inserted by forcing $f_b$ to memorize the specific pattern or trigger $\delta$, specifically the mapping of $\delta \rightarrow y_b$.  At the same time,  $f_b$ learns (does not memorize) image or object-related generic patterns in \(\xbf_c\) and maps them to $y_c$, similar to $f_c$. Let us denote the optimized parameters of $f_b$  as, $\theta_1$. Since $f_b$ is optimized to predict $y_b$, we have a high loss gradient $\nabla_\theta \ell(\xbf_b, \theta_1)$ if we consider GT label $y_c$ for $\xbf_b$. Note that the backdoor model becomes too sensitive to the trigger due to the memorization effect. However, a certain group of parameters ($\theta^{(b)}$) show far more sensitivity to the backdoor as compared to others ($\theta^{(c)}$), where $\theta = \theta^{(c)} \cup \theta^{(b)}$ and $|\theta^{(b)}|<<|\theta^{(c)}|$. This has also been shown in previous studies~\cite{wu2021adversarial, li2023reconstructive, chai2022one}. Therefore, even a small change to $\theta^{(b)}$ will make the backdoor model show significantly less (or no) sensitivity to the trigger. On the other hand, a small change in $\theta^{(c)}$ has very little impact on recognizing image-related generic patterns.  Now, consider a scenario where $\theta_1$ is slightly changed to $\theta_2$. Due to this shift, the loss gradient $\nabla_\theta \ell(\xbf_b, \theta_2)$ becomes significantly smaller if we calculate it \emph{w.r.t.} $y_c$. This happens for the following reasons: {\color{magenta}(1)}  Due to the change in $\theta^{(b)}$, the model does not show sensitivity towards $\delta$ anymore. {\color{magenta}(2)} As mentioned before, with small change in $\theta^{(c)}$, the model can still recognize patterns in samples. This means the model ignores $\delta$ in \(\xbf_b (= \xbf + \delta\)) while recognizing image-related patterns in \(\xbf\) and predicting the GT label $y_c$. Therefore, the change in loss gradient ($||\nabla_\theta \ell(\xbf_b, \theta_1) - \nabla_\theta \ell(\xbf_b,\theta_2)||$) is large. On the other hand, due to the reason {\color{magenta}(2)}, the change in loss gradient ($||\nabla_\theta \ell(\xbf_c, \theta_1) - \nabla_\theta \ell(\xbf_c,\theta_2)||$) is smaller. Finally, we can write the following,
\begin{equation}\label{eqn:lblc}
    \begin{aligned}
       & ||\nabla_\theta\ell(\xbf_c, \theta_1) - \nabla_\theta \ell(\xbf_c,\theta_2)|| \leq L_c||\theta_1 - \theta_2||\\
        & ||\nabla_\theta\ell(\xbf_b, \theta_1) - \nabla_\theta \ell(\xbf_b,\theta_2)|| \leq L_b||\theta_1 - \theta_2||
    \end{aligned}
\end{equation}
In our above discussion, we have shown that $$||\nabla_\theta\ell(\xbf_c, \theta_1) - \nabla_\theta \ell(\xbf_c,\theta_2)|| < ||\nabla_\theta \ell(\xbf_b, \theta_1) - \nabla_\theta \ell(\xbf_b,\theta_2)||$$ Therefore, for the same set of \(\theta_1, \theta_2\), Eq.~\ref{eqn:lblc} suggests that \(
L_c < L_b\). Note that $L_c<L_b$ holds if we consider the smallest Lipschitz constant for Eq.~\ref{eqn:lblc}~\cite{benyamini1998geometric} {\color{blue} (iii)} follows the definition of smoothness.

Hence, the loss of the backdoor model is \(L_b-\)Smooth and \(L_c<L_b\).

\end{proof}

\noindent\textbf{Takeaway from the theoretical analysis.} Based on Theorem~\ref{thm:smoothness}, we can rewrite Eq.~\ref{eqn-lsigma} for a backdoor model,
\begin{equation}\label{eqn:thm-imp}
    \sup_{\theta} 
 \sigma(\nabla^2_{\theta} \mathcal{L}) \leq \max\{L_c, L_b\}  
\end{equation}
where \(\mathcal{L}\) is the loss-function of the backdoor model computed over \(\{\xbf, y\} = \{\xbf_c, y_c\} \cup \ \{\xbf_b, y_c\}\).

The R.H.S. of Eq.~\ref{eqn:thm-imp} represents the supremum\footnote{We used the definition of supremum (\url{https://en.wikipedia.org/wiki/Infimum\_and\_supremum}) slightly differently than the formal definition. By supremum, we indicate that \(\max\{L_c,L_b\}\) is the lowest value for the Lipschitz constant of the backdoor model \(f_b(.)\) to hold Eq.~(\ref{eqn:thm-imp}).} for smoothness of a backdoor model.
From Eq.~(\ref{eq:final_smoothness2}), it can be observed that \(L_c<L_b\) which leads to the following form of Eq.~\ref{eqn:thm-imp},
\begin{equation}\label{eqn:thm-imp_2}
    \sup_{\theta} 
 \sigma(\nabla^2_{\theta} \mathcal{L}) \leq L_b  
\end{equation}
Hence, a backdoor model tends to show less smoothness on \(\mathcal{L}\), computed over \(\{\xbf, y\} = \{\xbf_c, y_c\} \cup \ \{\xbf_b, y_c\}\), as compared to a benign model with \(L_c-\)Lipschitz continuity.    

\subsection{Data Distribution for Smoothness Analysis}~\label{sec:smooth_Clean} 


\vspace{-1mm}
We choose to conduct the smoothness analysis \textit{w.r.t.} standard data distribution (training data with ground truth labels). The intuition behind using this standard data distribution is as follows: Both our empirical analysis and Theorem~\ref{thm:smoothness} show that a backdoor model is less smooth than a benign model. This reveals the unstable nature of a backdoor model. It can be observed from the threat model that any backdoor attack usually does not follow the standard training procedure. Instead, the attacker achieves his/her goal by introducing an anomaly in the training process. Note that the anomaly here is not the triggered data, but \emph{the pair of triggered data and target (poison) label}. For any triggered data, if we do not change the ground truth label to the target label, the resulting model will be a robust clean model instead of a backdoor model. Because, without the target label, the model will treat the trigger as one type of augmentation. Therefore, the source of the unstable nature (of a backdoor model) lies in not using the ground truth labels for triggered or poison data. On the other hand, the benign model follows the standard training procedure without any anomaly in it. As a result, it shows more stability as compared to a backdoor model.

\begin{table*}[ht]
\centering
\caption{Detailed information of the datasets and DNN architectures used in our experiments.}
\label{tab4}

\scalebox{0.8}{
\begin{tabular}{cccccc}
\toprule
Dataset & Classes & Image Size & Training Samples & Test Samples & Architecture \\ \midrule
CIFAR-10 & 10 & 32 x 32  & 50,000 & 10,000 & PreActResNet18 \\
GTSRB & 43 & 32 x 32  & 39,252 & 12,630 & WideResNet-16-1 \\
Tiny-ImageNet & 200 & 64 x 64  & 100,000 & 10,000 & ResNet34 \\
ImageNet & 1000 & 224 x 224  & 1.28M & 100,000 & ResNet50 \\
\bottomrule
\label{tab:datasets}
\end{tabular}}
\vspace{-4mm}
\end{table*}

\subsection{Experimental Details}\label{sec:experiment}

\subsubsection{Details of Attacks}\label{app:attack_details}

\noindent \textbf{Single-Label Settings.} We use 14 different attacks for CIFAR10. Each of them differs from the others in terms of either label mapping type or trigger properties. To ensure a fair comparison, we follow similar trigger patterns and settings as in their original papers. In Troj-one and Dyn-one attacks, all of the triggered images have the same target label. On the other hand, target labels are uniformly distributed over all classes for Troj-all and Dyn-all attacks. For label poisoning attacks, we use a fixed poison rate of 10\%. However, we need to increase this rate to 80\% for CLB and SIG. We use an image-trigger mixup ratio of 0.2 for Blend and SIG attacks. WaNet adopts a universal wrapping augmentation as the backdoor trigger. WaNet can be considered a non-additive attack since it works like an augmentation technique with direct information insertion or addition like Badnets or TrojanNet. ISSBA adds a specific trigger to each input that is of low magnitude and imperceptible. Both of these methods are capable of evading some existing defenses. For the BPPA attack, we follow the PyTorch implementation\emph{\footnote{\url{https://github.com/RU-System-Software-and-Security/BppAttack}}}. For Feature attack (FBA), we create a backdoor model based on this implementation\footnote{\url{https://github.com/Megum1/DFST}}.
Apart from clean-label attacks, we use a poison rate of 10\% for creating backdoor attacks. The details of these attacks are presented in Table~\ref{tab:attack_details}. In addition to these attacks, we also consider `All2All' attacks (Troj-all, Dyn-all), where we have more than one target label. We change the given label $i$ to the target label $i+1$ to implement this attack. For class 9, the target label is 0.   

For creating backdoor models with CIFAR10~\citep{krizhevsky2009learning}, we train a PreActResNet~\citep{he2016identity} model using an SGD optimizer with an initial learning rate of 0.01, learning rate decay of 0.1/100 epochs for 250 epochs. We also use a weight decay of $5e^{-4}$ with a momentum of 0.9. We use a longer backdoor training to ensure a satisfactory attack success rate. We use a batch size of 128. For GTSRB~\citep{stallkamp2011german}, we train a WideResNet-16-1~\citep{zagoruyko2016wide} model for 200 epochs with a learning rate of 0.01 and momentum of 0.9. We also regularize the weights with a weight-decay of $5e^{-4}$. We rescale each training image to $32\times32$ before feeding them to the model. The training batch size is 128, and an SGD optimizer is used for all training.  We further created backdoor models trained on the Tiny-ImageNet and ImageNet datasets.
For Tiny-ImageNet, we train the model for 150 epochs with a learning rate of 0.005, a decay rate of 0.1/60 epochs, and a weight decay of 1e-4. For ImageNet, we train the model for 200 epochs with a learning rate of 0.02 with a decay rate of 0.1/75 epochs. We also employ 0.9 and 1e-4 for momentum and weight decay, respectively. The details of these four datasets are presented in Table~\ref{tab:datasets}.

\noindent \textbf{Multi-Label Settings.} In case of single-label settings, we put a trigger on the image and change the corresponding ground truth of that image. However,~\citep{chen2023clean} shows that a certain combination of objects can also be used as a trigger pattern instead of using a conventional pattern, e.g., reverse lambda or watermark. For example, if a combination of car, person, and truck is present in the image, it will fool the model to misclassify. For creating this attack, we use three object detection datasets Pascal VOC 07, VOC 12, and MS-COCO. We use a poison rate of 5\% for the first 2 datasets and 1.5\% for the latter one.  The rest of the training settings are taken from the original work~\citep{chen2023clean}. 

\noindent \textbf{Video Action Recognition.} An ImageNet pre-trained ResNet50 network has been used for the CNN, and a sequential input-based Long Short Term Memory (LSTM)~\citep{sherstinsky2020fundamentals} network has been put on top of it. We subsample the input video by keeping one out of every 5 frames and use a fixed frame resolution of $224 \times 224$. We choose a trigger size of $20\times20$. Following~\citep{zhao2020clean}, we create the required perturbation for clean-label attack by running projected gradient descent (PGD)~\citep{madry2017towards} for 2000 steps with a perturbation norm of $\epsilon=16$. Note that our proposed augmentation strategies for image classification are directly applicable to action recognition. The rest of the settings are taken from the original work.

\noindent \textbf{3D Point Cloud.} PointBA~\citep{li2021pointba} proposes both poison-label and clean-label backdoor attacks in their work. For poison-label attacks, PointBA introduces specific types of triggers: orientation triggers and interaction triggers. A more sophisticated technique of feature disentanglement was used for clean-label attacks. \citep{xiang2021backdoor} inserts a small cluster of points as the backdoor pattern using a special type of spatial optimization. For evaluation purposes, we consider the ModelNet~\citep{wu20153d} dataset and PointNet++~\citep{qi2017pointnet++} architecture. We follow the attack settings described in ~\citep{li2021pointba,xiang2021backdoor} to create the backdoor model. We also consider ``backdoor points" based attack (3DPC-BA) described in \citep{xiang2021backdoor}. For creating these attacks, we consider a poison rate of 5\% and train the model for 200 epochs with a learning rate of 0.001 and weight decay 0.5/20 epochs. Rest of the settings are taken from original works.

\subsubsection{Implementation Details of FIP}\label{app:FIP_Implementation_Details}
We provide the implementation details of our proposed method here for different attack settings.

\noindent \textbf{Single-Label Settings.} apply FIP  on CIFAR10, we fine-tune the backdoor model following Eq.~(\ref{eqn:loss_spectral}) for $E_p$ epochs with $1\%$ clean validation data. Here, $E_p$ is the number of purification epochs, and we choose a value of 50 for this. Note that we set aside the 1\% validation data from the training set, not the test or evaluation set.  For optimization, we choose a learning rate of 0.01 with a decay rate of 0.1/40 epochs and choose a value of 0.001 and 5 for regularization constants $\eta_F$ and $\eta_r$, respectively. Note that we consider backpropagating the gradient of Tr(F) once every 10 iterations.  For GTSRB, we increase the validation size to $3\%$ as there are fewer samples available per class. The rest of the training settings are the same as CIFAR10. For FIP  on Tiny-ImageNet, we consider a validation size of 5\% as a size less than this seems to hurt clean test performance (after purification). We fine-tune the model for 15 epochs with an initial learning rate of 0.01 with a decay rate of 0.3/epoch. Finally, we validate the effectiveness of FIP  on ImageNet. For removing the backdoor, we use 3\% validation data and fine-tune it for 2 epochs. A learning rate of 0.001 has been employed with a decay rate of 0.005 per epoch. 

\noindent \textbf{Multi-Label Settings.} For attack removal, we take 5000 clean validation samples for all defenses. For removing the backdoor, we take 5000 clean validation samples and train the model for 20 epochs. It is worth mentioning that the paradigm of multi-label backdoor attacks is very recent, and there are not many defenses developed against it yet.


\noindent \textbf{Video Action Recognition.} During training, we keep 5\% samples from each class to use them later as the clean validation set. We train the model for 30 epochs with a learning rate of 0.0001.

\noindent \textbf{3D Point Cloud.} For removal, we use 400 point clouds as the validation set and fine-tune the backdoor model for 20 epochs with a learning rate of 0.001. Our proposed method outperforms other SoTA defenses in this task by a significant margin.

\begin{table*}[t]
\renewcommand{\arraystretch}{1.1}
\renewcommand\tabcolsep{1.75pt}
\small
\centering
\caption{Performance \textbf{comparison of FIP with additional defenses on CIFAR10 dataset under 7 different backdoor attacks}. FIP achieves SOTA performance for six attacks while sacrificing only $4.19\%$ in clean accuracy (ACC) on average. The average drop indicates the difference in values before and after removal. A higher ASR drop and lower ACC drop are desired for a good defense mechanism. Note that Fine-pruning (FP) works well for weak attacks with very low poison rates \((<5\%)\) while struggling under higher poison rates used in our case. }
\label{tab:cifar_defense}

\scalebox{1}{
\begin{tabular}{c|cc|cc|cc|cc|cc|cc|cc|cc}
\toprule
Attacks & 
\multicolumn{2}{c|}{None} & \multicolumn{2}{c|}{BadNets} &
\multicolumn{2}{c|}{Blend} & \multicolumn{2}{c|}{Trojan} & \multicolumn{2}{c|}{SIG} &  \multicolumn{2}{c|}{Dynamic}&  \multicolumn{2}{c|}{CLB}&  \multicolumn{2}{c}{LIRA}\\ 
\midrule
Defenses&ASR &ACC & ASR &ACC & ASR &ACC & ASR &ACC & ASR &ACC & ASR &ACC& ASR &ACC & ASR &ACC  \\ \midrule
\emph{No Defense} & 0 & 95.21 & 100 & 92.96  & 100 & 94.11 & 100 & 89.57 & 100 & 88.64 & 100 & 92.52 & 100 & 92.78 & 99.25 & 92.15   \\ 
Vanilla FT & 0 & 93.28 & 6.87 & 87.65 & 4.81 & 89.12 & 5.78 & 86.27 & 3.04 & 84.18 & 8.73 & 89.14 & 5.75 & 87.52 & 7.12 & 88.16  \\
 FP &  0 & 88.92 & 28.12 & 85.62 & 22.57 & 84.37 & 20.31 & 84.93 & 29.92 & 84.51  & 19.14&84.07&12.17&84.15&22.14&82.47 \\
 MCR & 0&90.32&3.99&81.85&9.77&80.39&10.84&80.88&3.71&82.44&8.83&78.69&7.74&79.56&11.81&81.75\\
 NAD & 0&92.71&4.39&85,61&5.28&84.99&8.71&83.57&2.17&83.77&13.29&82.61&6.11&84.12&13.42&82.64\\

 CBD & 0 & 92.87 & 2.27 & 87.92 & 2.96 & 89.61 & \textbf{1.78} & 86.18 & 1.98 & 84.17 & 2.03 & 88.41 & 4.21 & 87.70 & 6.67 & 87.42 \\ 

ABL & 0 & 91.64 & 3.08 & 87.70 & 7.74 & 89.15 & {3.53} & 86.38 & 3.65 & 85.20 & 8.07 & 88.36 & 2.21 & 89.42 & 4.24 & 90.18 \\ 
 \midrule
  FIP(Ours) & 0&94.10&\textbf{1.86}&\textbf{89.32}&\textbf{0.38}&\textbf{92.17}&{2.64}&\textbf{87.21}&\textbf{0.92}&\textbf{86.10}&\textbf{1.17}&\textbf{91.16}&\textbf{2.04}&\textbf{91.37}&\textbf{2.53}&\textbf{89.82}\\
 \bottomrule
\end{tabular}}
\end{table*}

\begin{table*}[t]
\renewcommand{\arraystretch}{1.1}
\renewcommand\tabcolsep{1.75pt}
\small
\centering
\caption{ Performance comparison of \textbf{FIP and additional defense methods for GTSRB dataset}. The average drop in ASR and ACC determines the effectiveness of a defense method.}
\vspace{-1.5mm}
\label{tab:gtsrb_defense}
\scalebox{0.8}{
\begin{tabular}{c|cc|cc|cc|cc|cc|cc|cc|cc}
\toprule
Attacks & \multicolumn{2}{c|}{None} & \multicolumn{2}{c|}{BadNets} & \multicolumn{2}{c|}{Blend} & \multicolumn{2}{c|}{Trojan} & \multicolumn{2}{c|}{SIG} &  \multicolumn{2}{|c}{Dynamic} & \multicolumn{2}{|c}{WaNet} & \multicolumn{2}{|c}{ISSBA}\\ \midrule
Defenses&ASR &ACC & ASR &ACC & ASR & ACC & ASR &ACC & ASR &ACC & ASR &ACC& ASR &ACC & ASR &ACC  \\ \midrule
\emph{No Defense} & 0&97.87&100&97.38&100&95.92&99.71&96.08&97.13&96.93&100&97.27& 98.19& 97.31 & 99.42& 97.26 \\ 
{Vanilla FT} & 0 & 95.08 & 5.36 & 94.16 & 7.08 & 93.32 & 4.07 & 92.45 & 5.83 & 93.41 & 8.27 & 94.26 & 6.56 & 95.32 & 5.48 & 94.73  \\
 FP &  0&90.14&29.57&88.61&24.50&86.67&19.82&84.03&14.28&90.50&24.84&88.38& 38.27 & 89.11 & 24.92 & 88.34 \\
 MCR & 0&95.49&4.02&93.45&6.83&92.91&4.25&92.18&8.98&91.83&14.82&92.41& 11.45 & 91.20 & 9.42 & 92.04 \\
 NAD & 0&95.18&5.19&89.52&8.10&89.37&6.98&90.27&9.36&88.71&16.93&90.83&14.52 & 90.73 & 16.65 & 91.18 \\
 
  CBD & 0&95.64&0.82&95.21&\textbf{1.90}&94.11&2.16&94.29&5.41&94.37&1.97&\textbf{95.91}& 3.87 & 95.67 & 5.15 & 94.12 \\
ABL & 0&96.41&\textbf{0.19}&96.01&12.54&93.14&1.76&94.84&4.86&94.91&6.75&95.10&7.90 & 93.41 & 8.71 & 95.39 \\
 \midrule
  FIP(Ours) & 0&\textbf{96.76}&{0.24}&\textbf{96.11}&{2.41}&\textbf{94.16}&\textbf{1.21}&\textbf{95.18}&\textbf{2.74}&\textbf{95.08}&\textbf{1.52}&{95.27}&\textbf{1.20}&\textbf{96.24}&\textbf{1.43}&\textbf{95.86} \\
 \bottomrule
\end{tabular}}
\vspace{-1mm}
\end{table*}

\subsubsection{Implementation Details of Other Defenses}\label{sec:other_defenses_imple}
For FT-SAM~\citep{zhu2023enhancing}, we follow the implementation of sharpness-aware minimization where we restrict the search region for the SGD optimizer. Pytorch implementation described here\footnote{\url{https://github.com/davda54/sam}} has been followed where we fine-tune the backdoor model for 100 epochs with a learning rate of 0.01, weight decay of $1e^{-4}$, momentum of 0.9, and a batch size of 128. For experimental results with ANP~\citep{wu2021adversarial}, we follow the source code implementation\footnote{\url{https://github.com/csdongxian/ANP_backdoor}}. After creating each of the above-mentioned attacks, we apply adversarial neural pruning on the backdoor model for 500 epochs with a learning rate of 0.02. We use the default settings for all attacks. For vanilla FT, we perform simple DNN fine-tuning with a learning rate of 0.01 for 125 epochs. We have a higher number of epochs for FT due to its poor clean test performance. The clean validation size is 1\% for both of these methods. 
For NAD~\citep{li2021neural}, we increase the validation data size to 5\% and use the teacher model to guide the attacked student model. We perform the training with distillation loss proposed in NAD\footnote{\url{https://github.com/bboylyg/NAD}}. For MCR~\citep{zhao2020bridging}, the training goes on for 100 epochs according to the provided implementation\footnote{\url{https://github.com/IBM/model-sanitization/tree/master/backdoor/backdoor-cifar}}. 
For I-BAU~\citep{zeng2021adversarial}, we follow their PyTorch Implementation\footnote{\url{https://github.com/YiZeng623/I-BAU}} and purify the model for 10 epochs. We use 5\% validation data for I-BAU. For AWM~\citep{chai2022one}, we train the model for 100 epochs and use the Adam optimizer with a learning rate of 0.01 and a weight decay of 0.001. We use the default hyper-parameter setting as described in their work $ \alpha = 0.9, \beta = 0.1, \gamma = [10,8], \eta = 1000$.
The above settings are for CIFAR10 and GTSRB only. We follow the GitHub \footnote{\url{https://github.com/bboylyg/RNP}} implementation of RNP where we use a learning rate of 0.01. We also do the same for ABL\footnote{\url{https://github.com/bboylyg/ABL}} and CBD\footnote{\url{https://github.com/zaixizhang/CBD}}.
For Tiny-ImageNet, we keep most of the training settings similar except for reducing the number of epochs significantly. We also increase the validation size to 5\% for vanilla FT, FT-SAM, ANP, and AWM. For I-BAU, we use a higher validation size of 10\%. For purification, we apply ANP and AWM for 30 epochs, I-BAU for 5 epochs, and Vanilla FT for 25 epochs. For ImageNet, we use a 3\% validation size for all defenses (except for I-BAU, where we use 5\% validation data) and use different numbers of purification epochs for different methods. We apply I-BAU for 2 epochs. On the other hand, we train the model for 3 epochs for ANP, AWM, and vanilla FT and FT-SAM.

\subsubsection{Comparison with Additional Baseline Defenses}\label{sec:comparison_additional_baselines}
In FP~\citep{liu2017neural}, pruning and fine-tuning are performed simultaneously to eliminate the backdoors. 
~\citep{liu2017neural} establishes that mere fine-tuning on a sparse network is ineffective as the probability is higher that the clean data doesn't activate the backdoor neurons, which emphasizes the significance of filter pruning in such networks. MCR~\citep{zhao2020bridging} put forward the significance of the mode connectivity technique to mitigate the backdoored and malevolent models. Prior to ~\citep{zhao2020bridging}, mode connectivity was only explored for generalization analysis in applications such as fast model assembling. However,~\citep{zhao2020bridging} is the preliminary study that investigated the role of mode connectivity to achieve model robustness against backdoor and adversarial attacks.
A neural attention distillation (NAD)~\citep{li2021neural} framework was proposed to erase backdoors from the model by using a teacher-guided finetuning of the poisoned student network with a small subset of clean data. However, the authors in~\citep{li2021neural} have reported overfitting concerns if the teacher network is partially purified. For Vanilla fine-tuning (FT), conventional weight fine-tuning has been used with SGD optimizer. In our work, we proposed a defense that purifies an already trained backdoor model that has learned both clean and poison distribution. Such defense falls under the category of test-time backdoor defense. To prevent the backdoor attack before even taking place, we need to develop a training-time defense where we have a training pipeline that will prevent the attack from happening. The training pipeline can consist of techniques such as specific augmentations like MixUp~\cite{zhang2017mixup}, where we mix both clean and poison samples to reduce the impact of the poison triggers. In recent times, several training-time defenses have been proposed such as CBD~\cite{zhang2023backdoor} and ABL~\cite{li2021anti}. Note that training-time defense is completely different from test-time defense and out of the scope of our paper. Nevertheless, we also show a comparison with these training-time defenses and other baselines in Table~\ref{tab:cifar_defense} and Table~\ref{tab:gtsrb_defense}. 
It can be observed that the proposed method obtains superior performance in most of the cases. 

\begin{figure*}[t]
\centering
\begin{subfigure}{0.38\textwidth}
    \includegraphics[width=1\linewidth]{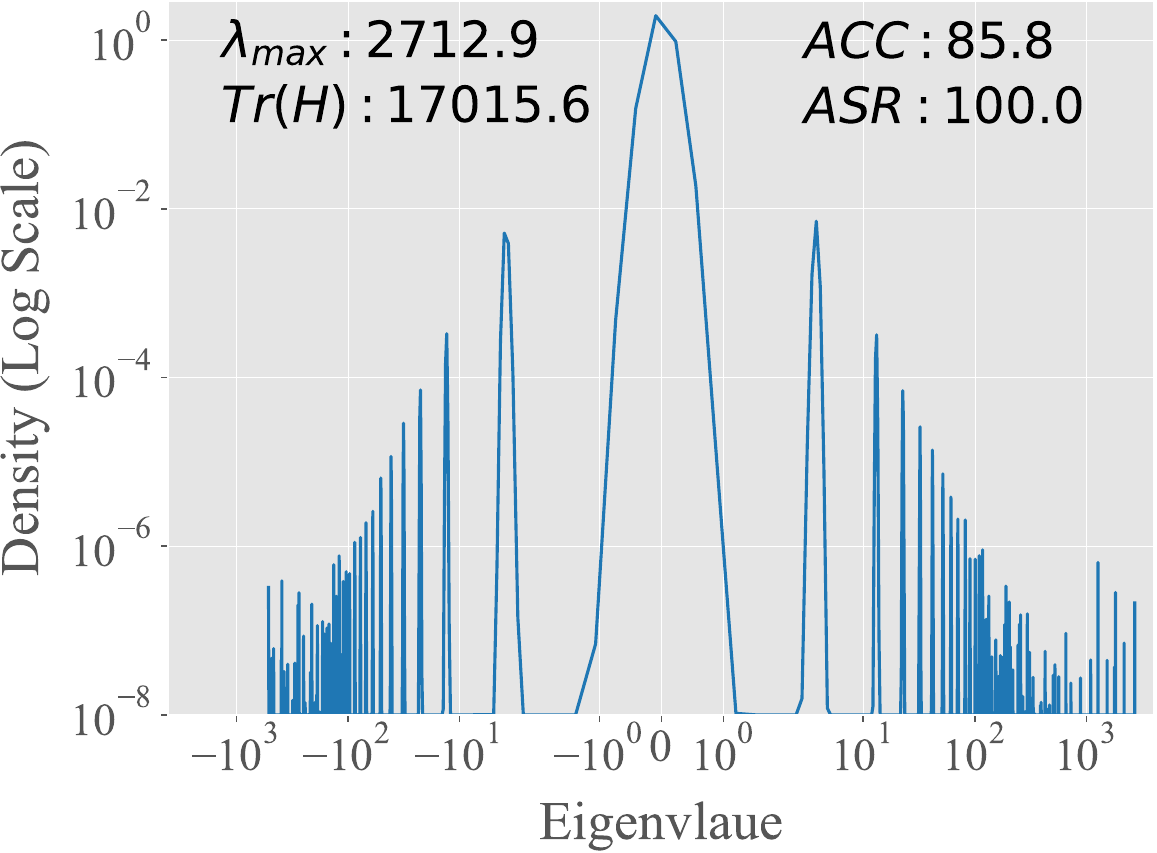}
    \caption{\footnotesize Badnets Attack}
    \label{fig:badnets}
\end{subfigure}
\begin{subfigure}{0.38\textwidth}
    \includegraphics[width=1\linewidth]{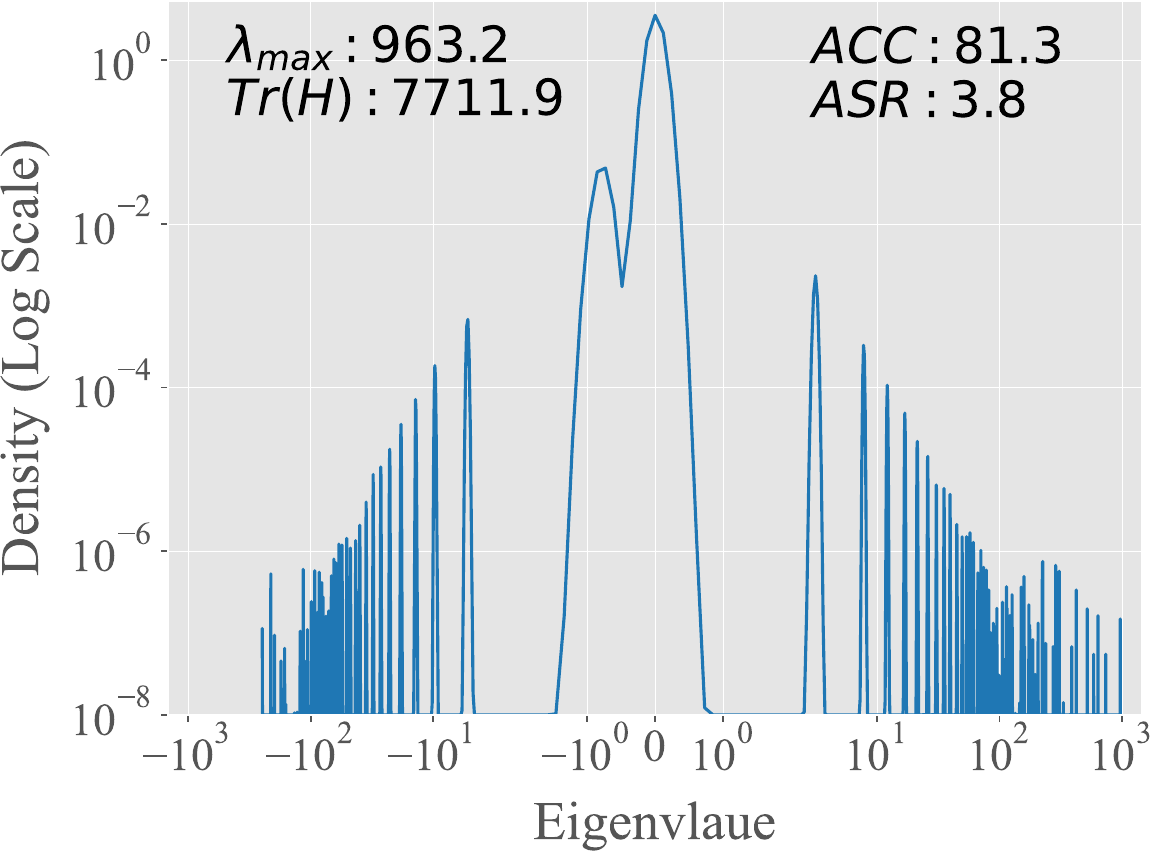}
    \caption{\footnotesize Badnets Purification }
    \label{fig:badnets}
\end{subfigure}
\begin{subfigure}{0.38\linewidth}
    \includegraphics[width=1\linewidth]{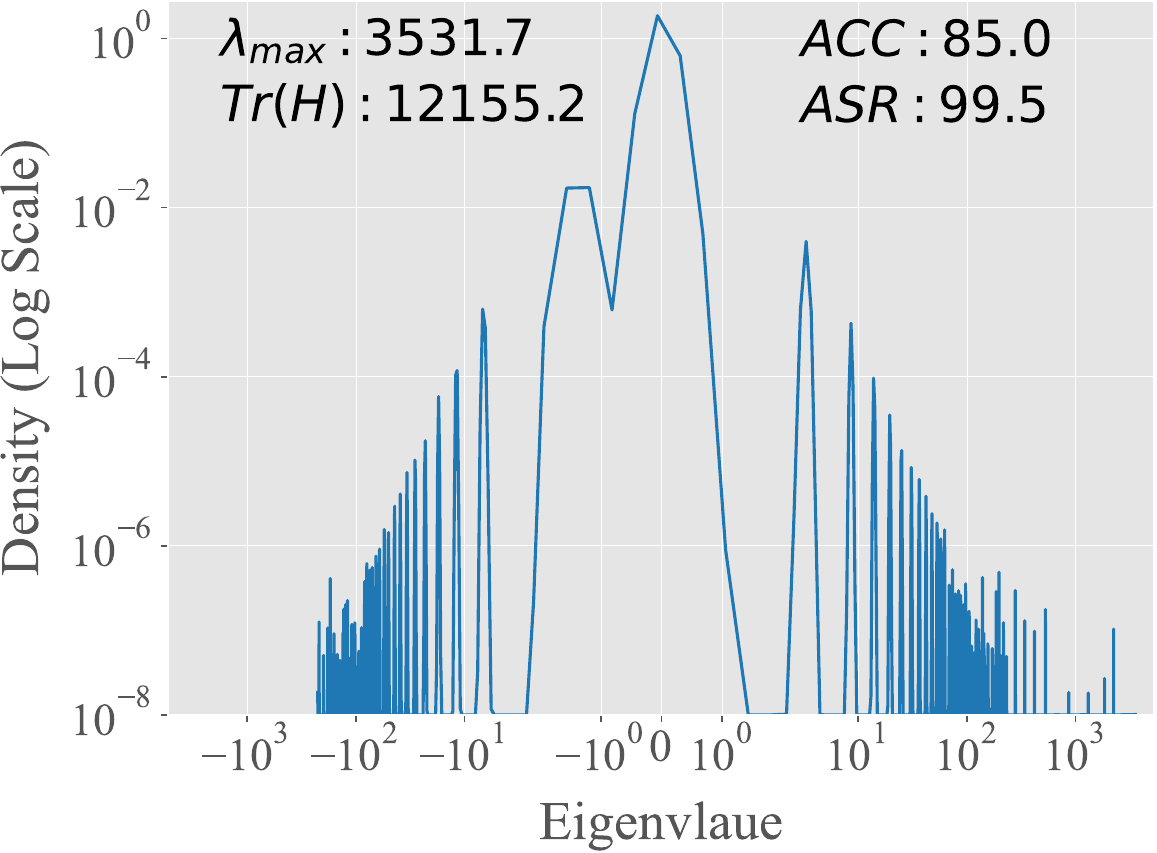}
    \caption{\footnotesize Clean-Label (CLB) Attack}
    \label{fig:badnets_pure}
\end{subfigure}
\begin{subfigure}{0.38\linewidth}
    \includegraphics[width=1\linewidth]{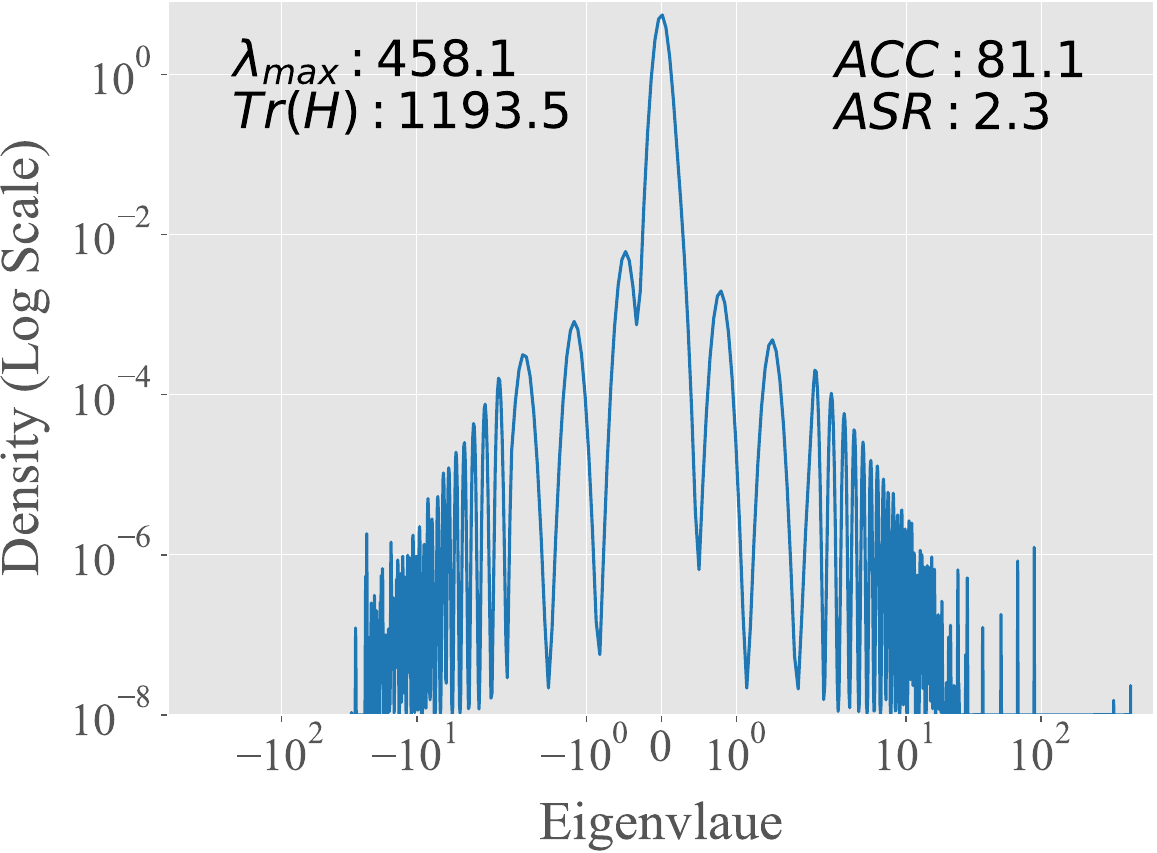}
    \caption{\footnotesize Clean-Label (CLB) Purification }
    \label{fig:badnets_pure}
\end{subfigure}
\begin{subfigure}{0.38\linewidth}
    \includegraphics[width=1\linewidth]{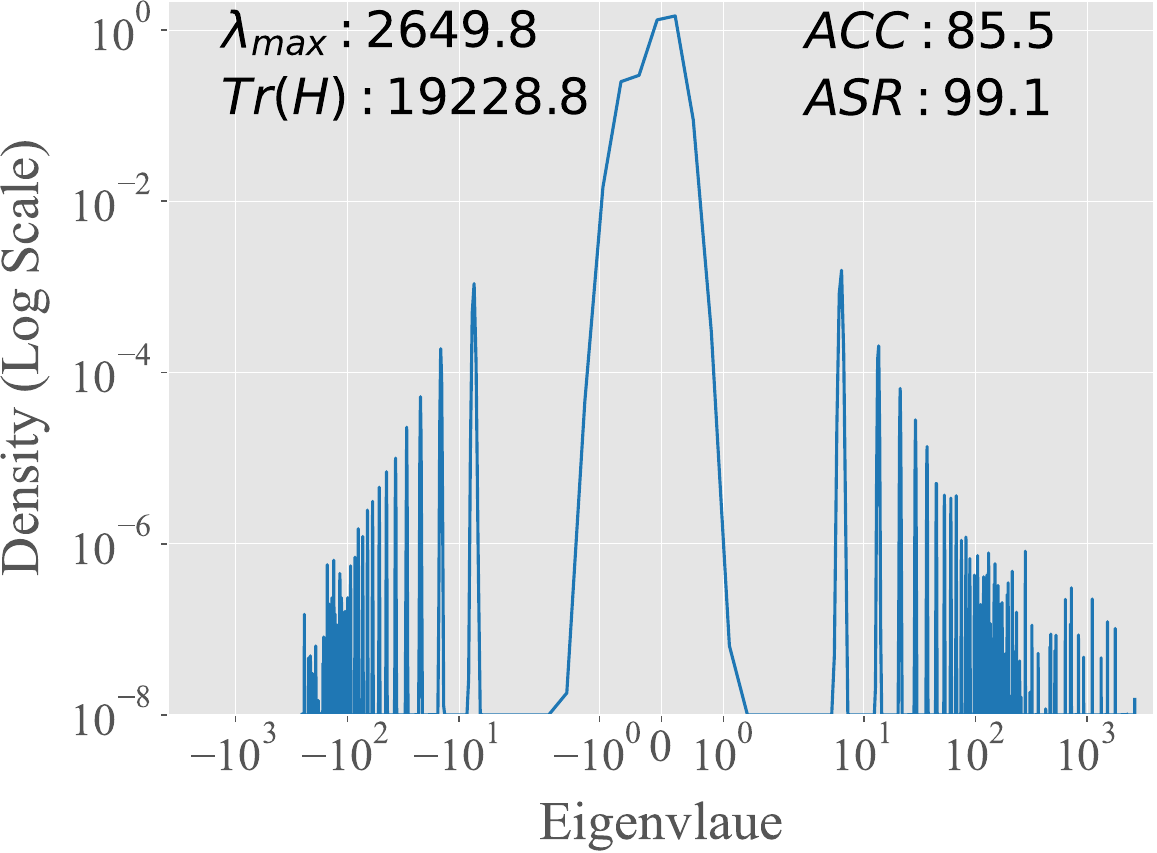}
    \caption{\footnotesize SIG Attack}
    \label{fig:badnets_pure}
\end{subfigure}
\begin{subfigure}{0.38\linewidth}
    \includegraphics[width=1\linewidth]{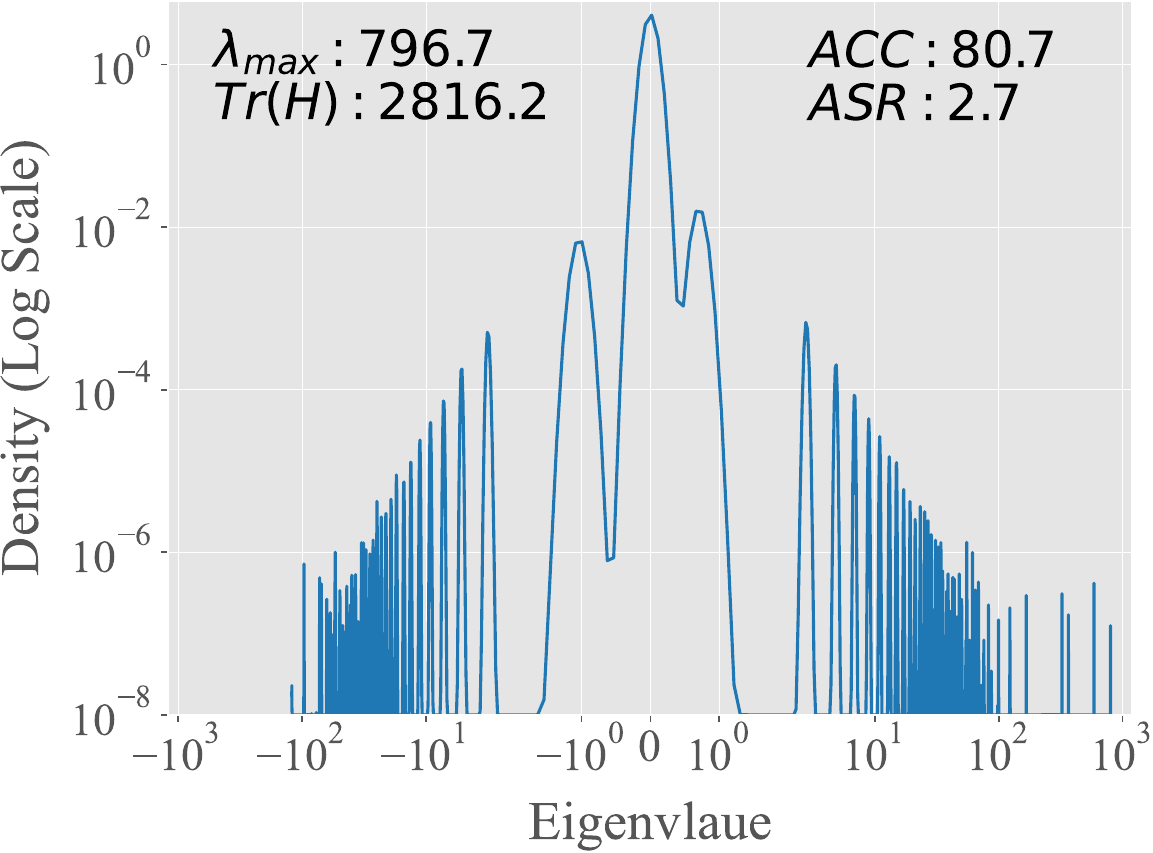}
    \caption{\footnotesize SIG Purification}
    \label{fig:badnets_pure}
\end{subfigure}
\begin{subfigure}{0.38\linewidth}
    \includegraphics[width=1\linewidth]{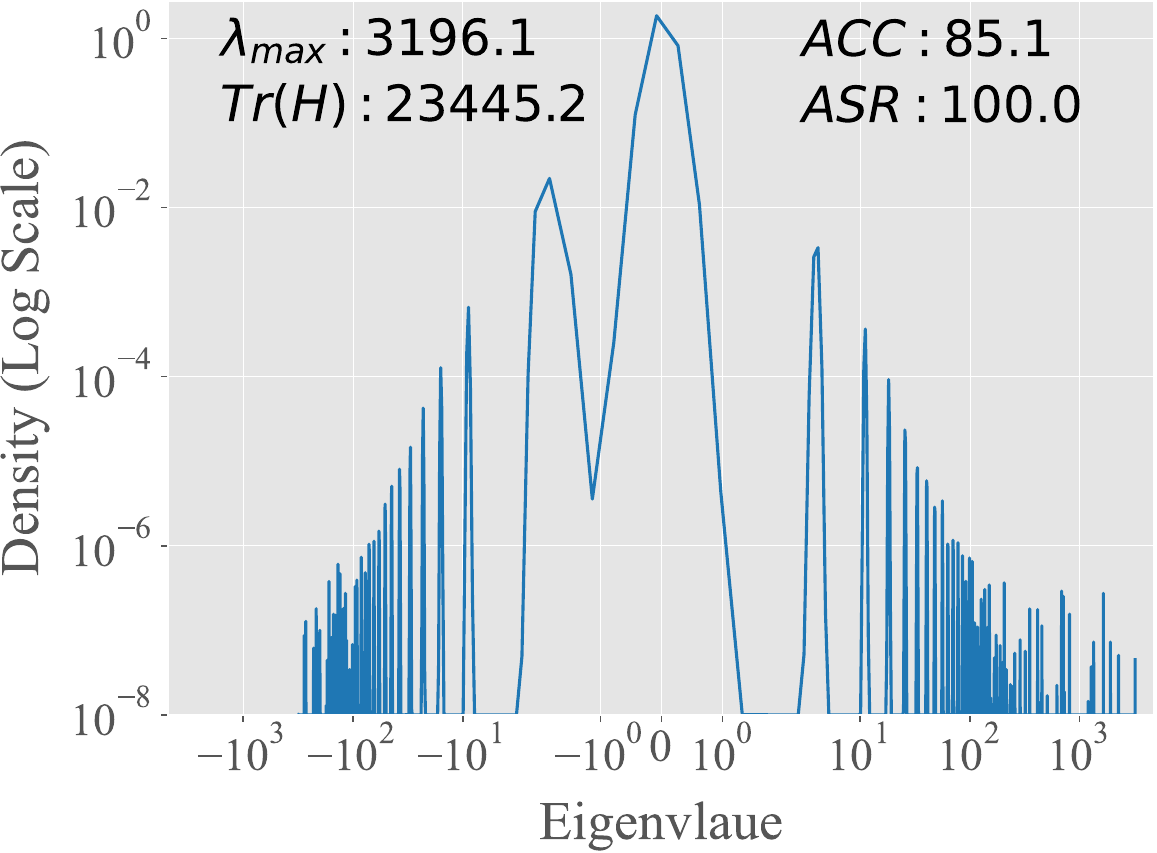}
    \vspace{-4mm}
    \caption{\footnotesize Blend Attack}
    \label{fig:badnets_pure}
\end{subfigure}
\begin{subfigure}{0.38\linewidth}
    \includegraphics[width=1\linewidth]{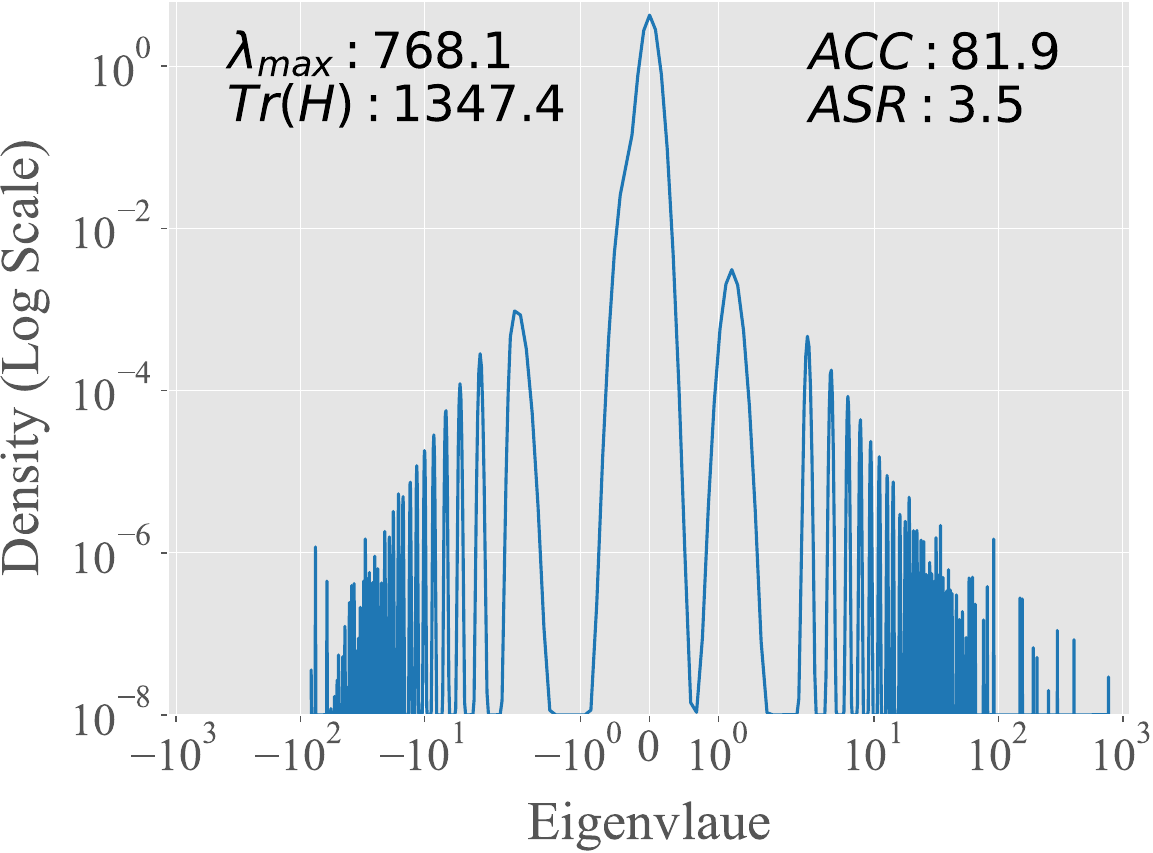}
    \vspace{-4mm}
    \caption{\footnotesize Blend Purification}
    \label{fig:badnets_pure}
\end{subfigure}
\caption{ Smoothness analysis for ImageNet Subset (first 12 classes). A ResNet34 architecture is trained on the subset. For GPU memory constraint, we consider only the first 12 classes while calculating the loss Hessian. Eigen Density plots of backdoor models  (before and after purification) are shown here.}

\label{fig:smoothness_imagenet}
\end{figure*}

\begin{figure*}[ht]
\centering
\begin{subfigure}{0.3\linewidth}
    \includegraphics[width=1\linewidth]{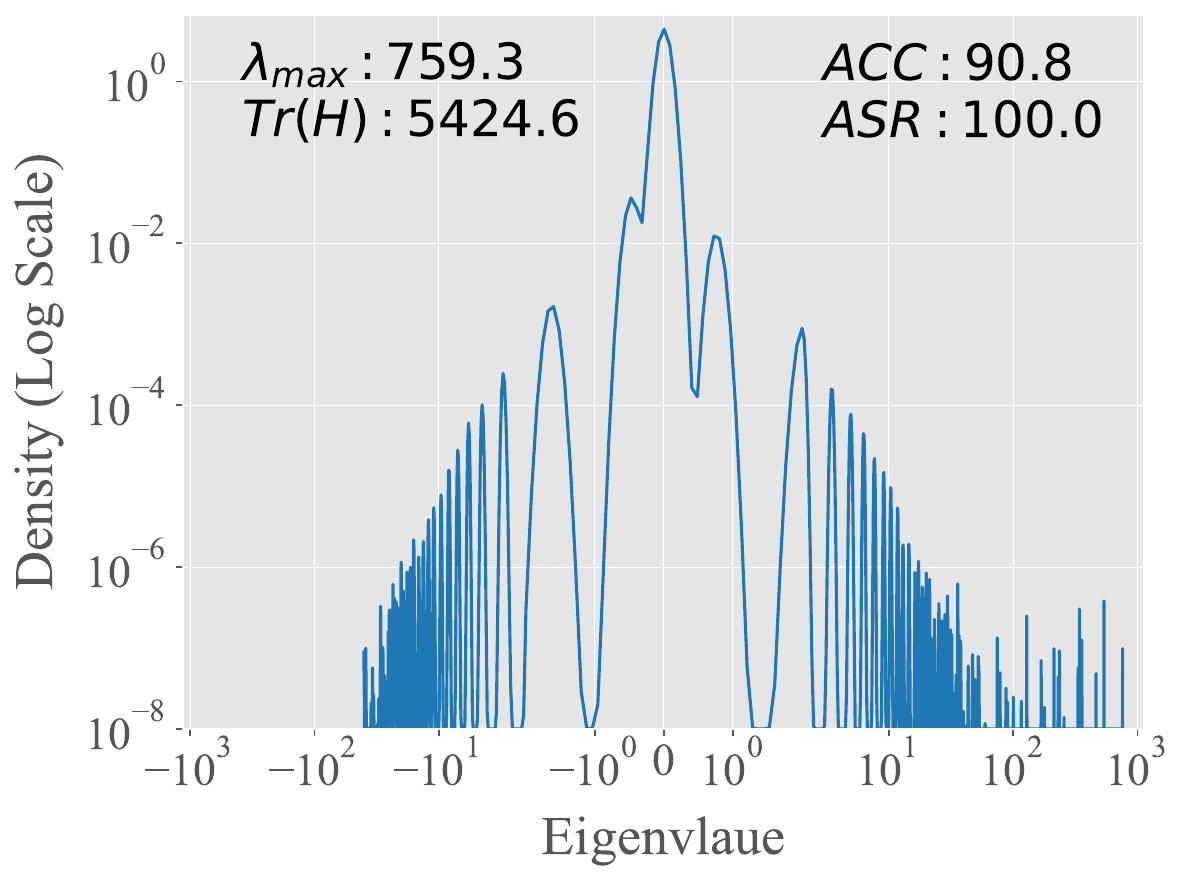}
    \caption{\footnotesize Attack (VGG19)}
    \label{fig:badnets}
\end{subfigure}
\begin{subfigure}{0.3\linewidth}
    \includegraphics[width=1\linewidth]{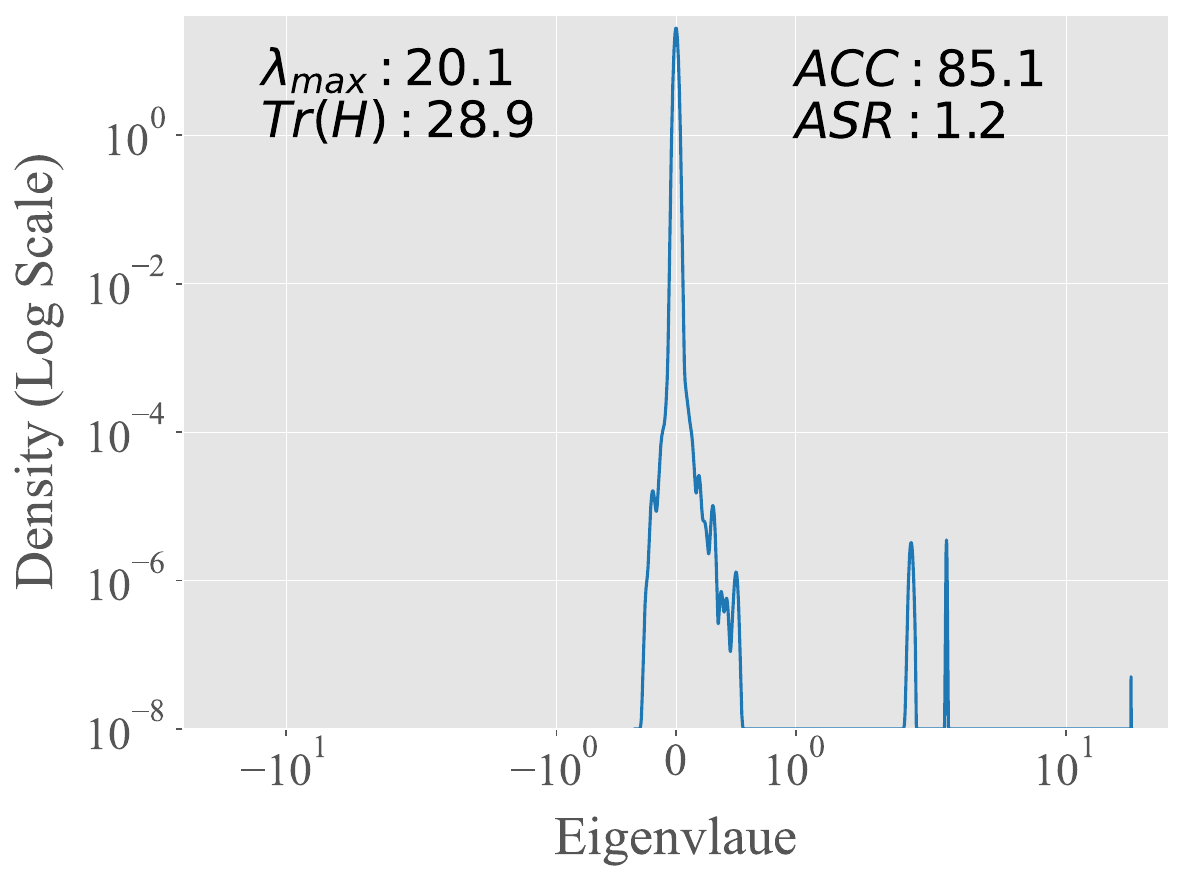}
    \caption{\footnotesize Purification (VGG19) }
    \label{fig:badnets}
\end{subfigure}
\begin{subfigure}{0.3\linewidth}
    \includegraphics[width=1\linewidth]{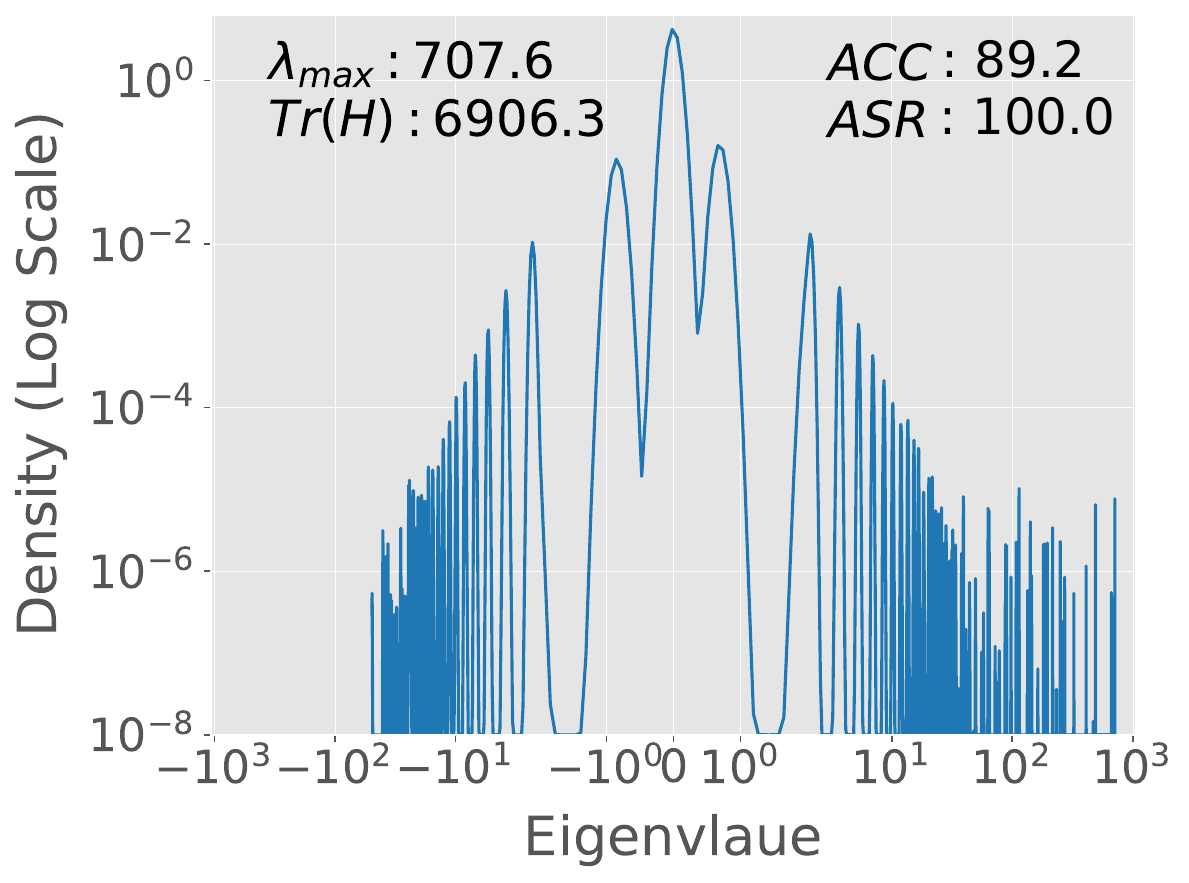}
    \caption{\footnotesize Attack (MobileNetV2)}
    \label{fig:badnets}
\end{subfigure}
\begin{subfigure}{0.3\linewidth}
    \includegraphics[width=1\linewidth]{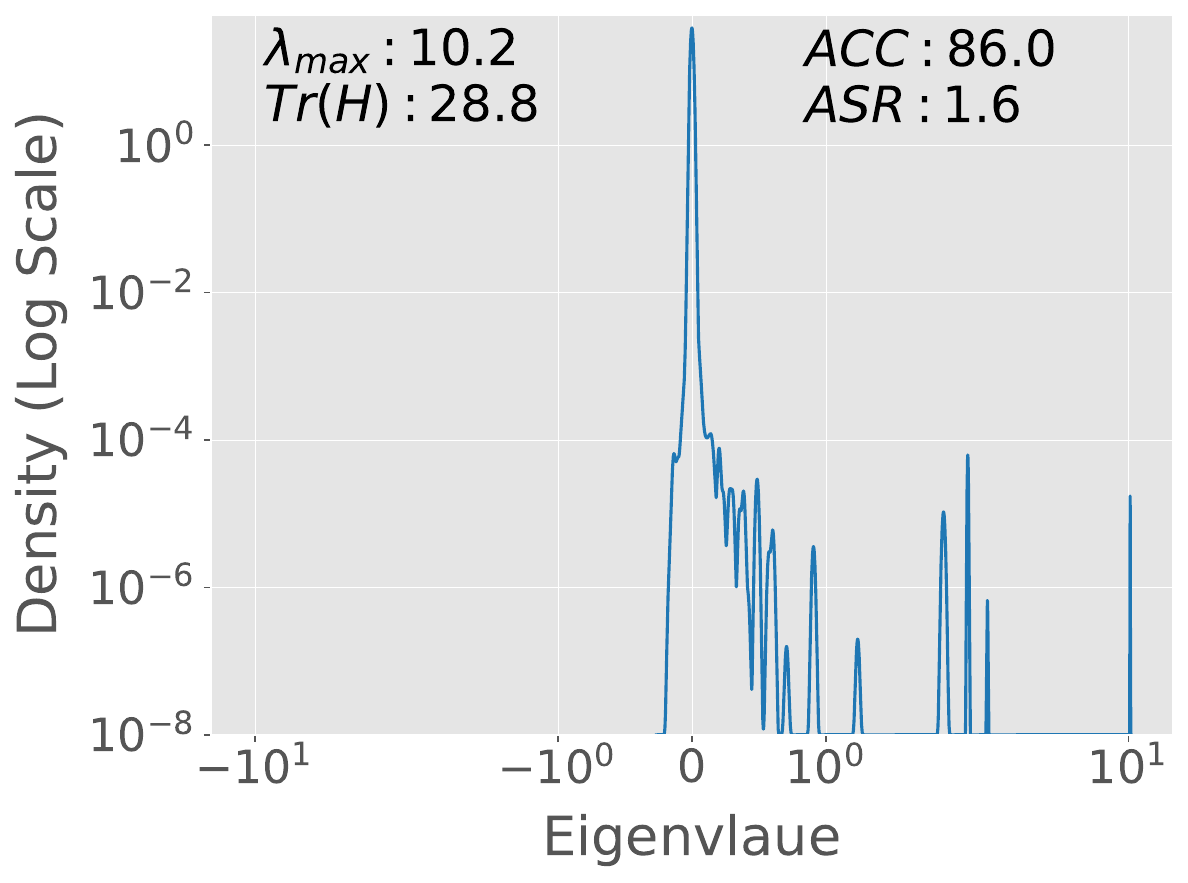}
    \caption{\footnotesize Purification (MobileNetV2) }
    \label{fig:badnets}
\end{subfigure}
\begin{subfigure}{0.3\linewidth}
    \includegraphics[width=1\linewidth]{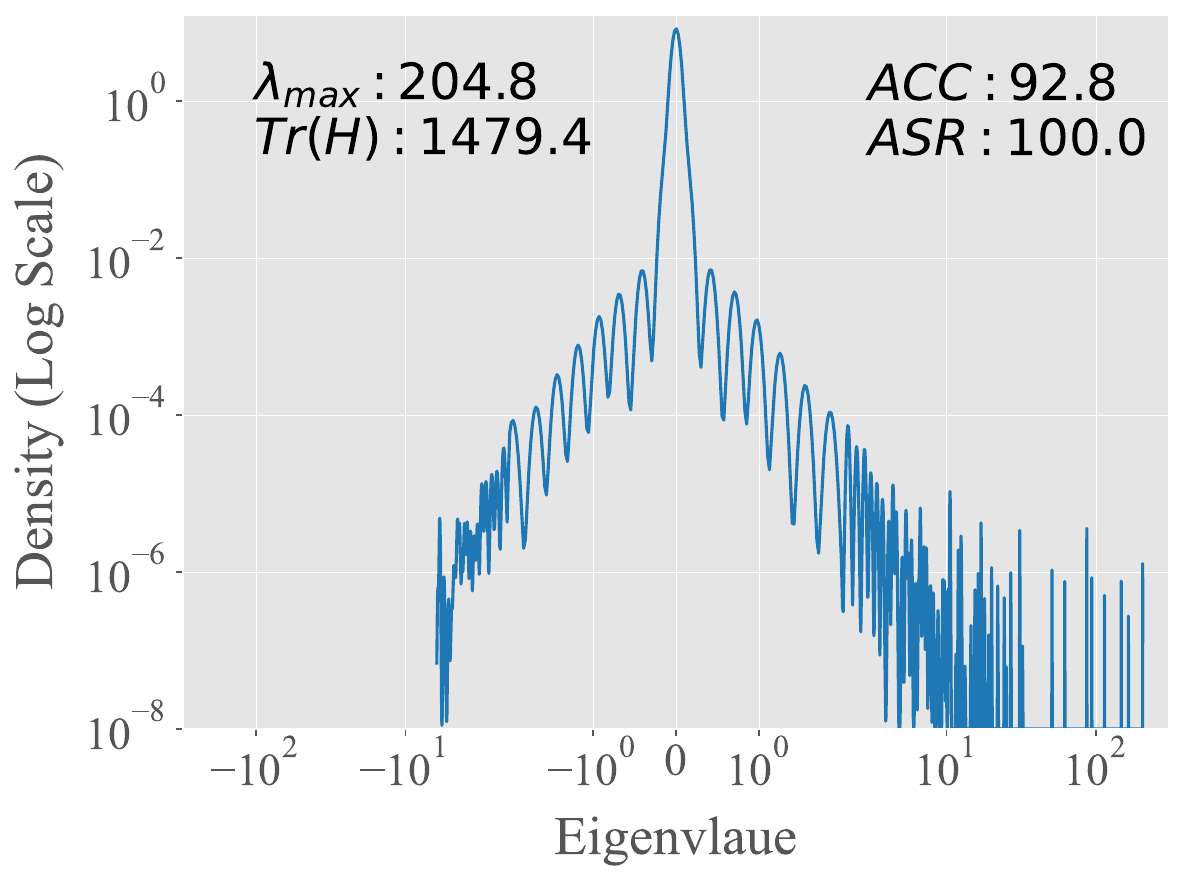}
    \caption{\footnotesize Attack (GoogleNet)}
    \label{fig:badnets}
\end{subfigure}
\begin{subfigure}{0.3\textwidth}
    \includegraphics[width=1\linewidth]{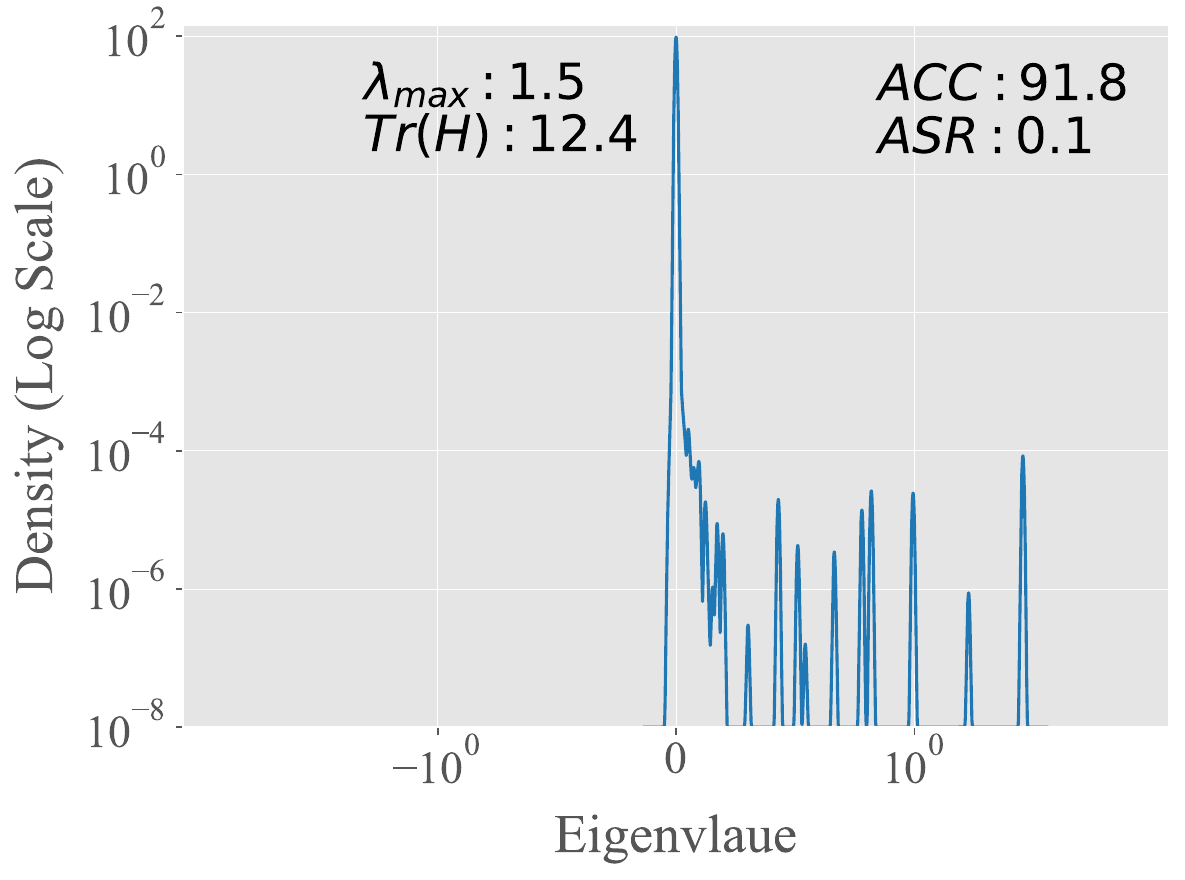}
    \caption{\footnotesize Purification (GoogleNet) }
    \label{fig:badnets}
\end{subfigure}
\begin{subfigure}{0.3\textwidth}
    \includegraphics[width=1\linewidth]{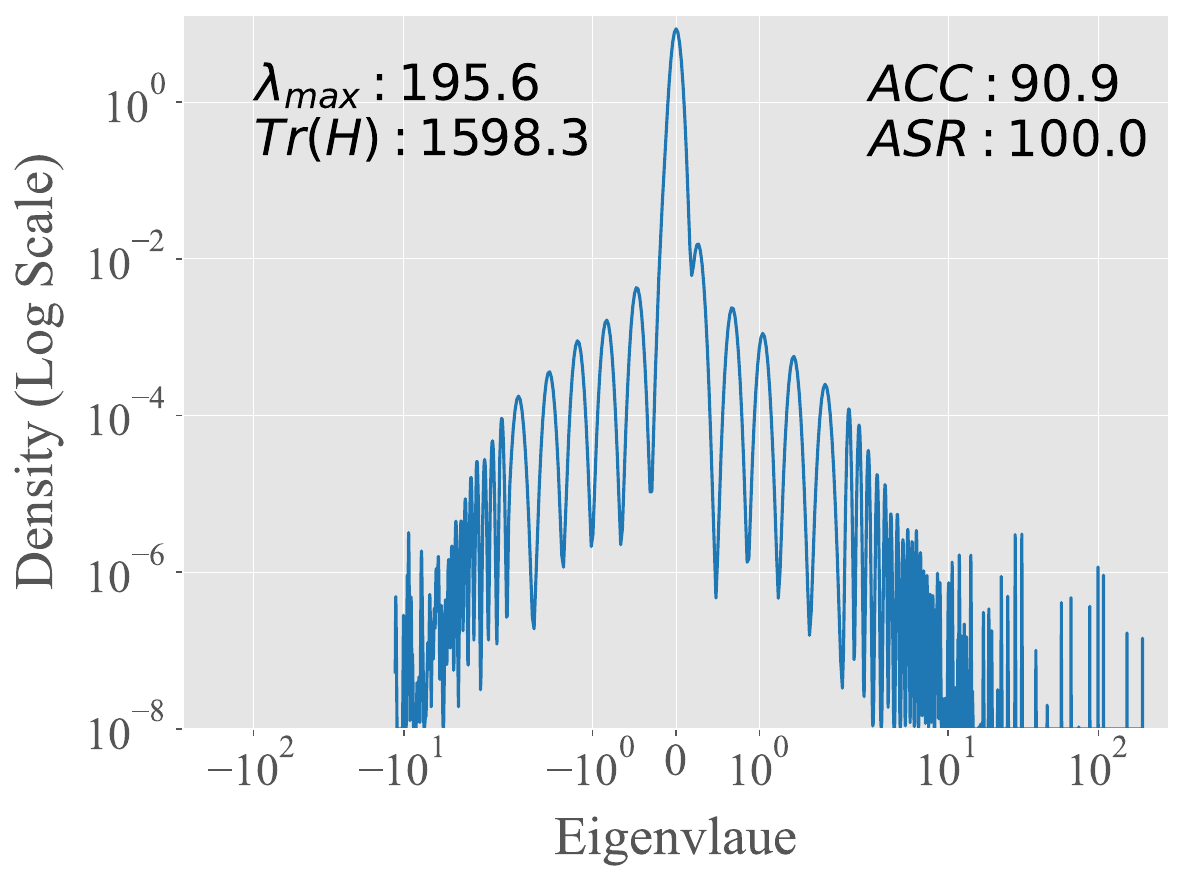}
    \caption{\footnotesize Attack (InceptionV3)}
    \label{fig:badnets}
\end{subfigure}
\begin{subfigure}{0.3\textwidth}
    \includegraphics[width=1\linewidth]{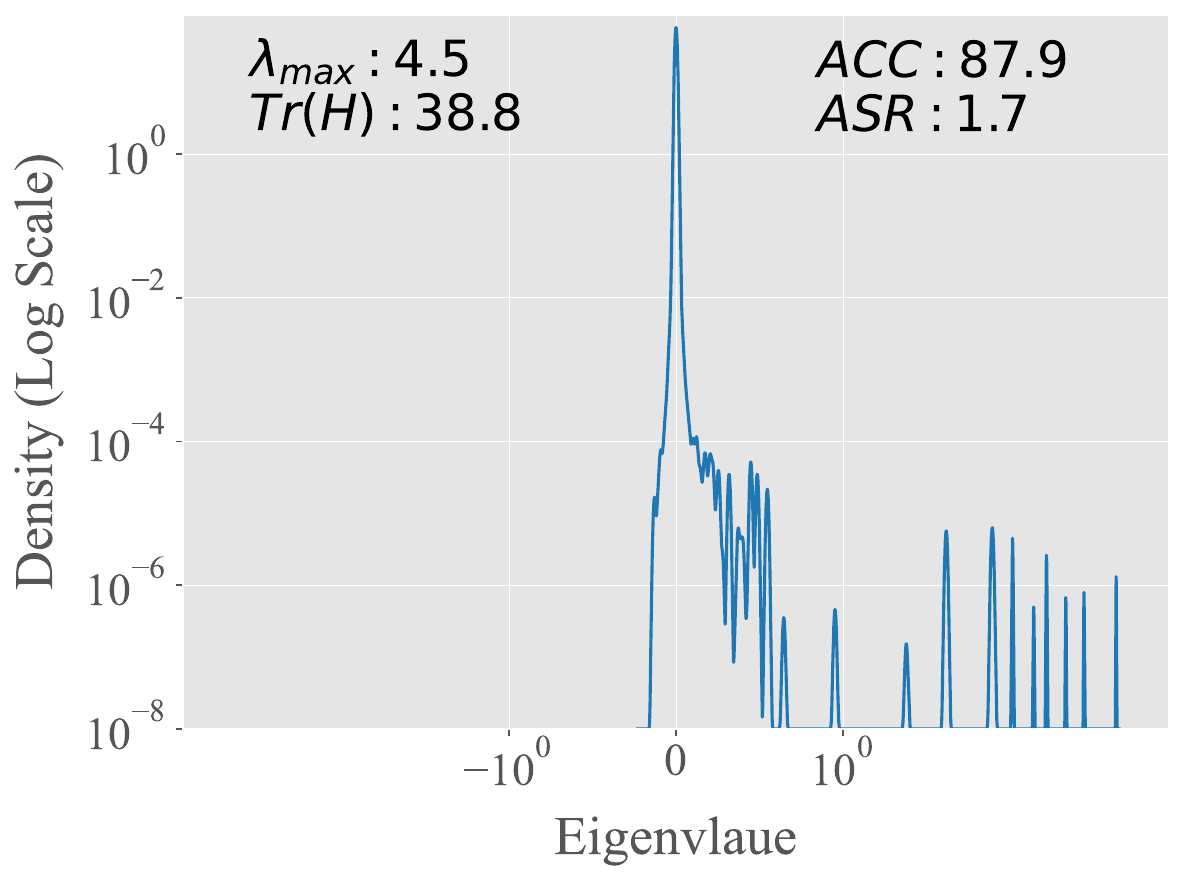}
    \caption{\footnotesize Purification (InceptionV3) }
    \label{fig:badnets}
\end{subfigure}
\begin{subfigure}{0.3\textwidth}
    \includegraphics[width=1\linewidth]{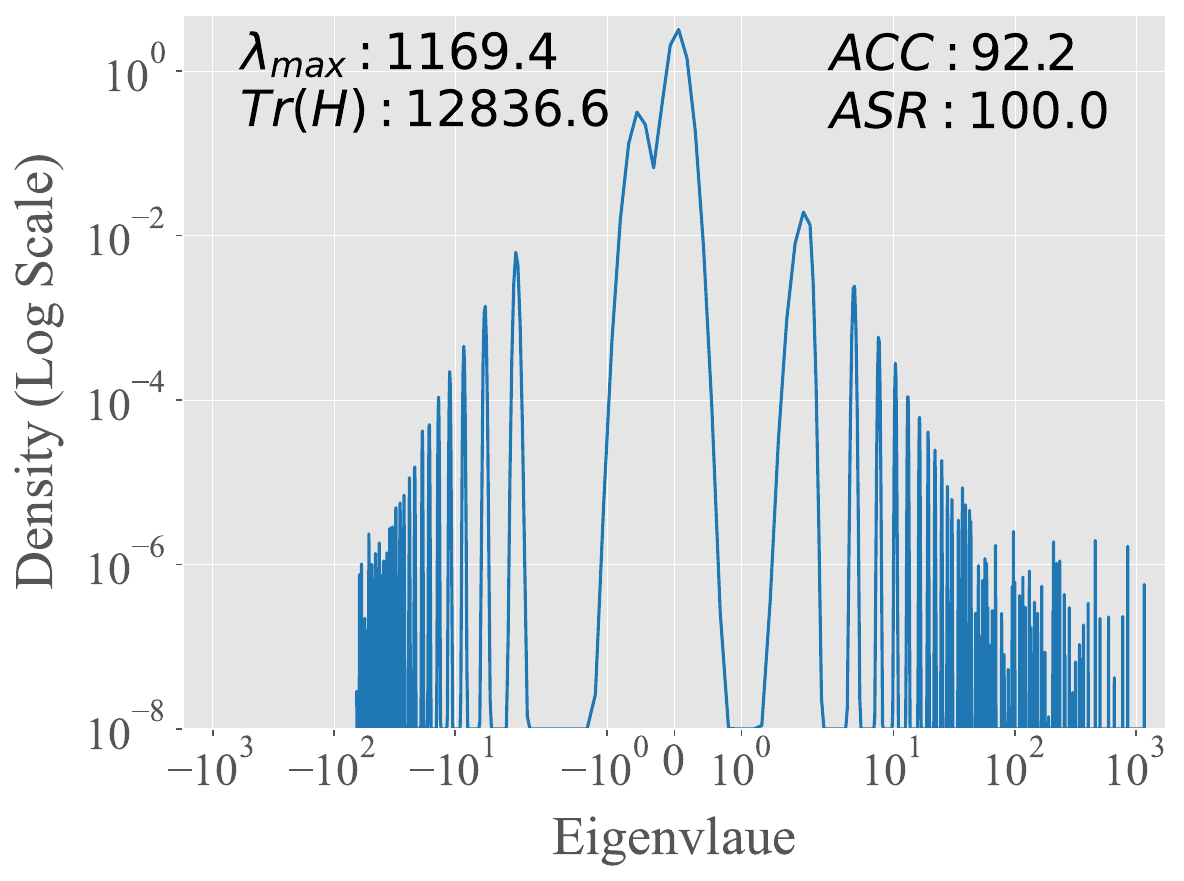}
    \caption{\footnotesize Attack (DenseNet121)}
    \label{fig:badnets}
\end{subfigure}
\begin{subfigure}{0.3\linewidth}
    \includegraphics[width=1\linewidth]{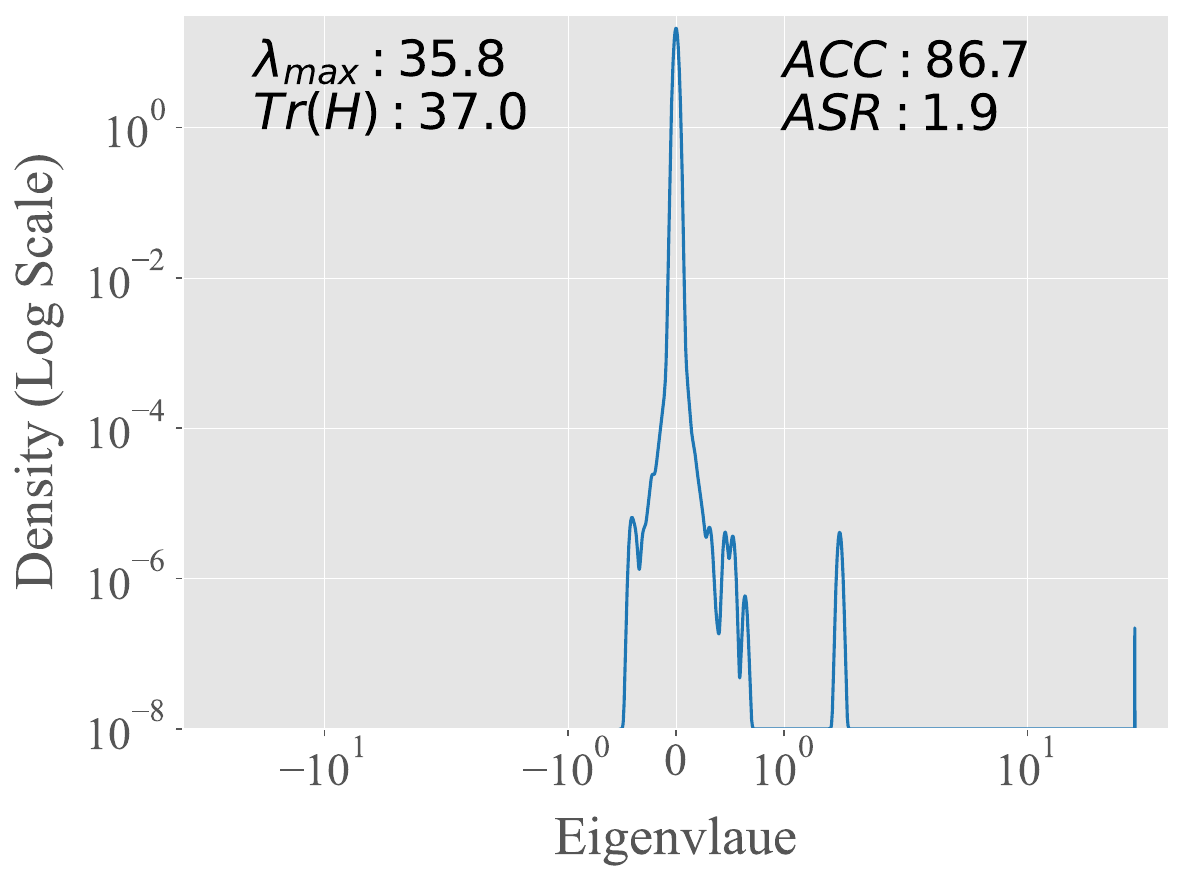}
    \caption{\footnotesize Purification (DenseNet121) }
    \label{fig:badnets}
\end{subfigure}
\caption{ Smoothness Analysis of Backdoor Attack and Purification for different architectures. For all architectures, we consider the Badnets attack on CIFAR10. }
\label{fig:smoothness_arch}

\end{figure*}

\subsubsection{More Results on Smoothness Analysis}\label{sec:more_smoothness_analysis}
For smoothness analysis, we follow the PyHessian implementation\footnote{\url{https://github.com/amirgholami/PyHessian}} and modify it according to our needs. We use a single batch with size 200 to calculate the loss Hessian for all attacks with CIFAR10 and GTSRB datasets. 


\noindent \textbf{Different Architectures.} We conduct further smoothness analysis for the ImageNet dataset and different architectures. In Fig.~\ref{fig:smoothness_imagenet}, we show the Eigendensity plots for different five different attacks. We used 2 A40 GPUs with 96GB system memory. However, it was not enough to calculate the loss hessian if we consider all 1000 classes of ImageNet. Due to GPU memory constraints, we consider an ImageNet subset with 12 classes. We train a ResNet34 architecture with five different attacks. To calculate the loss hessian, we use a batch size of 50. Density plots before and after purification further confirm our proposed hypothesis. To test our hypothesis for larger architectures, we consider five different architectures for CIFAR10, i.e., VGG19~\citep{simonyan2014very}, MobileNetV2~\citep{sandler2018mobilenetv2}, DenseNet121~\citep{huang2017densely}, GoogleNet~\citep{szegedy2014going}, Inception-V3~\citep{szegedy2016rethinking}. Each of the architectures is deeper compared to the ResNet18 architecture we consider for CIFAR10.

\end{document}